\documentclass[twoside,11pt]{article}

\usepackage[preprint]{jmlr2e}
\usepackage{amsmath, amssymb, bm, multirow, url, graphicx, booktabs, xcolor}
\usepackage{subcaption}
\usepackage[linesnumbered, ruled, noline]{algorithm2e}
\SetKwInput{KwInput}{Input}                   
\SetKwInput{KwOutput}{Output}              
\SetKwInput{KwInit}{Initialization}
\SetKwBlock{Workers}{Workers $c=1,2,\ldots,M$ in parallel do}{end}
\SetKwBlock{Master}{Master do}{end}

\hypersetup{hidelinks}

\newtheorem{spec}{Specification}

\usepackage{lastpage}
\jmlrheading{XX}{2023}{1-\pageref{LastPage}}{X/XX; Revised X/XX}{X/XX}{XX-XXXX}{Shuo-Chieh Huang and Ruey S. Tsay}

\ShortHeadings{Feature-distributed Multivariate Linear Regression}{Huang and Tsay}
\firstpageno{1}

\begin{document}

\title{Scalable High-Dimensional Multivariate Linear Regression for Feature-Distributed Data}

\author{\name Shuo-Chieh Huang \email shuochieh@chicagobooth.edu \\
       \name Ruey S. Tsay \email ruey.tsay@chicagobooth.edu \\
       \addr Booth School of Business\\
       University of Chicago\\
       Chicago, IL 60637, USA
    }

\editor{My editor}

\maketitle

\begin{abstract}
    Feature-distributed data, referred to data partitioned by features and stored across multiple computing nodes, are increasingly common in applications with a large number of features. 
    This paper proposes a two-stage relaxed greedy algorithm (TSRGA) for applying multivariate linear regression to such data. The main advantage of TSRGA is that its communication complexity does not depend on the feature dimension, making it highly scalable to very large data sets. In addition, for multivariate response variables, TSRGA can be used to yield low-rank coefficient estimates. The fast convergence of TSRGA is validated by simulation experiments. 
    Finally, we apply the proposed TSRGA in a financial application that leverages unstructured data from the 10-K reports, demonstrating its usefulness in applications with many dense large-dimensional matrices.
\end{abstract}

\begin{keywords}
  Frank-Wolfe algorithm, Distributed computing, Reduced-rank regression, Feature selection, Multi-view and multi-modal data
\end{keywords}

\section{Introduction}

A computational strategy often adopted for tackling high-dimensional big data is to employ feature-distributed analysis: to partition the data by features and to store them across multiple computing nodes. 
For instance, when the data have an extremely large number of features that do not fit in a single computer, this strategy is used to circumvent storage constraints or to accelerate computation \citep{Heinze16, Wang2017, Hydra, Gao2022}. 
In addition, feature-distributed data may be inevitable when the data are collected and maintained by multiple parties. Because of bandwidth or administrative reasons, merging them in a central computing node from those sources might not be feasible \citep{hu2019}.
In some applications, data come naturally feature-distributed, such as the wireless sensor networks \citep{Bertrand2010, Bertrand2014, Bertrand2015}.

A challenge in estimating statistical models with feature-distributed data is to avoid the high \textit{communication complexity}, which is the amount of data that are transmitted across the nodes. Indeed, because distributed computing systems typically operate under limited bandwidth, sending voluminous data significantly slows down the algorithm.
Unfortunately, data transmission is often a necessary evil with feature-distributed data: each node by itself is unable to learn about the parameters associated with the features it does not own.
Thus, algorithms that have lower communication complexities are preferred in practice. 

Based on the rationale that the empirical minimizers of certain optimization problems are desirable statistical estimators, 
prior works have proposed various optimization algorithms with feature-distributed data. \citet{Hydra} and \citet{hydra2} employed randomized coordinate descent to solve $\ell_{1}$-regularized problems and to exploit parallel computation from the distributed computing system. In addition, random projection techniques were used in \cite{Wang2017} and \cite{Heinze16} for $\ell_{2}$-regularized convex problems. 
However, for estimating linear models, the existing approaches usually incur a high communication complexity for very large data sets. To illustrate, consider the Lasso problem. The Hydra algorithm of \cite{Hydra} requires $O(np\log(1/\epsilon))$ bytes of communication to reach $\epsilon$-close to the optimal loss, where $n$ is the sample size and $p$ is the number of features. For data with extremely large $p$ and $n$ that do not fit in a single modern computer, such communication complexity appears prohibitively expensive. 
Similarly, the distributed iterative dual random projection (DIDRP) algorithm of \cite{Wang2017} needs $O(n^2 + n\log(1/\epsilon))$ bytes of total communication for estimating the ridge regression, where the dominating $n^{2}$ factor comes from each node sending the sketched data matrix to a coordinator node. Thus it incurs not only a high communication cost but also a storage bottleneck. 

This paper proposes a two-stage relaxed greedy algorithm (TSRGA) for feature-distributed data to mitigate the high communication complexity. TSRGA first applies the conventional relaxed greedy algorithm (RGA) to feature-distributed data. But we terminate the RGA with the help of a just-in-time stopping criterion, which aims to save excessive communication via reducing RGA iterations. In the second stage, we employ a modification of RGA to estimate the coefficient matrices associated with the selected predictors from the first stage.
The modified second-stage RGA yields low-rank coefficient matrices, that  exploit  information across tasks and improve statistical performance. 

Instead of treating TSRGA as merely an optimization means, we directly analyze the convergence of TSRGA to the unknown parameters, which in turn implies the communication costs of TSRGA. 
The key insight of the proposed method is that the conventional RGA often incurs a high communication cost because it takes many iterations to minimize its loss function, but it tends to select relevant predictors in its early iterations. 
Therefore, one should decide when the RGA has done screening the predictors \textit{before} it iterates too many steps. 
To this end, the just-in-time stopping criterion tracks the reduction in training error in each step, and calls for halting the RGA as soon as the reduction becomes smaller than some threshold. 
With the potential predictors narrowed down in the first stage, the second-stage employs a modified RGA and focuses on the more amenable problem of estimating the coefficient matrices of the screened predictors. 
The two-stage design enables TSRGA to substantially cut down the communication costs and produce even more accurate estimates than the original RGA. 

Our theoretical results show that the proposed TSRGA enjoys a communication complexity of $O_{p}(\mathfrak{s}_{n}(n +d_{n}))$ bytes, up to a multiplicative term depending logarithmically on the problem dimensions, where $d_{n}$ is the dimension of the response vector (or the number of tasks), and $\mathfrak{s}_{n}$ is a sparsity parameter defined later. 
This communication complexity improves that of Hydra by a factor of $p / \mathfrak{s}_{n}$, and is much smaller than that of DIDRP and other one-shot algorithms (for example, \citealt{deco} and \citealt{Heinze16}) if $\mathfrak{s}_{n} \ll n$. The RGA was also employed by \citet{bellet2015} as a solver for $\ell_{1}$-constrained problems, but it requires $O(n/\epsilon)$ communication since it only converges at a sub-linear rate (see also \citealp{jaggi2013} and \citealp{garber2020}), where $\epsilon$ is again the optimization tolerance. 
Hence TSRGA offers a substantial speedup for estimating sparse models compared to the conventional RGA.

To validate the performance of TSRGA, we apply it to both synthetic and real-world data sets and show that TSRGA converges much faster than other existing methods. In the simulation experiments, TSRGA achieved the smallest estimation error using the least number of iterations. 
It also outperforms other centralized iterative algorithms both in speed and statistical accuracy.
In a large-scale simulation experiment, TSRGA can effectively estimate the high-dimensional multivariate linear regression model with more than 16 GB data in 
less than 5 minutes.
For an empirical application, we apply TSRGA to predict simultaneously some financial outcomes (volatility, trading volume, market beta, and returns) of the S\&P 500 component companies using textual features extracted from their 10-K reports. 
The results show that TSRGA efficiently utilizes the information provided by the texts and works well with high dimensional feature matrices.

Finally, we propose some extensions of TSRGA. First, we extend TSRGA to big feature-distributed data which have not only many features but also a large number of observations.
Thus, in addition to separately storing each predictors in different computing nodes, it is also necessary to partition the observations of each feature into chunks that could fit in one node. 
In this case, the computing nodes shall coordinate both horizontally and vertically, and we show that the communication cost to carry out TSRGA in this setting is still free of the feature dimension $p$, but could be larger than that of the purely feature-distributed case.
Second, the idea of TSRGA can be extended beyond linear regression models. 
In Appendix \ref{App:GLM}, we show how TSRGA can be applied to the generalized linear models.

For ease in reading, we collect the notations used throughout the paper here. 
The transpose of a matrix $\mathbf{A}$ is denoted by $\mathbf{A}^{\top}$ 
and that of a vector $\mathbf{v}$ is $\mathbf{v}^{\top}$. The inner product between two vectors $\mathbf{u}$ and $\mathbf{v}$ is denoted interchangeably as $\langle \mathbf{u}, \mathbf{v} \rangle= \mathbf{u}^{\top}\mathbf{v}$. If $\mathbf{A}, \mathbf{B}$ are $\mathbb{R}^{m \times n}$, $\langle \mathbf{A}, \mathbf{B} \rangle = \mathrm{tr}(\mathbf{A}^{\top}\mathbf{B})$ denotes  their trace inner product. The minimum and maximum eigenvalues of a matrix $\mathbf{A}$ are denoted by $\lambda_{\min}(\mathbf{A})$ and $\lambda_{\max}(\mathbf{A})$, respectively.  We also denote by $\sigma_{l}(\mathbf{A})$ the $l$-th singular value of $\mathbf{A}$, in descending order. When the argument is a vector, $\Vert \cdot \Vert$ denotes the usual Euclidean norm and $\Vert \cdot \Vert_{p}$ the $\ell_{p}$ norm. If the argument is a matrix, $\Vert \cdot \Vert_{F}$ denotes the Frobenius norm, $\Vert \cdot \Vert_{op}$ the operator norm, and $\Vert \cdot \Vert_{*}$ the nuclear norm. For a set $J$, $\sharp(J)$ denotes its cardinality.
For an event $\mathcal{E}$, its complement is denoted as $\mathcal{E}^{c}$ and its associated indicator function is denoted as $\mathbf{1}\{\mathcal{E}\}$. For two positive (random) sequences $\{x_{n}\}$ and $\{y_{n}\}$, we write $x_{n}=o_{p}(y_{n})$ if $\lim_{n \rightarrow \infty} \mathbb{P}(x_{n} / y_{n} < \epsilon) = 1$ for any $\epsilon>0$ and write $x_{n}=O_{p}(y_{n})$ if for any $\epsilon >0$ there exists some $M_{\epsilon}<\infty$ such that 
$\limsup_{n \rightarrow \infty} \mathbb{P}(x_{n}/y_{n} > M_{\epsilon}) < \epsilon$.

\section{Distributed framework and two-stage relaxed greedy algorithm}
In this section, we first introduce the multivariate linear regression model considered in the paper and show how the data are distributed across the nodes. Then we lay out the implementation details of the proposed TSRGA, 
which consists of two different implementations of the conventional RGA and a just-in-time stopping criterion 
to guide the termination of the first-stage RGA. The case of needing horizontal 
partition will be discussed in Section~\ref{Sec::horizontal}.

\subsection{Model and distributed framework}
Consider the following multivariate linear regression model:
\begin{align} \label{Model1}
    \mathbf{y}_{t} = \sum_{j=1}^{p_{n}} \mathbf{B}_{j}^{* \top} \mathbf{x}_{t,j} + \bm{\epsilon}_{t}, \quad t=1,\ldots,n,
\end{align}
where $\mathbf{y}_{t} \in \mathbb{R}^{d_{n}}$ is the response vector, $\mathbf{x}_{t,j} \in \mathbb{R}^{q_{n,j}}$ a multivariate predictor, for $j=1,2,\ldots,p_{n}$, and $\mathbf{B}_{j}^{*}$ is the $(q_{n,j} \times d_{n})$ unknown coefficient matrix, for $j = 1,\ldots, p_{n}$. In particular, we are most interested in the case $p_{n} \gg n$ and $q_{n,j} < n$. Clearly, when $d_{n}=q_{n,1}=\ldots=q_{n,p_{n}} = 1$, \eqref{Model1} reduces to the usual multiple linear regression model. Without loss of generality, we assume $\mathbf{y}_t$, $\mathbf{x}_{t,j}$ and $\bm{\epsilon}_{t}$ are mean zero. 

There are several motivations for considering general $d_{n}$ and $q_{n,j}$'s. 
First, imposing group-sparsity can be advantageous when the predictors display a natural grouping structure (e.g. \citealt{lounici2011}).
This advantage is inherited by \eqref{Model1} when only a limited number of $\mathbf{B}_{j}^{*}$'s are non-zero.
Second, it is not uncommon that we are interested in modeling more than one response variable ($d_{n}>1$). In this case, one can gain statistical accuracy if the prediction tasks are related, which is often embodied by the assumption that $\mathbf{B}_{j}^{*}$'s are of low rank (see, e.g., \citealt{RRR}). 
In modern machine learning, some predictors  may be constructed from unstructured data sources. For instance, for functional data, $\mathbf{x}_{t,j}$'s may be the first few Fourier coefficients \citep{Fan2015}. On the other hand, for textual data, $\mathbf{x}_{t,j}$'s may be topic loading or outputs from some pre-trained neural networks \citep{kogan2009, Yeh2020, Bybee2021}.
Finally, model \eqref{Model1} can also accommodate the so-called multi-view of multi-modal data, which have also received considerable attention in recent years.

Next, we specify how the data are distributed across computing nodes. In matrix notations, we can write \eqref{Model1} as
\begin{align}\label{Model1Mtx}
    \mathbf{Y} = \sum_{j=1}^{p_{n}} \mathbf{X}_{j}\mathbf{B}_{j}^{*} + \mathbf{E},
\end{align}
where $\mathbf{Y} = (\mathbf{y}_{1}, \ldots, \mathbf{y}_{n})^{\top}$, $\mathbf{X}_{j} = (\mathbf{x}_{1,j},\ldots, \mathbf{x}_{n,j})^{\top} \in \mathbb{R}^{n \times q_{n,j}}$, for $j=1,2,\ldots,p_{n}$, and $\mathbf{E} = (\bm{\epsilon}_{1}, \ldots, \bm{\epsilon}_{n})^{\top}$. As discussed in the Introduction, since pooling the large matrices $\mathbf{X}_{1}, \ldots, \mathbf{X}_{p_{n}}$ in a central node may not be feasible, a common strategy is to store them across nodes. In the following, we suppose that $M$ nodes are available. Furthermore, the $i$-th node contains the data $\{\mathbf{Y}, \mathbf{X}_{j}: j \in \mathcal{I}_{i}\}$, for $i=1,2,\ldots,M$, where $\cup_{i=1}^{M}\mathcal{I}_{i} = \{1,2,\ldots,p_{n}\}:=[p_{n}]$. For ease in exposition, we assume a master node coordinates the other computing nodes. In particular, each worker node is able to send and receive data from the master node. 

\subsection{First-stage relaxed greedy algorithm and a just-in-time stopping criterion}

We now introduce the first-stage RGA and describe how it can be applied to feature-distributed data. 
First, initialize $\hat{\mathbf{G}}^{(0)} = \mathbf{0}$ and $\hat{\mathbf{U}}^{(0)} = \mathbf{Y}$.
For iteration $k=1,2,\ldots$, RGA finds $(\hat{j}_{k}, \tilde{\mathbf{B}}_{\hat{j}_{k}})$ such that
\begin{align} \label{2.2.1}
    (\hat{j}_{k}, \tilde{\mathbf{B}}_{\hat{j}_{k}}) \in \arg\max_{\substack{1 \leq j \leq p_{n} \\ \Vert \mathbf{B}_{j} \Vert_{*} \leq L_{n}}} \langle \hat{\mathbf{U}}^{(k-1)}, \mathbf{X}_{j}\mathbf{B}_{j} \rangle,
\end{align}
where $L_{n}=d_{n}^{1/2}L_{0}$ for some large constant $L_{0}>0$. Then RGA constructs updates by 
\begin{align} \label{2.2.2}
\begin{split}
    \hat{\mathbf{G}}^{(k)} =& (1 - \hat{\lambda}_{k}) \hat{\mathbf{G}}^{(k-1)} + \hat{\lambda}_{k} \mathbf{X}_{\hat{j}_{k}} \tilde{\mathbf{B}}_{\hat{j}_{k}}, \\
    \hat{\mathbf{U}}^{(k)} =& \mathbf{Y} - \hat{\mathbf{G}}^{(k)},
\end{split}
\end{align}
where $\hat{\lambda}_{k}$ is determined by 
\begin{align} \label{2.2.3}
    \hat{\lambda}_{k} \in \arg\min_{0 \leq \lambda \leq 1} \Vert \mathbf{Y} - (1 - \lambda)\hat{\mathbf{G}}^{(k-1)} - \lambda \mathbf{X}_{\hat{j}_{k}}\tilde{\mathbf{B}}_{\hat{j}_{k}}\Vert_{F}.
\end{align}

RGA has important computational advantages that are attractive for big data computation. 
First, for a fixed $j$, the maximum in \eqref{2.2.1} is achieved at $\mathbf{B}_{j} = L_{n}\mathbf{u}\mathbf{v}^{\top}$, where $(\mathbf{u}, \mathbf{v})$ is the leading pair of singular vectors 
(i.e., corresponding to the largest singular value) of $\mathbf{X}_{j}^{\top}\hat{\mathbf{U}}^{(k-1)}$. 
Since computing the leading singular vectors is much cheaper than full SVD, RGA is computationally lighter than algorithms using singular value soft-thresholding, such as the alternating direction method of multipliers (ADMM). 
This feature has already been exploited in \citet{zheng2018} and \citet{zhuo2020} for nuclear-norm constrained optimization.  
Second, $\hat{\lambda}_{k}$ is easy to compute and has the closed-form $\hat{\lambda}_{k} = \max\{\min\{\hat{\lambda}_{k,uc}, 1\},0\}$, where
\begin{align*}
    \hat{\lambda}_{k,uc} = \frac{\langle\hat{\mathbf{U}}^{(k-1)}, \mathbf{X}_{\hat{j}_{k}}\tilde{\mathbf{B}}_{\hat{j}_{k}} - \hat{\mathbf{G}}^{(k-1)}\rangle}{\Vert \mathbf{X}_{\hat{j}_{k}}\tilde{\mathbf{B}}_{\hat{j}_{k}} - \hat{\mathbf{G}}^{(k-1)} \Vert_{F}^{2}} 
\end{align*}
is the unconstrained minimizer of \eqref{2.2.3}. 

When applied to feature-distributed data, we can leverage these advantages. 
Observe from \eqref{2.2.1}-\eqref{2.2.3} that the history of RGA is encoded in $\hat{\mathbf{G}}^{(k)}$. That is, to construct $\hat{\mathbf{G}}^{(k+1)}$, which predictors were chosen and the order in which they were chosen are irrelevant, provided $\hat{\mathbf{G}}^{(k)}$ is known. 
In particular, each node only needs $\hat{\lambda}_{k+1}$ and $\mathbf{X}_{\hat{j}_{k+1}}\tilde{\mathbf{B}}_{\hat{j}_{k+1}}$ to construct $\hat{\mathbf{G}}^{(k+1)}$. 
As argued in the previous paragraph, $\mathbf{X}_{\hat{j}_{k+1}}\tilde{\mathbf{B}}_{\hat{j}_{k+1}}$ is a rank-one matrix. Thus transmitting this matrix only requires $O(n+d_{n})$ bytes of communication, which are much lighter than that of the full matrix with $O(nd_{n})$ bytes. 
In addition, each node requires only the extra memory to store $\hat{\mathbf{G}}^{(k)}$ throughout the training. 
This is less burdensome than random projection techniques, which require at least one node to make extra room to store the sketched matrix of size $O(n^{2})$.  

The above discussions are summarized in Algorithm \ref{alg:dRGA}, detailing how workers and the master node 
communicate to implement RGA with feature-distributed data. 
Clearly, each node sends and receives data of size $O(n + d_{n})$ bytes (line 4 and 15) in each iteration. We remark that Algorithm \ref{alg:dRGA} asks each node to send the potential updates to the master (line 15). This is for reducing rounds of communications, which can be a bottleneck in practice. If bandwidth limit is more stringent, one can instead first ask the workers to send $\rho_{c}$ to the master. After master decides $c^{*}$, it only asks the $c^{*}$-th node to send the update, so that only one node is transmitting the data. 

\begin{algorithm}[h!] 
\DontPrintSemicolon
  
  \KwInput{Number of maximum iterations $K_{n}$; $L_{n} > 0$.}
  \KwOutput{Each worker $1 \leq c \leq M$ obtains the coefficient matrices $\{ \hat{\mathbf{B}}_{j}: j \in\mathcal{I}_{c} \}$.}
  \KwInit{$\hat{\mathbf{B}}_{j} = \mathbf{0}$ for all $j$ and $\hat{\mathbf{G}}^{(0)} = \mathbf{0}$}
  \For{$k = 1, 2, \ldots, K_{n}$}{
  	\Workers{
		\If{$k>1$}{
		Receive $(c^{*}, \hat{\lambda}_{k-1}, \sigma_{\hat{j}_{k-1}}, \mathbf{u}_{\hat{j}_{k-1}}, \mathbf{v}_{\hat{j}_{k-1}})$ from the master. \\
		$\hat{\mathbf{G}}^{(k-1)} = (1 - \hat{\lambda}_{k-1})\hat{\mathbf{G}}^{(k-2)} + \hat{\lambda}_{k-1}\sigma_{\hat{j}_{k-1}}\mathbf{u}_{\hat{j}_{k-1}}\mathbf{v}_{\hat{j}_{k-1}}^{\top}$. \\
		$\hat{\mathbf{B}}_{j} = (1 - \hat{\lambda}_{k-1}) \hat{\mathbf{B}}_{j}$ for $j \in \mathcal{I}_{c}$. \\
		\If{$c = c^{*}$}{
			$\hat{\mathbf{B}}_{\hat{j}_{k-1}^{(c)}} = \hat{\mathbf{B}}_{\hat{j}_{k-1}^{(c)}} + \hat{\lambda}_{k-1} \tilde{\mathbf{B}}_{\hat{j}_{k-1}^{(c)}}$
			}
		}
		$\hat{\mathbf{U}}^{(k-1)} = \mathbf{Y} - \hat{\mathbf{G}}^{(k-1)}$ \\
		$(\hat{j}_{k}^{(c)}, \tilde{\mathbf{B}}_{\hat{j}_{k}^{(c)}}) \in \arg\max_{\substack{j \in \mathcal{I}_{c} \\ \Vert \mathbf{B}_{j} \Vert_{*} \leq L_{n}}}|\langle \hat{\mathbf{U}}^{(k-1)}, \mathbf{X}_{j}\mathbf{B}_{j} \rangle|$ \\
		$\rho_{c} = |\langle \hat{\mathbf{U}}^{(k-1)}, \mathbf{X}_{\hat{j}_{k}^{(c)}}\tilde{\mathbf{B}}_{\hat{j}_{k}^{(c)}} \rangle|$ \\
		Find the leading singular value decomposition: $\mathbf{X}_{\hat{j}_{k}^{(c)}}\tilde{\mathbf{B}}_{\hat{j}_{k}^{(c)}}  = \sigma_{\hat{j}_{k}^{(c)}}\mathbf{u}_{\hat{j}_{k}^{(c)}}\mathbf{v}_{\hat{j}_{k}^{(c)}}^{\top}$ \\
		Send $(\sigma_{\hat{j}_{k}^{(c)}}, \mathbf{u}_{\hat{j}_{k}^{(c)}}, \mathbf{v}_{\hat{j}_{k}^{(c)}}, \rho_{c})$ to the master.
	}
	\Master{
		Receives $\{(\sigma_{\hat{j}_{k}^{(c)}}, \mathbf{u}_{\hat{j}_{k}^{(c)}}, \mathbf{v}_{\hat{j}_{k}^{(c)}}, \rho_{c}): c=1,2,\ldots,M\}$ from the workers. \\
		$c^{*} = \arg\max_{1\leq c \leq N} \rho_{c}$ \\
		$\sigma_{\hat{j}_{k}} = \sigma_{\hat{j}_{k}^{(c^{*})}}, \mathbf{u}_{\hat{j}_{k}} = \mathbf{u}_{\hat{j}_{k}^{(c^{*})}}, \mathbf{v}_{\hat{j}_{k}} = \mathbf{v}_{\hat{j}_{k}^{(c^{*})}}$ \\
		$\hat{\mathbf{G}}^{(k)} = (1 - \hat{\lambda}_{k}) \hat{\mathbf{G}}^{(k-1)} + \hat{\lambda}_{k} \sigma_{\hat{j}_{k}} \mathbf{u}_{\hat{j}_{k}}\mathbf{v}_{\hat{j}_{k}}^{\top}$,
		where $\hat{\lambda}_{k}$ is determined by 
		\begin{align*} 
    			\hat{\lambda}_{k} \in \arg\min_{0 \leq \lambda \leq 1} \Vert \mathbf{Y} - (1-\lambda) \hat{\mathbf{G}}^{(k-1)} - \lambda\sigma_{\hat{j}_{k}} \mathbf{u}_{\hat{j}_{k}}\mathbf{v}_{\hat{j}_{k}}^{\top} \Vert_{F}^{2}.
		\end{align*}
		\\
		Broadcasts $(c^{*}, \hat{\lambda}_{k}, \sigma_{\hat{j}_{k}}, \mathbf{u}_{\hat{j}_{k}}, \mathbf{v}_{\hat{j}_{k}})$ to all workers.
	}
  }
\caption{Feature-distributed relaxed greedy algorithm (RGA)}\label{alg:dRGA}
\end{algorithm}

Although the per-iteration communication complexity is low for RGA, the total communication can still be costly if the required number of iteration is high.
Indeed, RGA converges to $\arg\min_{\sum_{j=1}^{p_{n}} \Vert \mathbf{B}_{j} \Vert_{*} \leq L_{n}} \Vert \mathbf{Y} - \sum_{j=1}^{p_{n}}\mathbf{X}_{j}\mathbf{B}_{j}\Vert_{F}^{2}$ at the rate $O(k^{-1})$, where $k$ is the number of iterations \citep{jaggi2013, temlyakov2015greedy}. 
There are many attempts to design variants of RGA that converge faster (see \citealp{jaggi2015, qi2019, garber2020} and references therein).
Instead of adapting these increasingly sophisticated optimization schemes with feature-distributed data, we propose to terminate RGA early with the help of a just-in-time stopping criterion.
The key insight, as to be shown in Theorem \ref{3.thm1}, is that RGA is capable of screening relevant predictors in the early iterations. The stopping criterion is defined as follows. Let $\hat{\sigma}_{k}^{2} = (nd_{n})^{-1} \Vert \mathbf{Y} - \hat{\mathbf{G}}^{(k)} \Vert_{F}^{2}$. We terminate the first-stage RGA at step $\hat{k}$, defined as
\begin{align} \label{jit}
    \hat{k} = \min \left\{ 1 \leq k \leq K_{n}: \frac{\hat{\sigma}_{k}^{2}}{\hat{\sigma}_{k-1}^{2}} \geq 1 - t_{n} \right\},
\end{align}
and $\hat{k} = K_{n}$ if $\hat{\sigma}_{k}^{2} / \hat{\sigma}_{k-1}^{2} < 1- t_{n}$, for all $1 \leq k \leq K_{n}$, where $t_{n}$ is some threshold specified later and $K_{n}$ is a prescribed maximum number of iterations. Intuitively, $\hat{k}$ is determined based on whether the current iteration provides sufficient improvement in reducing the training error.
Note that $\hat{k}$ is determined just-in-time without fully iterating $K_{n}$ steps. 
The algorithm is halted once the criterion is triggered, thereby saving excessive communication costs. 
This is in sharp contrast to the model selection criteria used in prior works to terminate greedy-type algorithms that compare all $K_{n}$ models, such as the information criteria \citep{ing2011, ing2020}. 

\subsection{Second-stage relaxed greedy algorithm}

After the first-stage RGA is terminated, the second-stage RGA focuses on estimation of the coefficient matrices. 
In this stage, we implement a modified version of RGA so that the coefficient estimates are of low rank. 

For predictors with ``large'' coefficient matrices, failing to account for their low-rank structure may result in statistical inefficiency. 
To see this, let $\hat{J}:=\hat{J}_{\hat{k}}$ be the predictors selected by the first-stage RGA, and let $\hat{\mathbf{B}}_{j}$, $j \in \hat{J}$, be the corresponding coefficient estimates produced by the first-stage RGA. 
Assume for now $q_{n,j}=q_{n}$.
If $\min\{q_{n}, d_{n}\} > \hat{r} = \sum_{j \in \hat{J}}\hat{r}_{j}$, where $\hat{r}_{j} = \mathrm{rank}(\hat{\mathbf{B}}_{j})$, then estimating this coefficient matrix alone without regularization amounts to estimating $d_{n}q_{n}$ parameters. 
It will be shown later in Theorem \ref{3.thm1} that $\hat{r}_{j} \geq \mathrm{rank}(\mathbf{B}_{j}^{*})$ with probability tending to one.
Since $d_{n}q_{n} \asymp \min\{d_{n}, q_{n}\} (q_{n}+d_{n}) > \hat{r}(q_{n} + d_{n})$, estimating this coefficient matrix would cost us more than the best achievable degrees of freedom \citep{RRR}.

To avoid loss in efficiency for these large coefficient estimators, we impose a constraint on the space in which our final estimators reside.
Suppose the $j$-th predictor, $j \in \hat{J}$, satisfies $\min\{q_{n,j}, d_{n}\} > \hat{r}$. We require its coefficient estimator to be of the form $\hat{\mathbf{\Sigma}}_{j}^{-1} \mathbf{U}_{j}\mathbf{S}\mathbf{V}_{j}^{\top}$, where $\hat{\mathbf{\Sigma}}_{j} = n^{-1}\mathbf{X}_{j}^{\top}\mathbf{X}_{j}$; $\mathbf{U}_{j} = (\mathbf{u}_{1,j}, \ldots, \mathbf{u}_{\hat{r},j})$ and $\mathbf{V}_{j} = (\mathbf{v}_{1,j}, \ldots, \mathbf{v}_{\hat{r},j})$ form the leading $\hat{r}$ pairs of singular vectors of $\mathbf{X}_{j}^{\top}\mathbf{Y}$, and $\mathbf{S}$ is an $\hat{r} \times \hat{r}$ matrix to be optimized. 

The second-stage RGA proceeds as follows. Initialize again $\hat{\mathbf{G}}^{(0)} = \mathbf{0}$ and $\hat{\mathbf{U}}^{(0)} = \mathbf{Y}$. For $k = 1, 2, \ldots$, choose
\begin{align} \label{2.3.1}
    (\hat{j}_{k}, \hat{\mathbf{S}}_{k}) \in \arg\max_{\substack{j \in \hat{J} \\ \Vert \mathbf{S} \Vert_{*} \leq L_{n}}} \langle \hat{\mathbf{U}}^{(k-1)}, \mathbf{X}_{j}\hat{\mathbf{\Sigma}}_{j}^{-1}\mathbf{U}_{j}\mathbf{S}\mathbf{V}_{j}^{\top} \rangle,
\end{align}
where the maximum is searching over $\mathbf{S} \in \mathbb{R}^{\hat{r} \times \hat{r}}$ if $\hat{r} < \min\{q_{n,j}, d_{n}\}$. For $j$ such that $\hat{r} \geq \min\{q_{n,j}, d_{n}\}$, we define $\mathbf{U}_{j}$ and $\mathbf{V}_{j}$ to be the full set of singular vectors and the maximum is searching over $\mathbf{S} \in \mathbb{R}^{q_{n,j} \times d_{n}}$. Next, we construct the update by
\begin{align} \label{2.3.2}
\begin{split}
    \hat{\mathbf{G}}^{(k)} =& (1 - \hat{\lambda}_{k}) \hat{\mathbf{G}}^{(k-1)} + \hat{\lambda}_{k} \mathbf{X}_{\hat{j}_{k}}\hat{\mathbf{\Sigma}}_{\hat{j}_{k}}^{-1}\mathbf{U}_{\hat{j}_{k}}\hat{\mathbf{S}}_{k}\mathbf{V}_{\hat{j}_{k}}^{\top}, \\
    \hat{\mathbf{U}}^{(k)} =& \mathbf{Y} - \hat{\mathbf{G}}^{(k)},
\end{split}
\end{align}
where $\hat{\lambda}_{k}$ is, again, determined by 
\begin{align} \label{2.3.3}
    \hat{\lambda}_{k} \in \arg\min_{0 \leq \lambda \leq 1} \Vert \mathbf{Y} - (1 - \lambda)\hat{\mathbf{G}}^{(k-1)} - \lambda \mathbf{X}_{\hat{j}_{k}}\hat{\mathbf{\Sigma}}_{\hat{j}_{k}}^{-1}\mathbf{U}_{\hat{j}_{k}}\hat{\mathbf{S}}_{k} \mathbf{V}_{\hat{j}_{k}}^{\top} \Vert_{F}^{2}.
\end{align}
At first glance, the updating scheme \eqref{2.3.1}-\eqref{2.3.3} may appear similar to those proposed by \citet{ding2020} or \citet{spec20}, but we note one important difference here: the matrices $\mathbf{U}_{j}$ and $\mathbf{V}_{j}$ are fixed at the onset of the second stage. 
Thus our estimators' ranks remain controlled, which is not the case in the aforementioned works.
More comparisons between TSRGA and these works will be made in Section \ref{sec::main_results}.

We briefly comment on the computational aspects of the second-stage RGA. First, similarly to the first-stage, for a fixed $j$ the maximum in \eqref{2.3.1} is attained at $\mathbf{S} = L_{n}\mathbf{u}\mathbf{v}^{\top}$, where $(\mathbf{u}, \mathbf{v})$ is the leading pair of singular vectors of $\mathbf{U}_{j}^{\top}\hat{\mathbf{\Sigma}}_{j}^{-1}\mathbf{X}_{j}^{\top}\hat{\mathbf{U}}^{(k-1)}\mathbf{V}_{j}$, which can be computed locally by each node. As a result, the per-iteration communication is still $O(n+d_{n})$ for each node. For $j \in \hat{J}$ with $\hat{r} \geq \min\{q_{n,j}, d_{n}\}$, since $\mathbf{U}_{j}$ and $\mathbf{V}_{j}$ are non-singular, the parameter space is not limited except for the bounded nuclear norm constraint. Indeed, it is not difficult to see that for such $j$,
\begin{align*}
    \max_{\Vert \mathbf{S} \Vert_{*} \leq L_{n}} \langle \hat{\mathbf{U}}^{(k-1)}, \mathbf{X}_{j}\hat{\mathbf{\Sigma}}_{j}^{-1}\mathbf{U}_{j}\mathbf{S}\mathbf{V}_{j}^{\top} \rangle
\end{align*}
is equivalent to
\begin{align} \label{2.3.4}
    \max_{\Vert \mathbf{B} \Vert_{*} \leq L_{n}} \langle \hat{\mathbf{U}}^{(k-1)}, \mathbf{X}_{j}\hat{\mathbf{\Sigma}}_{j}^{-1}\mathbf{B} \rangle
\end{align}
with the correspondence $\mathbf{B} = \mathbf{U}_{j}\mathbf{S}\mathbf{V}_{j}^{\top}$. Thus, for such $j$, it is not necessary to compute  the singular vectors $\mathbf{U}_{j}$ and $\mathbf{V}_{j}$. Instead, one can directly solve \eqref{2.3.4}. Finally, it is straightforward to modify Algorithm \ref{alg:dRGA} to implement the second-stage RGA with feature-distributed data. We defer the details to Appendix \ref{sec::secondstage}.

It is worth mentioning that the idea of two-stage RGA can be employed beyond the linear regression setup. 
For example, by replacing the squared loss with log likelihood function, we can use TSRGA to estimate generalized linear models, which include logistic regression for classification tasks and Poisson regression for modeling count data. 
The details of the modified algorithm are deferred to Appendix \ref{App:GLM}, where we also examine its performance through simulations.

\subsection{Related algorithms}
In this subsection, we consider TSRGA in several contexts and compare it with related algorithms.
By viewing TSRGA as either a novel feature-distributed algorithm, an improvement over the Frank-Wolfe algorithm, a new method to estimate the integrative multi-view regression \citep{li2019}, or a close relative of the greedy-type algorithms \citep{temlyakov2000}, we highlight both its computational ease in applying to feature-distributed data and its theoretical applicability in estimating high-dimensional linear models.

Over the last decade, a few methods for estimating linear regression with feature-distributed data have been proposed. 
For instance, \citet{Hydra} and \citet{hydra2} use randomized coordinate descent to solve $\ell_{1}$-regularized optimization problem, and \citet{hu2019} proposes an asynchronous stochastic gradient descent algorithm, to name just a few. 
These methods either require a communication complexity that scales with $p_{n}$, or converge only at sub-linear rates, both of which translate to high communication costs. 
The screen-and-clean approach of \citet{screenandclean}, similar in spirit to TSRGA, first applies sure independence screening \citep[SIS,][]{fan2008sure} to identify a subset of potentially relevant predictors. 
Then it uses an iterative procedure similar to the iterative Hessian sketch \citep{pilanci2016iterative} to estimate the associated coefficients. 
While SIS does not require communication, it imposes stronger assumptions on the predictors and the error term. 
In contrast, the proposed TSRGA can be applied at low communication complexity without succumbing to those assumptions. 

TSRGA also adds to the line of studies that attempt to modify the conventional Frank-Wolfe algorithm \citep{fw1956}. 
RGA, more often called the Frank-Wolfe algorithm in the optimization literature, has been widely adopted in big data applications for its computational simplicity. 
Recently, various modifications of the Frank-Wolfe algorithm have been proposed to attain a linear convergence rate that does not depend on the feature dimension $p_{n}$ \citep{qi2019, garber2020, ding2020, spec20}. 
However, strong convexity or quadratic growth of the loss function is typically assumed in these works, which precludes high-dimensional data ($n \ll p_{n}$). 
Frank-Wolfe algorithm has also been found useful in distributed systems, though most prior works employed the horizontally-partitioned data \citep{zheng2018, zhuo2020}.
That is, data are partitioned and stored across nodes by observations instead of by features. 
A notable exception is \citet{bellet2015}, who found that Frank-Wolfe outperforms ADMM in communication and wall-clock time for sparse scalar regression with feature-distributed data, despite that Frank-Wolfe still suffers from sub-linear convergence.
In this paper, we neither assume strong convexity (or quadratic growth) nor limit ourselves to scalar regression, and TSRGA demands much less computation than the usual Frank-Wolfe algorithm.

Model \eqref{Model1} was also employed by \citet{li2019}, and they termed it the integrative multi-view regression.
They propose an ADMM-based algorithm, integrative reduced-rank regression (iRRR), for optimization in a centralized computing framework. The major drawback, as discussed earlier, is a computationally-expensive step of singular value soft-thresholding. 
Thus, TSRGA can serve as a computationally attractive alternative.
In Section \ref{Sec::simulation}, we compare their empirical performance and find that TSRGA is much more efficient. 

Other closely related greedy algorithms such as the orthogonal greedy algorithm (OGA) have also been applied to high-dimensional linear regression.
OGA, when used in conjunction with an information criterion, attains the optimal prediction error \citep{ing2020} under various sparsity assumptions. 
However, it is computationally less adaptable to feature-distributed data.
To keep the per-iteration communication low, the sequential orthogonalization scheme of \citet{ing2011} can be used with feature-distributed data, but the individual nodes would not have the correct coefficients to use at the prediction time when new data, possibly not orthogonalized, become available.
Alternatively, one needs to allocate extra memory in each node to store the history of the OGA path to compute the projection in each iteration.

\section{Communication complexity of TSRGA}

In this section, we derive theoretical guarantees on the communication complexity of TSRGA.
Specifically, we show that the communication complexity of TSRGA does not scale with the feature dimension $p_{n}$, but instead depends on the sparsity of the underlying problem.


\subsection{Assumptions}

For the theoretical analysis, we maintain the following mild assumptions of model \eqref{Model1}.

\begin{description}
\item[(C1)] There exists some $0 < \mu < \infty$ such that with probability approaching one,
\begin{align*}
    \mu^{-1} \leq \min_{1\leq j \leq p_{n}} \lambda_{\min}(\hat{\mathbf{\Sigma}}_{j}) \leq \max_{1\leq j \leq p_{n}} \lambda_{\max}(\hat{\mathbf{\Sigma}}_{j}) \leq \mu,
\end{align*}
where $\hat{\mathbf{\Sigma}}_{j}  = n^{-1}\mathbf{X}_{j}^{\top}\mathbf{X}_j$ with $\mathbf{X}_j$ being defined in (\ref{Model1Mtx}).
\end{description}

\begin{description}
\item[(C2)] Let $\xi_{E} = \max_{1 \leq j \leq p_{n}} \Vert \mathbf{X}_{j}^{\top}\mathbf{E} \Vert_{op}$. There exists a sequence of $K_{n} \rightarrow \infty$ such that $K_{n}\xi_{E} = O_{p}(nd_{n}^{1/2})$.
\end{description}

\begin{description}
\item[(C3)] 
\begin{align*}
    \lim_{n \rightarrow \infty} \mathbb{P}\left( \min_{\sharp(J) \leq 2K_{n}} \lambda_{\min} (n^{-1} \mathbf{X}(J)^{\top}\mathbf{X}(J)) > \mu^{-1} \right) = 1,
\end{align*}
where $\mathbf{X}(J) = (\mathbf{X}_{j}: j \in J) \in \mathbb{R}^{n \times (\sum_{j\in J}q_{n,j})}$. 
\end{description}

\begin{description}
\item[(C4)] There exists some large $L < \infty$ such that $d_{n}^{-1/2}\sum_{j=1}^{p_{n}}\Vert \mathbf{B}_{j}^{*} \Vert_{*} \leq L$. Moreover, there exists a non-decreasing sequence $\{s_{n}\}$ such that $s_{n}^{2} = o(K_{n})$ and 
\begin{align*}
    \min_{j \in J_{n}} \sigma_{r_{j}^{*}}^{2}\left( d_{n}^{-1/2} \mathbf{B}_{j}^{*} \right) \geq s_{n}^{-1},
\end{align*}
where $J_{n} = \{1 \leq j \leq p_{n}: \mathbf{B}_{j}^{*} \neq \mathbf{0} \}$ is the set of indices corresponding the relevant predictors, and $r_{j}^{*} = \mathrm{rank}(\mathbf{B}_{j}^{*})$.
\end{description}

These assumptions are quite standard.
(C1) requires the variances of the predictors to be on the same order of magnitude, which is often the case if the predictors are normalized. 
$\xi_{E}$ in (C2) is typically regarded as the effect size of the noise. 
Through auxiliary concentration inequalities in the literature, we will verify (C2) in the examples following the main result. 
(C3) assumes a lower bound on the minimum eigenvalue of the covariance matrices formed by small subsets of predictors.
Note that (C3) could hold even when $p_{n} \gg n$ and the observations are dependent; we refer to \citet{ing2011} and \citet{ing2020} for related discussions on (C3).
$s_{n}$ in (C4) imposes a lower bound on the minimum non-zero singular value of the (normalized) coefficient matrices $d_{n}^{-1/2}\mathbf{B}_{j}^{*}$. 
Since (C4) implies $\sharp(J_{n}) \leq s_{n}^{1/2} L$, it can be interpreted as a measure of sparsity of the underlying model.

Next, we introduce two assumptions that are important to the feature-distributed problem.
Let $\tilde{\mathbf{Y}} = \sum_{j=1}^{p_{n}}\mathbf{X}_{j}\mathbf{B}_{j}^{*}$ be the noiseless part of $\mathbf{Y}$.

\begin{description}
\item[(C5)] Let $\bar{r}_{j} = \mathrm{rank}(\mathbf{X}_{j}^{\top}\tilde{\mathbf{Y}})$ and $J_{o} = J_{n} \cap \{j: \min\{q_{n,j}, d_{n}\} > \bar{r}_{j}\}$. 
There exists $\delta_{n}>0$ such that $\xi_{E} = o_{p}(n\delta_{n})$ and with probability approaching one,
\begin{align*}
    \min_{j \in J_{o}} \sigma_{\bar{r}_{j}}(\mathbf{X}_{j}^{\top} \tilde{\mathbf{Y}}) \geq n\delta_{n}.
\end{align*}
\end{description}

\begin{description}
\item[(C6)] (Local revelation) If the column vectors of $\tilde{\mathbf{U}}_{j} \in \mathbb{R}^{q_{n,j} \times \bar{r}_{j}}$ and $\tilde{\mathbf{V}}_{j} \in \mathbb{R}^{d_{n} \times \bar{r}_{j}}$ are the leading pairs of singular vectors corresponding to the non-zero singular values of $\mathbf{X}_{j}^{\top}\tilde{\mathbf{Y}}$, then with probability approaching one, there exists an $\bar{r}_{j} \times \bar{r}_{j}$ matrix $\mathbf{\Lambda}_{j}$ such that
\begin{align} \label{LRC}
    \hat{\mathbf{\Sigma}}_{j}\mathbf{B}_{j}^{*} = \tilde{\mathbf{U}}_{j} \mathbf{\Lambda}_{j} \tilde{\mathbf{V}}_{j}^{\top}
\end{align}
for all $j \in J_{o}$.
\end{description}

(C5) and (C6) are assumptions that endow the local nodes sufficient information in the feature-distributed setting.
Both assumptions concern relevant predictors that are ``large'' such that their dimensions $q_{n,j} \times d_{n}$ satisfy $\min\{q_{n,j}, d_{n}\} > \Bar{r}_{j}$.
Intuitively, (C5) requires, for relevant predictors which are of large dimension, the marginal correlations between these predictors and $\tilde{\mathbf{Y}}$ are sufficiently large. 
The local revelation condition (C6) assumes each node could use its local data to re-construct $\hat{\mathbf{\Sigma}}_{j} \mathbf{B}_{j}^{*}$ for $j \in J_{o}$. 
This would simplify information sharing between the nodes. 
Although they are key assumptions used to derive a fast convergence rate for the second-stage RGA, they are not needed for establishing the sure-screening property of the just-in-time stopping criterion (see Theorem \ref{3.thm1}).
In addition, these two assumptions are vacuous when all predictors are of small dimensions.
For instance, for scalar group-sparse linear regression, $\min\{d_{n},q_{n,j}\} = \min\{1, q_{n,j}\} = 1 \leq \bar{r}_{j}$.
Hence $J_{o} = \emptyset$ and the two assumptions are immaterial. 


To better understand \eqref{LRC}, consider the following example.
\begin{align*}
    \mathbf{y}_{t} = \mathbf{B}_{1}^{* \top} \mathbf{x}_{t,1} + \mathbf{B}_{2}^{* \top} \mathbf{x}_{t,2} + \bm{\epsilon}_{t},
\end{align*}
where $\mathbf{B}_{1}^{*}, \mathbf{B}_{2}^{*}$ are rank-1 matrices such that $\mathbf{B}_{1}^{*} = \mathbf{u}_{1}^{*}\mathbf{v}_{1}^{* \top}$ and $\mathbf{B}_{2}^{*} = \mathbf{u}_{2}^{*}\mathbf{v}_{2}^{* \top}$. 
In matrix notation, we write $\mathbf{Y} = \mathbf{X}_{1}\mathbf{B}_{1}^{*} + \mathbf{X}_{2}\mathbf{B}_{2}^{*} + \mathbf{E}$.
Suppose $q_{n,1} = q_{n,2} > 2$, and consider
\begin{align*}
    \mathbf{X}_{1}^{\top}\Tilde{\mathbf{Y}} = \underbrace{\begin{pmatrix}
        \mathbf{X}_{1}^{\top}\mathbf{X}_{1}\mathbf{u}_{1}^{*} & \mathbf{X}_{1}^{\top}\mathbf{X}_{2}\mathbf{u}_{2}^{*}
    \end{pmatrix}}_{\mathbf{A}} \underbrace{\begin{pmatrix}
        \mathbf{v}_{1}^{* \top} \\
        \mathbf{v}_{2}^{* \top} 
    \end{pmatrix}}_{\mathbf{B}}.
\end{align*}
It is not difficult to show that \eqref{LRC} holds (for $j=1$) if $\mathbf{A}$ and $\mathbf{B}$ are of full rank.
Since $\mathbf{y}_{t} = (\mathbf{x}_{t,1}^{\top}\mathbf{u}_{1}^{*}) \mathbf{v}_{1}^{*} + (\mathbf{x}_{t,2}^{\top}\mathbf{u}_{2}^{*}) \mathbf{v}_{2}^{*}$, one can interpret $f_{t, j} = \mathbf{x}_{t,j}^{\top}\mathbf{u}_{j}^{*}$ as the predictive factor associated with predictor $j$, for $j = 1, 2$. $f_{t,j}$ has differential effects on each element of $\mathbf{y}_{t}$, which are determined by $\mathbf{v}_{j}^{*}$. 
Hence, that $\mathbf{B}$ has full rank translates to that the two factors $f_{t,1}$ and $f_{t,2}$ have distinct impacts on $\mathbf{y}_{t}$. 
On the other hand, $\mathbf{A}$ has full rank if and only if $\mathbf{u}_{1}^{*} \neq \alpha (\mathbf{X}_{1}^{\top}\mathbf{X}_{1})^{-1}\mathbf{X}_{1}^{\top}\mathbf{X}_{2}\mathbf{u}_{2}^{*}$ for any $\alpha \neq 0$. 
This implies the factor $f_{t,1}$ must not be equal to the projection of $f_{t,2}$ onto the space spanned by $\mathbf{X}_{1}$. 
Therefore, \eqref{LRC} can be interpreted as requiring the factors $f_{t,1}$ and $f_{t,2}$ are truly distinct and make distinguishable contributions to the response vector.
Moreover, if \eqref{LRC} fails, the marginal product $\mathbf{X}_{1}^{\top}\tilde{\mathbf{Y}}$ may no longer be useful, because the signals are contaminated by possible collinearity.


\subsection{Main results} \label{sec::main_results}

We now present some theoretical properties of TSRGA, with proofs relegated to Appendix \ref{sec::proofs}. 
In the following, we assume $L_{n}$, the hyperparameter input to the TSRGA algorithm, is chosen to be $L_{n} = d_{n}^{1/2}L_{0}$ with $L_{0} \geq L/(1-\epsilon_{L})$, where $1 - \epsilon_{L} \leq \mu^{-2} / 4$. 

Our first result proves that RGA, coupled with the just-in-time stopping criterion, can screen the relevant predictors. 
Moreover, it provides an upper bound on the rank of the corresponding coefficient matrices.

\begin{theorem} \label{3.thm1}
Assume (C1)-(C4) hold. Suppose there exists an  $M_o<\infty$ such that $M_o^{-1} \leq (nd_{n})^{-1}\Vert \mathbf{E} \Vert_{F}^{2} \leq M_o$ with probability tending to one.
Write $\hat{\mathbf{G}}^{(k)} = \sum_{j=1}^{p_{n}}\mathbf{X}_{j}\hat{\mathbf{B}}_{j}^{(k)}$, $k=1,2,\ldots,K_{n}$, for the iterates of the first-stage RGA.
If $\hat{k}$ is defined by \eqref{jit} with $t_{n} = Cs_{n}^{-2}$ for some sufficiently small $C>0$, then
\begin{align} \label{sure} 
    \lim_{n\rightarrow\infty} \mathbb{P} \left( \mathrm{rank}(\mathbf{B}_{j}^{*}) \leq \mathrm{rank}(\hat{\mathbf{B}}_{j}^{(\hat{k})}) \mbox{ for all }j \right) = 1.
\end{align}
\end{theorem}

Although Theorem \ref{3.thm1} only provides an upper bound for the ranks of $\mathbf{B}_{j}^{*}$'s, it renders a useful diagnosis for the rank of the coefficient matrices for model \eqref{Model1}. 
When $p_{n}=1$, \citet{Bunea2011} proposed a rank selection criterion (RSC) to select the optimal reduced rank estimator, which is shown to be a consistent estimator of the effective rank.
However, rank selection for model \eqref{Model1} with $p_{n} > 1$ is less investigated. 
Moreover, we can bound $\hat{k}$ by the following lemma. 

\begin{lemma} \label{3.lem1}
Under the assumptions of Theorem \ref{3.thm1},  $\hat{k} = O_{p}(s_{n}^{2})$.
\end{lemma}

Lemma \ref{3.lem1} ensures the just-in-time stopping criterion is triggered in no more than $O(s_{n}^{2})$ iterations, which is much smaller than $O(K_{n})$ by (C4). 
Thus compared to the model selection rules using information criteria that iterate $K_{n}$ steps in full, the proposed method greatly reduces communication costs.

Next, we derive the required number of iterations for TSRGA to converge near the unknown parameters, which translates to its communication costs.
With a slight abuse of notation, we also write the second-stage RGA iterates as $\hat{\mathbf{G}}^{(k)} = \sum_{j \in \hat{J}}\mathbf{X}_{j}\hat{\mathbf{B}}_{j}^{(k)}$.

\begin{theorem} \label{3.thm2}
Assume the assumptions of Theorem \ref{3.thm1} hold, and additionally (C5) and (C6) also hold.
If $\xi_{E}=O_{p}(\xi_{n})$ and $m_{n} = \lceil \rho \kappa_{n} \log(n^{2}d_{n}/\xi_{n}^{2}) \rceil$ for some sequence $\{\xi_{n}\}$ of positive numbers, where $\rho = 64\mu^{5}/\tau^{2}$ with $0 < \tau < 1$ being arbitrary, and
\begin{align*}
    \kappa_{n} = \sharp(\hat{J})\max\left\{ \max_{j \in \hat{J} - \hat{J}_{o}}(q_{n,j} \wedge d_{n}), \hat{r}\mathbf{1}\{\hat{J}_{o} \neq \emptyset \} \right\},
\end{align*}
with $a \wedge b = \min\{a, b\}$ and $\hat{J}_{o} = \{j \in \hat{J}: \hat{r} < \min\{q_{n,j}, d_{n}\}\}$, then the proposed second-stage RGA satisfies
\begin{align*} 
    \sup_{m \geq m_{n}} \frac{1}{d_{n}} \sum_{j=1}^{p_{n}} \Vert \mathbf{B}_{j}^{*} - \hat{\mathbf{B}}_{j}^{(m)} \Vert_{F}^{2} = O_{p}\left( \frac{\kappa_{n}\xi_{n}^{2}}{n^{2}d_{n}}\log \frac{n^{2}d_{n}}{\xi_{n}^{2}} + \frac{\xi_{n}^{2}}{n^{2}\delta_{n}^{2}}\mathbf{1}\{J_{o} \neq \emptyset\} \right).
\end{align*}    
\end{theorem}

Since the per-iteration communication cost of TSRGA is $O(n+d_{n})$, Theorem \ref{3.thm2}, together with Lemma \ref{3.lem1}, directly imples the communication complexity of TSRGA, which we state as the following corollary.

\begin{corollary} \label{3.cor1}
If $\kappa_{n} = O_{p}(\mathfrak{s}_{n})$ for some sequence $\{\mathfrak{s}_{n}\}$ of positive numbers, then TSRGA achieves an error of order 
\begin{align*}
     O_{p}\left( \frac{\mathfrak{s}_{n}\xi_{n}^{2}}{n^{2}d_{n}}\log \frac{n^{2}d_{n}}{\xi_{n}^{2}} + \frac{\xi_{n}^{2}}{n^{2}\delta_{n}^{2}}\mathbf{1}\{J_{o} \neq \emptyset\} \right),
\end{align*}
with a communication complexity of order
\begin{align*}
    O_{p}\left( (n + d_{n}) \mathfrak{s}_{n} \log \frac{n^{2}d_{n}}{\xi_{n}^{2}} \right).
\end{align*}
\end{corollary}

Thus, the communication complexity, up to a logarithmic factor, scales mainly with $\mathfrak{s}_{n}$.
In general, Lemma \ref{3.lem1} implies $\mathfrak{s}_{n}=O_{p}(s_{n}^{4})$. 
Thus $\mathfrak{s}_{n}$ is also a measure of the sparsity of the underlying model.
Moreover, in the important special case when the response is a scalar, $\mathfrak{s}_{n} = O_{p}(s_{n}^{2})$ since $d_{n}=1$ and $\hat{J}_{o} = \emptyset$. 
To demonstrate this result more concretely, we discuss the communication complexity of TSRGA when applied to several well-known models below. 

\begin{example}[High-dimensional sparse linear regression]
Consider the model \newline $y_{t} = \sum_{j=1}^{p_{n}}\beta_{j}x_{t,j} + \epsilon_{t}$.
Under suitable conditions, such as $\{\epsilon_{t}\}$ being i.i.d. sub-Gaussian random variables, it can be shown that $\xi_{E} = O_{p}(\sqrt{n\log p_{n}})$ (see, for example, \citealp{ing2011} and \citealp{ing2020}). Then TSRGA achieves an error of order 
\begin{align} \label{example1.1}
    \sum_{j=1}^{p_{n}} |\beta_{j} - \hat{\beta}_{j}|^2 = O_{p}\left(\frac{s_{n}^{2}\log p_{n}}{n}\right)
\end{align}
with a communication complexity of 
\begin{align*}
    O_{p}\left(ns_{n}^{2}\log \frac{n}{\log p_{n}} \right).
\end{align*}
\end{example}

To reach $\epsilon$-close to the minimizer of the Lasso problem, the communication complexity of the Hydra algorithm \citep{Hydra} is 
\begin{align*}
    O\left( \frac{np_{n}}{M\tau} \log \frac{1}{\epsilon} \right),
\end{align*}
where $M$ is the number of nodes and $\tau$ is the number of coordinates to update in each iteration. 
Given limited computational resources, $\tau M$ may still be of order smaller than $p_{n}$.
Thus the communication complexity of TSRGA, which does not scale with $p_{n}$, is more favorable for large data sets with huge $p_{n}$.
In our simulation studies, we also observe that TSRGA converges near $(\beta_{1}, \ldots, \beta_{p_{n}})$ much faster than Hydra-type algorithms. 

\begin{example}[Multi-task linear regression with common relevant predictors]
Suppose we are interested in modeling $T$ tasks simultaneously. 
Let $\mathbf{y}_{1}, \mathbf{y}_{2}, \ldots, \mathbf{y}_{T}$ be the vectors of $n$ observations of the $T$ responses, and $\mathbf{X}$ be the $n \times p$ design matrix consisting of $p$ predictors. Consider the system of linear regressions
\begin{align}
	\mathbf{y}_{t} =& \mathbf{X}\mathbf{b}_{t} + \mathbf{e}_{t},\quad t=1, \ldots, T,  \label{example2.1}
\end{align}
where $\mathbf{b}_{i} = (\beta_{i,1}, \beta_{i,2}, \ldots, \beta_{i,p})^{T}$, for $i=1,2,\ldots,T$, and $\mathbf{e}_{i}$, for $1 \leq i \leq T$, are independent standard Gaussian random vectors. 
Let $\mathbf{x}_{j}$ be the $j$-th column vector of $\mathbf{X}$. Then we may rearrange \eqref{example2.1} as 
\begin{align} \label{example2.2}
	\begin{pmatrix} \mathbf{y}_{1} \\
	\mathbf{y}_{2} \\
	\vdots \\
	\mathbf{y}_{T}
	\end{pmatrix} = \sum_{j=1}^{p}\mathbf{X}_{j}\mathbf{B}_{j} +
		\begin{pmatrix} \mathbf{e}_{1} \\
			\mathbf{e}_{2} \\
			\vdots \\
			\mathbf{e}_{T}
		\end{pmatrix}, 
\end{align}
where $\mathbf{B}_{j} = (\beta_{1,j},\beta_{2,j},\ldots,\beta_{T,j})^{T}$ and $\mathbf{X}_{j} = \mathbf{I}_{T} \otimes \mathbf{x}_{j}$,
where $\mathbf{I}_{T}$ is the $T \times T$ identity matrix and $\mathbf{A} \otimes \mathbf{B}$ denotes the Kronecker product of $\mathbf{A}$ and $\mathbf{B}$.
Now \eqref{example2.2} falls under our general model \eqref{Model1}.
Sparsity of the $\mathbf{B}_{j}$'s promotes that each task is driven by the same small set of predictors, or equivalently, $\mathbf{b}_{j}$'s in \eqref{example2.1} have a common support. 
By a similar argument used in Lemma 3.1 of \citet{lounici2011}, it can be shown that $\xi_{E} = O_{p}(\sqrt{nT(1+T^{-1}\log p)})$. 
Hence Corollary \ref{3.cor1} implies that TSRGA applied to \eqref{example2.2} achieves an error of order 
\begin{align} \label{example2.3}
	\sum_{j=1}^{p} \Vert \mathbf{B}_{j} - \hat{\mathbf{B}}_{j} \Vert^{2} = O_{p} \left( \frac{s_{n}^{2}}{nT} (1 + \frac{\log p}{T}) \right)
\end{align}
with the communication complexity 
\begin{align*}
    O_{p}\left( nTs_{n}^{2} \log \frac{nT}{1 + T^{-1}\log p} \right).
\end{align*}
\end{example}

Notice again that the iteration complexity scales primarily with the strong sparsity parameter $s_{n}$, not with $p$.
As illustrated by \citet{lounici2011}, \eqref{example2.1} can be motivated from a variety of applications, such as the seemingly unrelated regressions (SUR) in econometrics and the conjoint analysis in marketing research.

\begin{example}[Integrative multi-view regression]
Consider the general model \eqref{Model1}, which is called the integrative multi-view regression by \citet{li2019}.
Assume $\mathbf{E}$ has i.i.d. Gaussian entries, and for simplicity that $q_{n,1} = q_{n,2} = \ldots = q_{n,p_{n}} = q_{n}$. 
Then by a similar argument used by \citet{li2019} it follows that $\xi_{E} = O_{p}(\sqrt{n\log p_{n}}(\sqrt{d_{n}} + \sqrt{q_{n}}))$.
Suppose the predictors $\mathbf{X}_{j}$, for $j=1,2,\ldots,p_{n}$, are distributed across computing nodes. TSRGA achieves 
\begin{align} \label{example3.1}
    \frac{1}{d_{n}} \sum_{j=1}^{p_{n}} \Vert \mathbf{B}_{j}^{*} - \hat{\mathbf{B}}_{j} \Vert_{F}^{2} = 
        O_{p}\left( \frac{s_{n}^{4}(d_{n}+q_{n})\log p_{n}}{nd_{n}} + \frac{(d_{n}+q_{n})\log p_{n}}{n\delta_{n}} \right)
\end{align}
with a communication complexity of
\begin{align*}
    O_{p}\left( (n+d_{n})s_{n}^{4} \log \frac{nd_{n}}{(d_{n}+q_{n})\log p_{n}} \right).
\end{align*}
\end{example}

Although \citet{li2019} did not consider the feature-distributed data, they offer an ADMM-based algorithm, iRRR, for estimating \eqref{Model1}. 
However, updating many parameters in each iteration causes significant computational bottleneck.
In our Monte Carlo simulation, iRRR is unable to run efficiently with $p_{n} \geq 50$ even with centralized computing and a moderate sample size, whereas TSRGA can handle such data sizes easily.

In general, the statistical errors of TSRGA in the above examples (\eqref{example1.1}, \eqref{example2.3}, and \eqref{example3.1}) are sub-optimal compared to the minimax rates unless $s_{n}=O(1)$, in which case the model is strongly sparse with a fixed number of relevant predictors. 
One reason is that Theorem \ref{3.thm1} only guarantees sure-screening instead of predictor and rank selection consistency. 
In Examples 1 and 2, the statistical error could be improved if one applies hard-thresholding after the second-stage RGA, and then estimates the coefficients associated with the survived predictors again.
This would not hurt the communication complexity in terms of the order of magnitude since this step takes even less number of iterations.
Nevertheless, in our simulation studies, TSRGA performs on par with and in many cases even outperforms strong benchmarks in the finite-sample case. 

Another reason for the sub-optimality comes from the dependence on $\delta_{n}$ in the error. In the second-stage, TSRGA relies on the sample SVD of the (scaled) marginal covariance $\mathbf{X}_{j}^{\top}\mathbf{Y}$ to estimate the singular subspaces of the unknown coefficient matrices. 
How well these sample singular vectors recover their noiseless counterparts depends on the strength of the marginal covariance, which is controlled by $\delta_{n}$ in Assumption (C5). 
This is needed because we try to avoid searching for the singular subspaces of the coefficient matrices, a challenging task for greedy algorithms.
Unlike the scalar case, for the multivariate linear regression the dictionary for RGA contains all rank-one matrices and therefore the geometric structure is more intricate to exploit.
For example, the argument used in \cite{ing2020} will not work with this dictionary.

Recently, \cite{spec20} and \cite{ding2020} proposed new modifications of the Frank-Wolfe algorithm that directly search within the nuclear norm ball, under the assumptions of strict complementarity and quadratic growth.
These algorithms rely on solving more complicated sub-problems. 
To illustrate one main difference between these modifications and TSRGA, note that for the usual reduced rank regression where $\min\{d_{n}, q_{n,1}\}>1$ and $p_{n}=1$, one of the leading examples in \cite{spec20} and \cite{ding2020}, our theoretical results for TSRGA still hold (though in this case the data are not feature-distributed because $p_{n}$ is only one).
In this case, (C5) and (C6) automatically hold with $\delta_{n} \leq d_{n}^{1/2} / (\mu s_{n}^{1/2})$. 
Consequently, Corollary \ref{3.cor1} implies the error is of order $O_{p}(\frac{s_{n}^{2}\xi_{n}^{2}}{n^{2}d_{n}}\log \frac{n^{2}d_{n}}{\xi_{n}^{2}})$ using $O_{p}(s_{n}^{2} \log \frac{n^{2}d_{n}}{\xi_{n}^{2}})$ iterations, regardless of whether strict complementarity holds.
This advantage precisely comes from that TSRGA uses the singular vectors of $\mathbf{X}_{1}^{\top}\mathbf{Y}$ in its updates in the second stage instead of searching over the intricate space of nuclear norm ball in each iteration.

\section{Simulation experiments} \label{Sec::simulation}
In this section, we apply TSRGA to synthetic data sets and compare its performance with some existing distributed as well as centralized methods. 
We first examine how well TSRGA and other algorithms estimate the unknown parameters.
Then we apply TSRGA to a large-scale feature-distributed data to measure its prowess in speed. 
In both experiments, TSRGA delivered superior performance. 

\subsection{Statistical performance of TSRGA}
In this subsection, we compare the effectiveness of TSRGA in estimating the parameters.
Specifically, it is applied to the well-known high-dimensional linear regression and the general multi-view regression \eqref{Model1Mtx}.


Consider first the high-dimensional linear regression model:
\begin{align*}
    y_{t} = \sum_{j=1}^{p_{n}} \beta_{j}^{*} x_{t,j} + \epsilon_{t},\quad t=1,\ldots, n,
\end{align*}
which is sparse with only $a_{n} = \lfloor p_{n}^{1/3} \rfloor$ non-zero $\beta_{j}^{*}$'s, where $\lfloor x \rfloor$ denotes the integer part of $x$. 
We also generate $\{\epsilon_{t}\}$ as i.i.d. $t$-distributed random variables with five degrees of freedom.

To estimate this model, we employ the Hydra \citep{Hydra} and Hydra$^2$ \citep{hydra2} algorithms to solve the Lasso problem, namely, 
\begin{align} \label{4.1.1}
    \min_{\{\beta_{j}\}_{j=1}^{p_{n}}} \left\{ \frac{1}{2n} \sum_{t=1}^{n} \left( y_{t} - \sum_{j=1}^{p_{n}} \beta_{j}x_{t,j} \right)^{2} + \lambda \sum_{j=1}^{p_{n}}|\beta_{j}| \right\}.
\end{align}
The predictors are divided into 10 groups at random; each of the groups is owned by one node in the Hydra-type algorithm.
The step size of the Hydra-type algorithms is set to the lowest value so that we observe convergence of the algorithms instead of divergence. 
As a benchmark, we also solve the Lasso problem with 5-fold cross validation using \texttt{glmnet} package in R. 
To further reduce the computational burden, we use the $\lambda$ selected by the cross-validated Lasso in implementing Hydra-type algorithms. 

Choosing the hyperparameter for RGA-type methods is more straightforward, but there is one subtlety. 
It is well-known that the Lasso problem corresponds to the constrained minimization problem
\begin{align*}
    \min_{\{\beta_{j}\}_{j=1}^{p_{n}}} \frac{1}{2n} \sum_{t=1}^{n} \left( y_{t} - \sum_{j=1}^{p_{n}} \beta_{j}x_{t,j} \right)^{2} \mbox{ subject to } \sum_{j=1}^{p_{n}}|\beta_{j}| \leq L_{n}.
\end{align*}
Moreover, setting $L_{n}$ to $\sum_{j=1}^{p_{n}}|\beta_{j}^{*}|$, which is nonetheless unknown in practice, would yield the usual Lasso statistical guarantee (see, e.g., Theorem 10.6.1 of \citealp{vershynin_2018}).
However, our theoretical results in Section \ref{sec::main_results} recommend setting $L_{n}$ to a larger value than this conventionally recommended value.
To illustrate the advantage of a larger $L_{n}$, we employ two versions of RGA: one with $L_{n} = 500$ and the other with $L_{n} =  \sum_{j=1}^{p_{n}}|\beta_{j}^{*}|$.
For TSRGA, we simply set $L_{n} = 500$ and $t_{n} = 1 / (10 \log n)$, and the performance is not too sensitive to these choices.

For Specifications 1 and 2 below, we consider three cases with $(n,p_{n}) \in \{(800, 1200), (1200, 2000),$ $(1500, 3000)\}$.
In Specification 1, we simulate the predictors as independent, $t$-distributed data. 
Together with the $t$-distributed errors, this specification simulates the situation where heavy-tailed data are frequently observed.

\begin{spec} \normalfont
In the first experiment, we generate $x_{t,j}$ as i.i.d. $t(6)$ random variables, for all $t = 1,2,\ldots,n$, and $j=1,2,\ldots,p_{n}$. 
Hence the predictors have heavy tails with only 6 finite moments.
The nonzero coefficients are generated independently by $\beta_{j}^{*} = z_{j}u_{j}$, where $z_{j}$ is uniform over $\{-1,+1\}$ and $u_{j}$ is uniform over $[2.5,5.5]$. 
The coefficients are drawn at the start of each of the 100 Monte Carlo simulations. 
\end{spec}

\begin{figure}[h!]
    \centering
    \begin{subfigure}[t]{0.45\textwidth}
        \centering
        \includegraphics[width=\linewidth, height=50mm]{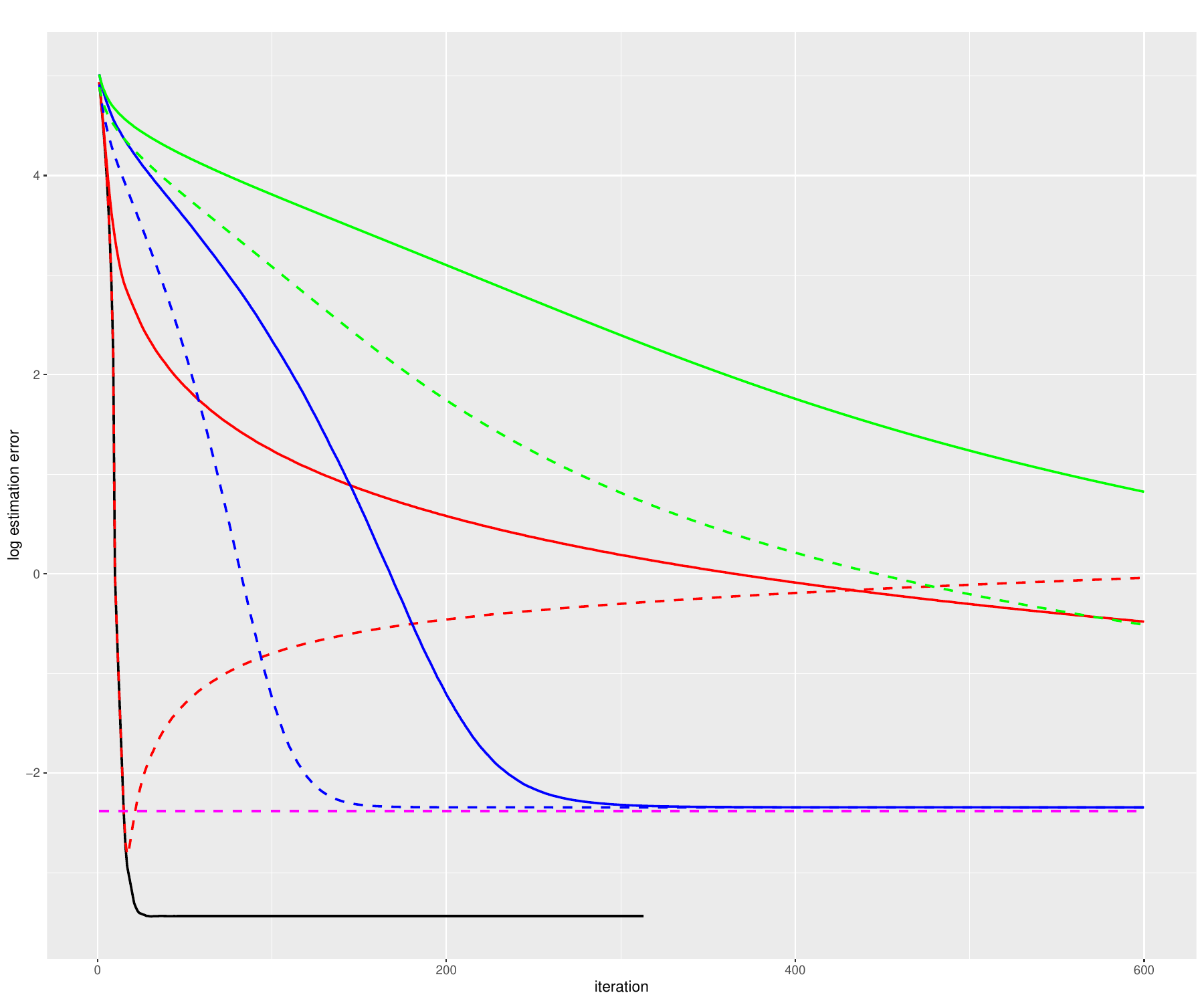}
        \caption{$n=800$, $p_{n}=1200$}    
    \end{subfigure}
    \hfill
    \begin{subfigure}[t]{0.45\textwidth}
        \centering
        \includegraphics[width=\linewidth, height=50mm]{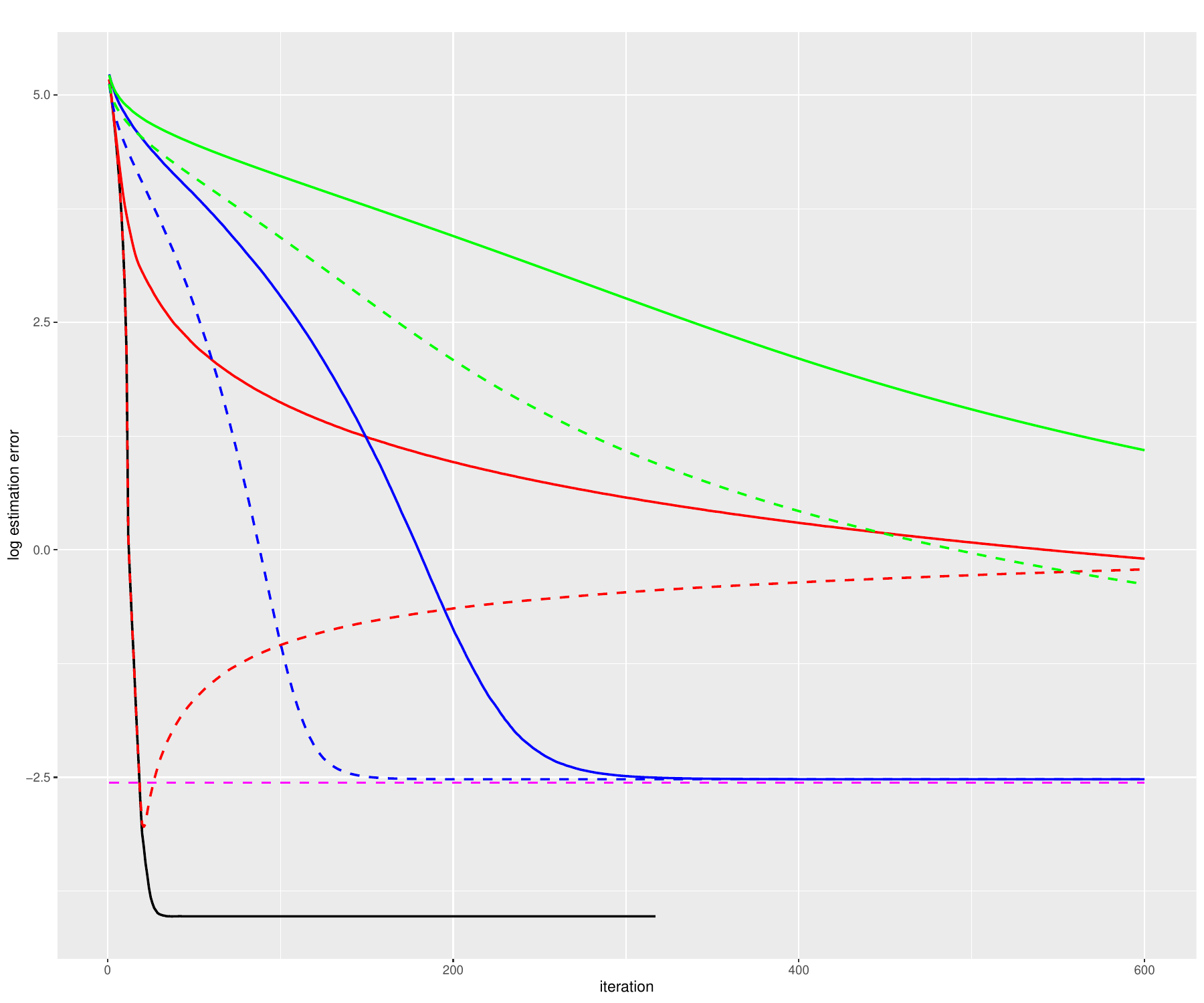}
        \caption{$n=1200$, $p_{n}=2000$}    
    \end{subfigure}
    \hfill
    \begin{subfigure}[t]{\textwidth}
        \centering
        \includegraphics[width=0.8\textwidth, height=85mm]{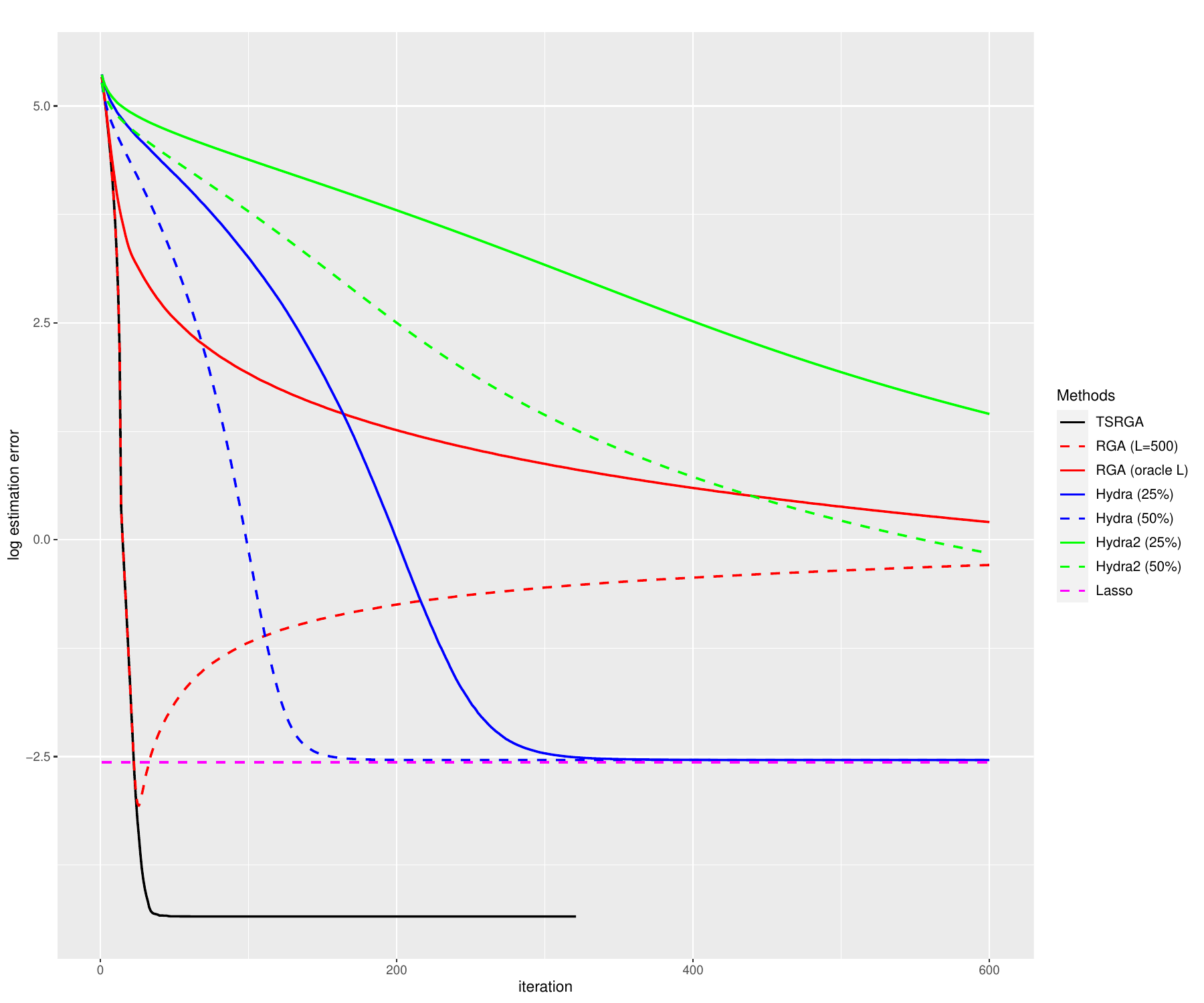}
        \caption{$n=1500$, $p_{n}=3000$}    
    \end{subfigure}
    \caption{Logarithm of parameter estimation errors of various methods under Specification 1, where $n$ is the sample size 
    and $p_n$ is the dimension of predictors. The results are averages of 100 simulations.}
    \label{fig:spec1_b_error}
\end{figure}

Figure \ref{fig:spec1_b_error} plots the logarithm of the parameter estimation error against the number of iterations. 
The parameter estimation error is defined as $\sum_{j=1}^{p_{n}}(\beta_{j}^{*} - \hat{\beta}_{j})^{2}$, where $\{\hat{\beta}_{j}\}$ are the estimates made by the aforementioned methods.
In the plot, the trajectories are averaged across 100 simulations.
TSRGA (black) converges using the least number of iterations.
Since the per-iteration communication costs of TSRGA and Hydra-type algorithms are similar ($O(n)$ bytes), this serves as a proxy for a smaller communication overhead of TSRGA.
In addition, the parameter estimation error of TSRGA is also the smallest among the employed methods.
RGA with $L_{n} = 500$ (dashed red) follows the same trajectories as TSRGA in the first few iterations, but without the two-step design, it suffers from over-fitting in later iterations and hence an increasing parameter estimation error.
On the other hand, RGA with oracle $L_{n} = \sum_{j=1}^{p_{n}}|\beta_{j}^{*}|$ (solid red) converges much slower than TSRGA due to a sub-linear convergence rate.
For Hydra (blue lines) and Hydra$^2$ (green lines) algorithms, we consider two implementations: updating 25\% of the coordinates in each node (solid), and updating 50\% of the coordinates in each node (dashed).
Hydra converges to the centralized Lasso (dashed magenta) at a faster rate if 50\% of the coordinates were updated in each iterations than the 25\% counterparts.
However, Hydra$^2$ converges much slower.

In the next specification, we generate the predictors so that they are correlated and the correlations are the same between any two predictors. This simulates the situation where one cannot simply divide groups of variables that have weak inter-group dependence into different computing nodes to alleviate the difficulties caused by feature-distributed data.

\begin{spec} \normalfont
In this experiment, we generate the predictors by
\begin{align*}
    x_{t,j} = \nu_{t} + w_{t,j},\quad t=1, \ldots, n; \quad j=1, \ldots, p_n,
\end{align*}
where $\{\nu_t\}$ and $\{w_{t,j}\}$ are independent $N(0,1)$ random variables. 
Consequently, $\mathrm{Cor}(x_{t,k}, x_{t,j})$  = 0.5, for $k \neq j$. 
The coefficients are set to $\beta_{j}^{*} = 2.5 + 1.2(j - 1)$ for $j = 1,2, \ldots, a_{n} = \lfloor p_{n}^{1/3} \rfloor$. 
The rest of the specification is the same as that of Specification 1.
\end{spec}

\begin{figure}[h!]
    \centering
    \begin{subfigure}[t]{0.45\textwidth}
        \centering
        \includegraphics[width=\linewidth, height=50mm]{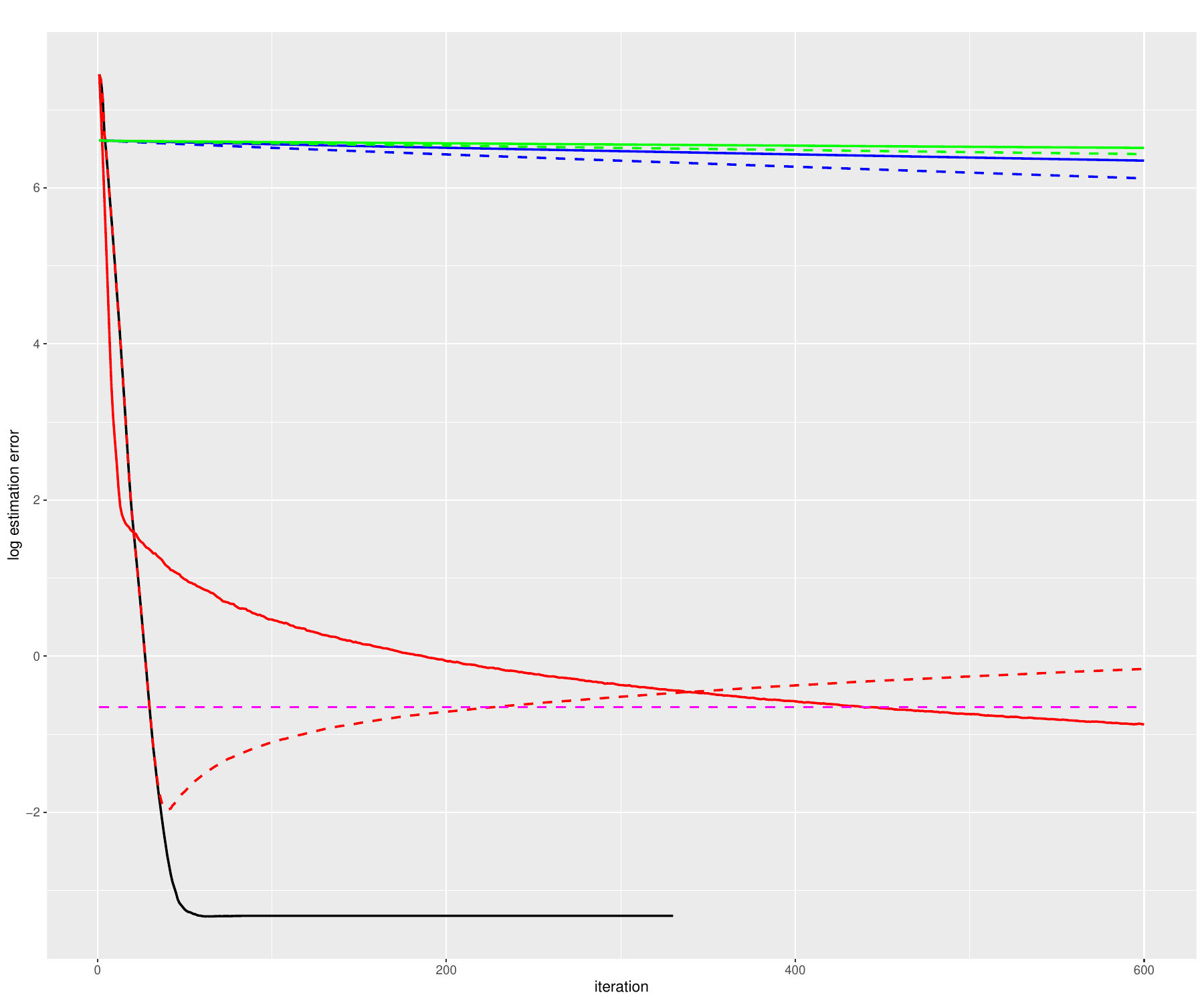}
        \caption{$n=800$, $p_{n}=1200$}    
    \end{subfigure}
    \hfill
    \begin{subfigure}[t]{0.45\textwidth}
        \centering
        \includegraphics[width=\linewidth, height=50mm]{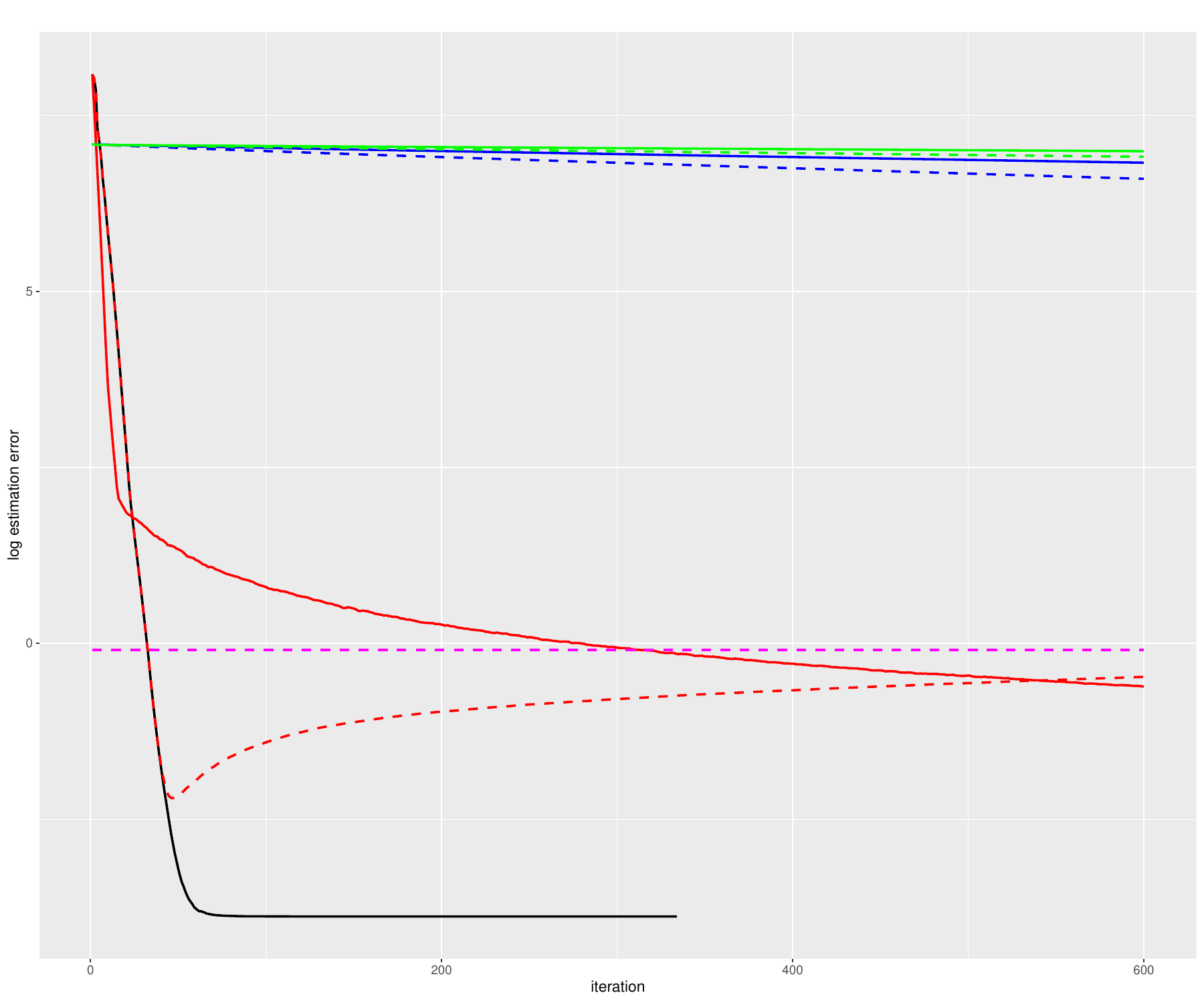}
        \caption{$n=1200$, $p_{n}=2000$}    
    \end{subfigure}
    \hfill
    \begin{subfigure}[t]{\textwidth}
        \centering
        \includegraphics[width=0.8\textwidth, height=85mm]{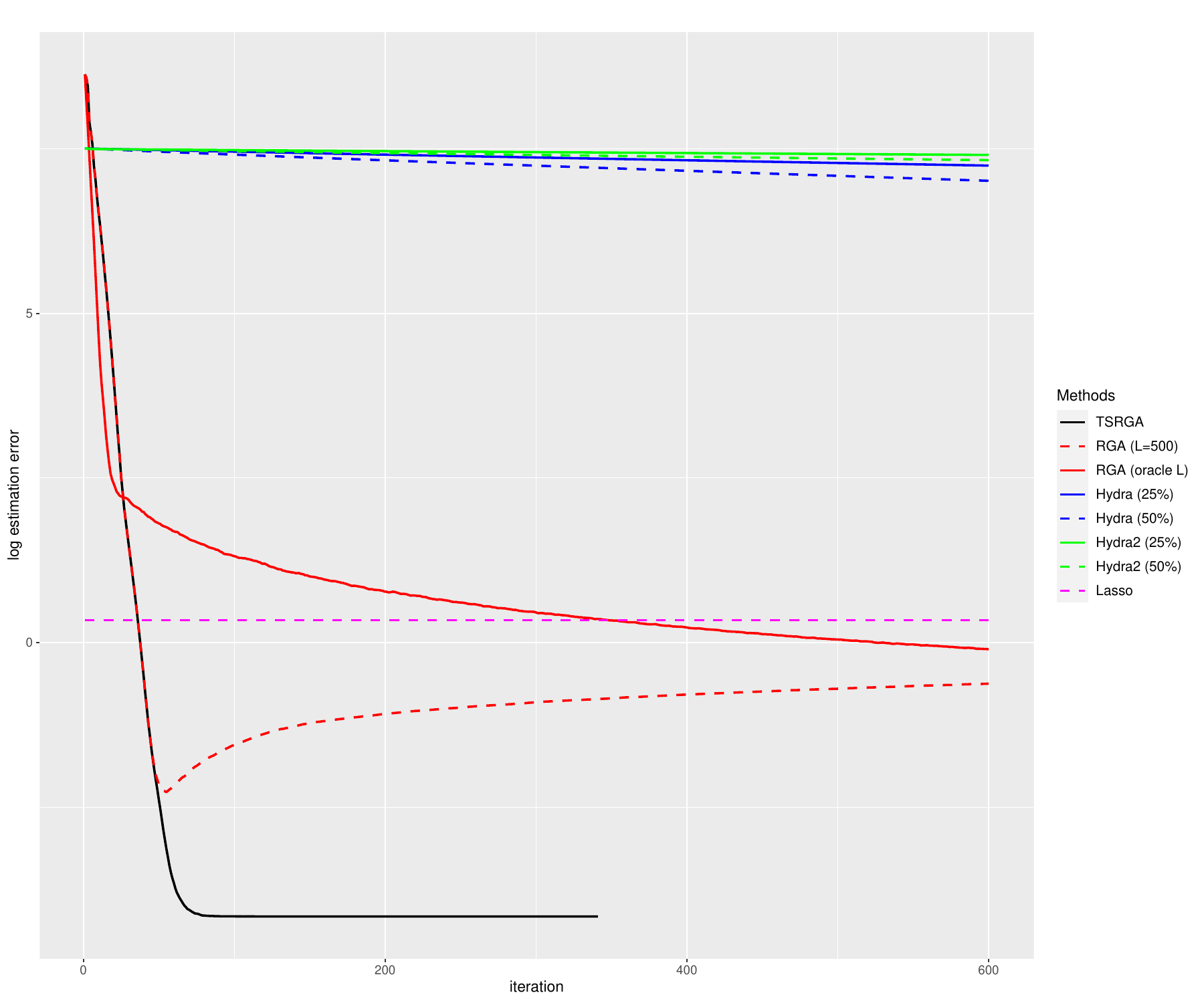}
        \caption{$n=1500$, $p_{n}=3000$}    
    \end{subfigure}
    \caption{Parameter estimation errors of various estimation methods under Specification 2, where 
    $n$ is the sample size and $p_n$ is the number of predictors. The results are averages of 100 simulations.}
    \label{fig:spec2_b_error}
\end{figure}

Figure \ref{fig:spec2_b_error} plots the parameter estimation errors under 
Specification 2.
TSRGA remains the most effective method for estimating the unknown parameters, which converges within 100 iterations in all cases.
It is worth noting that the Hydra-type algorithms display a substantially deteriorated rate of convergence compared to the previous specification, highlighting their sensitivity to the dependence between predictors, and potentially high computational expenses in certain scenarios.

It is also important to study the performance of these methods in terms of elapsed time and out-of-sample performance. 
To save space, we postpone the discussion to Appendix \ref{App_simluation}, as most conclusions drawn above remain valid in examining the elapsed time and prediction performance.

Next we consider the general model: 
\begin{align} \label{4.1.2}
    \mathbf{y}_{t} = \sum_{j=1}^{p_{n}} \mathbf{B}_{j}^{* \top} \mathbf{x}_{t,j} + \bm{\epsilon}_{t},\quad 
    t=1,\ldots,n,
\end{align}
where $\mathbf{y}_{t} \in \mathbb{R}^{d_{n}}$ and $\mathbf{x}_{t,j} \in \mathbb{R}^{q_{n}}$, for $j=1,2,\ldots,p_{n}$.
We generate $\bm{\epsilon}_{t}$ as i.i.d. random vectors with each entry having independent $t(5)$ distributions. 
In the following cases, the model is sparse with $a_{n}$ non-zero $\mathbf{B}_{j}^{*}$'s, each of which is only of rank $r_{n}$.
In particular, we generate $\mathbf{B}_{j}^{*}$, for $j \leq a_{n}$, independently by
\begin{align} \label{4.1.3}
    \mathbf{B}_{j}^{*} = \sum_{k=1}^{r_{n}} \sigma_{k,n} \mathbf{u}_{k,j} \mathbf{v}_{k,j}^{\top},
\end{align}
where $\{\mathbf{u}_{k,j}\}_{k=1}^{r_{n}}$ and $\{\mathbf{v}_{k,j}\}_{k=1}^{r_{n}}$ are independently drawn ($q_{n}$- and $d_{n}$-dimensional) orthonormal vectors and $\{\sigma_{k,n}\}$ are i.i.d. uniform over [7,15].

We employ the iRRR method \citep{li2019} to estimate \eqref{4.1.2}. 
To select its tuning parameter, we execute iRRR with a grid of tuning parameter values and opt for the one with the lowest mean square prediction error on an independently generated validation set of 500 observations.
Although centralized computation is used to implement iRRR, it is too computationally demanding to implement the algorithm for the two cases with $n=600$ and $n=1200$.
Therefore, we use the least squares estimator with only the relevant variables as another benchmark.
For TSRGA, $L_{n}$ is set to $10^{5}$, and we hold one third of the training data as validation set to select the tuning parameter $t_{n}$ for TSRGA over a grid of values\footnote{$t_{n}$ is selected among $\mathbf{t} = (0.01, 0.07, 1.10, 1.39, 1.61, 1.79, 1.95, 2.08, 2.20, 2.30) / \log n$.}. 

Since iRRR is not a feature-distributed algorithm, we directly report their parameter estimation errors (averaged across 500 Monte Carlo simulations) defined as
\begin{align} \label{4.1.4} 
    \sqrt{\sum_{j=1}^{p_{n}}\Vert \mathbf{B}_{j}^{*} - \hat{\mathbf{B}}_{j} \Vert_{F}^{2}},
\end{align}
where $\{ \hat{\mathbf{B}}_{j} \}$ are the estimated coefficient matrices. 
Additionally, the out-of-sample prediction performance of these methods are evaluated on an independent test sample of size 500, measured by $(\Vert \mathbf{Y} - \hat{\mathbf{Y}} \Vert_{F}^{2} / (nd_{n}))^{1/2}$.
We consider the cases $(n,d_{n}, q_{n},p_{n},a_{n},r_{n}) \in \{(200, 10, 12, 20, 1, 2)$, $(400, 15, 18, 50, 2, 2)$, $(600, 20, 25, 400, 3, 2)$, $(1200, 40, 45, 800, 3, 3)\}$.

\begin{spec} \normalfont
In this specification, we consider \eqref{4.1.2} with the predictors generated as in Specification 1. 
Note that $\{\mathbf{B}_{j}^{*}:j \leq a_{n}\}$ are drawn at the start of each of the 500 Monte Carlo simulations.
\end{spec}

Table \ref{tab:spec3} reports the results of the methods averaged over 500 Monte Carlo simulations of data generated under Specification 3.
TSRGA achieved the lowest estimation error in all constellations of problem sizes. 
On the other hand, iRRR yielded larger estimation error than the least squares method using exactly the relevant predictors when $n=200$, but when $n$ increases, iRRR outperforms the least squares method. 
However, the computational costs of iRRR became so high that completing 500 simulations would require more than days, even when parallelism with 15 cores is used.
TSRGA circumvents such computational overhead and delivers superior estimates. 
The prediction errors suggest the same conclusions even though the difference is less significant.

\begin{table}[htb]
\centering
\begin{tabular}{@{}ccccccc@{}}
\toprule
 & \multicolumn{3}{c}{Parameter estimation} & \multicolumn{3}{c}{Prediction} \\ \cmidrule(lr){2-4} \cmidrule(lr){5-7}
$(n,d_{n},q_{n},p_{n},a_{n},r_{n})$             & TSRGA & iRRR  & Oracle LS  & TSRGA & iRRR  & Oracle LS \\ \midrule
(200, 10, 12, 20, 1, 2) & 0.666 & 0.929 & 0.851 & 1.138 & 1.339 & 1.331 \\ 
(400, 15, 18, 50, 2, 2) & 0.858 & 1.245 & 1.287 & 1.322 & 1.351 & 1.355 \\ 
(600, 20, 25, 400, 3, 2) & 1.223 & - & 1.787 & 1.361 & - & 1.381 \\
(1200, 40, 45, 800, 3, 3) & 1.388 & - & 2.378 & 1.345 & - & 1.371 \\\bottomrule
\end{tabular}
\caption{Parameter estimation and prediction errors of various methods under Specification 3. We do not report the results for iRRR with sample sizes of 600 and 1200 since the computation required for these cases is excessively time-consuming. In the table, $n,d_n,q_n,p_n,a_n$ and $r_n$ are the sample size, number of targeted variables, dimension of predictors, 
number of predictors, number of non-zero coefficient matrices, and rank of coefficient matrices, respectively. The results are based on 500 simulations.}
\label{tab:spec3}
\end{table}

\begin{spec} \normalfont
In this specification, we generalize \eqref{4.1.2} to group predictors as follows.
Let $\{\bm{\nu}_{t}:t=1,2,\ldots\}$ and $\{\mathbf{w}_{t,j}:t=1,2,\ldots; j=1,2,\ldots, p_{n}\}$ be independent $N(\mathbf{0}, \mathbf{I}_{q_n})$ random vectors.
The group predictors are then constructed as $\mathbf{x}_{t,j} = 2\bm{\nu}_{t} + \mathbf{w}_{t,j}$, $1\leq t \leq n$, $1\leq j \leq p_{n}$.
Hence $\mathbb{E}(\mathbf{x}_{t,j}\mathbf{x}_{t,i}^{\top}) = 4\mathbf{I}_{q_{n}}$, for $1 \leq i < j \leq p_{n}$.
Note that $\mathrm{Corr}(x_{t,i,l}, x_{t,j,l}) = 0.8$ for $i \neq j$, $1 \leq l \leq q_{n}$, where $\mathbf{x}_{t,i} = (x_{t,i,1}, \ldots, x_{t,i,q_{n}})^{\top}$.
Hence, the $l$-th components in each of the group predictors are highly correlated.
\end{spec}

\begin{table}[htb]
\centering
\begin{tabular}{@{}ccccccc@{}}
\toprule
 & \multicolumn{3}{c}{Parameter estimation} & \multicolumn{3}{c}{Prediction} \\ \cmidrule(lr){2-4} \cmidrule(lr){5-7}
$(n,d_{n},q_{n},p_{n},a_{n},r_{n})$             & TSRGA & iRRR  & Oracle LS  & TSRGA & iRRR  & Oracle LS \\ \midrule
(200, 10, 12, 20, 1, 2) & 0.401 & 0.616 & 0.460 & 1.324 & 1.337 & 1.330 \\ 
(400, 15, 18, 50, 2, 2) & 0.562 & 0.993 & 1.172 & 1.345 & 1.344 & 1.354 \\ 
(600, 20, 25, 400, 3, 2) & 0.812 & - & 1.817 & 1.362 & - & 1.379 \\
(1200, 40, 45, 800, 3, 3) & 0.751 & - & 2.419 & 1.310 & - & 1.371 \\\bottomrule
\end{tabular}
\caption{Parameter estimation and prediction errors under Specification 4. We do not report the results for iRRR with sample sizes of 600 and 1200 since the computation required for these sample sizes is excessively time-consuming. The same notations as those of Table~\ref{tab:spec3} are used. The results are based on 500 simulations.}
\label{tab:spec4}
\end{table}

Table \ref{tab:spec4} reports the results for Specification 4. 
As in the previous specifications, TSRGA continues to surpass the benchmarks.
When $n=400$, iRRR gains an advantage over the least squares method, despite of a high computational cost.
The results in Tables \ref{tab:spec3} and \ref{tab:spec4} suggest that TSRGA is both a fast and a statistically effective tool for parameter estimation for model \eqref{4.1.2}.

\subsection{Large-scale performance of TSRGA}

In this subsection, we apply TSRGA to large feature-distributed data. 
We have an MPI implementation of TSRGA through OpenMPI and the \texttt{Python} binding \texttt{mpi4py} \citep{mpi4py1, mpi4py2}.
The algorithm runs on the high-performance computing cluster of the university, which comprises multiple computing nodes equipped with Intel Xeon Gold 6248R processors.
We consider again Specification 4 in the previous subsection, with $(n, d_{n}, q_{n}, p_{n}, a_{n}, r_{n}) = (20000, 100, 100, 1024, 4, 4)$.
In the following experiments we employ $M/4$ nodes, each of which runs 4 processes and each process owns $p_{n}/M$ predictors, with $M$ varying from 16 to 64. 
When combined, the data are approximately over 16 GB of size, exceeding the usual RAM capacity on most laptops. 

There are two primary goals for the experiments.
The first goal is to investigate the wall-clock time required by TSRGA to estimate \eqref{4.1.2}.
The second goal is to examine the effect of the number of nodes on the required wall-clock time.
Each experiment is repeated 10 times, and we average the wall-clock time needed to complete the $k$-th iteration as well as the parameter estimation error \eqref{4.1.4} at the $k$-th iteration.

\begin{figure}
    \centering
    \includegraphics[width=0.7\textwidth]{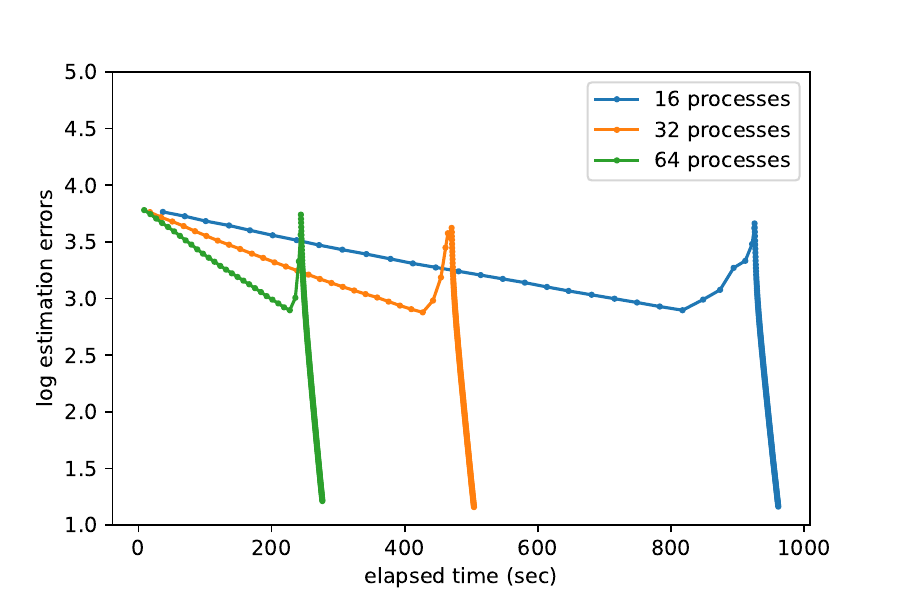}
    \caption{Logarithm of the average parameter estimation errors at each iteration of TSRGA, plotted against the average time elapsed at the end of each iteration. Various number of processes are employed for feature-distributed implementation. 
    10 simulations are used.}
    \label{fig:large_scale_time}
\end{figure}

Figure \ref{fig:large_scale_time} plots the (log) estimation errors against the wall-clock time of TSRGA iterations. 
When using 16 processes, TSRGA took about 16 minutes to estimate \eqref{4.1.2}, and the time reduced to less than 5 minutes when 64 processes were employed.
The acceleration primarily occurred in the first stage, because solving \eqref{2.2.1} becomes faster when each process handles only a small number of predictors.
After screening, there is a spike in estimation error due to re-initialization of the estimators but subsequent second-stage RGA runs extremely fast in all cases and yields accurate estimates.
Indeed, Figure \ref{fig:large_scale_size} shows that the estimation error of TSRGA quickly drops below that of the oracle least squares in the second stage.
We remark that with more diligent programming, one can apply the advanced protocols introduced in Section 6 of \citet{Hydra} to TSRGA, using both multi-process and multi-thread techniques.
It is anticipated that the required time will be further shortened.

\begin{figure}
    \centering
    \includegraphics[width=0.6\textwidth]{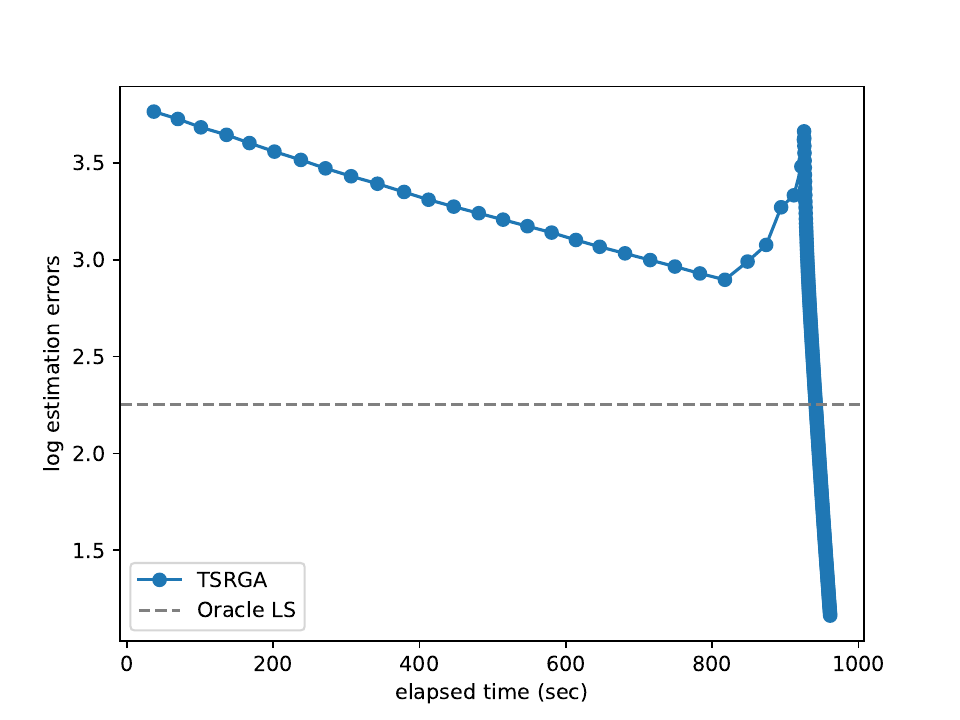}
    \caption{Logarithm of the estimation errors of TSRGA (running with 16 processes) and the oracle least squares. The oracle least squares method is performed by applying the second-stage RGA with exactly the relevant predictors and no rank constraints. 10 simulations are used.}
    \label{fig:large_scale_size}
\end{figure}

\section{Empirical application}
This section showcases an application of TSRGA to financial data. 
In addition to the conventional financial data, we further collect the annual 10-K reports of firms under study to extract useful features for augmenting the predictors. 
Thus, in this application, both the response and predictors are multivariate, and the predictors may consist of large dense matrices, leading to potential computational challenges in practice. 

\subsection{Financial data and 10-K reports}

We aim to predict four key financial outcomes for companies in the S\&P 500 index: 
volatility, trading volume, market beta, and return. 
We obtain daily return series for each company from 2010 through 2019, calculate the sample 
variances of the daily returns in each month, and transform them by taking the logarithm to get the 
volatility series $\{V_{it}(m):m=1,2,\ldots,12\}$ for the $i$-th company in the $m$-th month 
of year $t \in \{2010, \ldots, 2019\}$.
Next, we regress each company's daily returns on the daily returns of the S\&P 500 index for each month and use the slope estimates as market beta, $\{B_{it}(m):m=1,2,\ldots,12\}$.
Finally, we also obtain data of the monthly returns series $\{R_{it}(m):m=1,2,\ldots,12\}$ and the logarithm of the trading volumes $\{M_{it}(m):m=1,2,\ldots,12\}$, for the $i$th company. 
All series are obtained from Yahoo! Finance via the \texttt{tidyquant} package in R.

After obtaining these series, some data cleaning is performed to facilitate subsequent analysis.
First, the volume series exhibits a high degree of serial dependence, which could be due to unit-roots caused by the 
persistence in trading activities. Therefore, 
we apply a year-to-year difference, i.e., $\Delta M_{it}(m) = M_{i,t}(m) - M_{i,t-1}(m)$ for all $i$, 
$1\leq m \leq 12$, and $t = 2011, \ldots, 2019$. 
Additionally, we remove companies that have outlying values in these series.


In addition to these financial time series, we also make use of the information from a pertinent collection of textual data: 10-K reports.
Publicly traded companies in the U.S. are required to file these annual reports with the aim of increasing 
transparency and satisfying the regulation of exchanges.
The reports are maintained by the Securities and Exchange Commission (SEC) in the Electronic Data Gathering, Analysis, and Retrieval system (EDGAR), and provide information about a company's risks, liabilities, and corporate agreements and operations.
Due to their significance in communicating information to the public, the 10-K reports are an important corpus in finance, economics, and computational social sciences studies (\citealp{hanley2019, kogan2009, gandhi2019, Jegadeesh2013}).

The corpus utilized in this application is sourced from the EDGAR-CORPUS, originally prepared by \citet{edgar2021}. 
Our analysis specifically focuses on Section 7, titled ``Management's Discussion and Analysis.''
To process the reports, we preprocess each document using the default functionality in the \texttt{gensim} package in Python and discard the documents that consist of fewer than 50 tokens. 
As a result, we have data of both the financial time series and 10-K reports of 256 companies over the period from 2011 through 2019.

To extract features from the textual data, we employ a technique called Latent Semantic Indexing (LSI, see, e.g., \citealp{Deerwester1990}). 
We first construct the term-document matrix as follows. 
Suppose we have $D$ documents in the training set, and there are $V$ distinct tokens in these documents. The term-document matrix $\mathbf{\Theta}$ is a $V \times D$ matrix, whose entries are given by
\begin{align*}
    \mathbf{\Theta}_{ij} = &(\mbox{number of times the } i\mbox{-th token appears in document } j) \times \\
    &\log \frac{D}{\sharp\{1 \leq k \leq D: \mbox{the } i \mbox{-th token appears in document } k\}},
\end{align*}
for $1 \leq i \leq V$, $1 \leq j \leq D$. The entries are known as one form of the term-frequency inverse document frequency (TFIDF, see, e.g., \citealp{Salton1988}). 
Then, to extract $K$ features from the text data, LSI uses the singular value decomposition,
\begin{align*}
    \mathbf{\Theta} = \mathbf{U}_{\mathbf{\Theta}}\mathbf{\Sigma}_{\mathbf{\Theta}}\mathbf{V}_{\mathbf{\Theta}}^{\top},
\end{align*}
and the first $K$ rows of $\mathbf{\Sigma}_{\mathbf{\Theta}}\mathbf{V}_{\mathbf{\Theta}}^{\top}$ are used as the features in the training set.
For a new document in the test set, we compute its TFIDF representation $\bm{\theta} \in \mathbb{R}^{V}$, and then use $\mathbf{x} = \mathbf{U}_{K}^{\top} \bm{\theta}$ as its textual features, where $\mathbf{U}_{K}$ is the sub-matrix of the first $K$ columns of $\mathbf{U}_{\mathbf{\Theta}}$. 

\subsection{Results}

For each of the four financial response variables, we estimate the following model.
\begin{align} \label{5.2.1}
    \mathbf{y}_{it} = \bm{\beta}_{0} + \mathbf{A}_{1}^{\top} \mathbf{v}_{i,t-1} + \mathbf{A}_{2}^{\top} \mathbf{m}_{i,t-1} + \mathbf{A}_{3}^{\top} \mathbf{b}_{i,t-1} + \mathbf{A}_{4}^{\top} \mathbf{r}_{i,t-1} + \mathbf{A}_{5}^{\top} \mathbf{x}_{i,t-1} + \bm{\epsilon}_{it},
\end{align}
where $\mathbf{y}_{it} = (y_{it}(1), \ldots, y_{it}(12))^{\top}$ is the response variable under study, $\mathbf{v}_{it} = (V_{it}(1), \ldots, V_{it}(12))^{\top}$, $\mathbf{m}_{it} = (\Delta M_{it}(1), \ldots, \Delta M_{it}(12))^{\top}$, 
$\mathbf{b}_{it} = (B_{it}(1),\ldots,B_{it}(12))^{\top}$, $\mathbf{r}_{it}=(R_{it}(1), \ldots, R_{it}(12))^{\top}$, $\mathbf{x}_{it} \in \mathbb{R}^{K}$ is the extracted text features, and $\{\bm{\beta}_{0}, \mathbf{A}_{1}, \ldots, \mathbf{A}_{5}\}$ are unknown parameters. 
When predicting each of the four financial outcomes, we replace $\mathbf{y}_{it}$ in \eqref{5.2.1} with the corresponding vector ($\mathbf{v}_{it}$, $\mathbf{m}_{it}$, $\mathbf{b}_{it}$, or $\mathbf{r}_{it}$), while keeping the same model structure.
Since predicting next-year's financial outcomes in one month is related to predicting the same variable in other months, it is natural to expect low-rank coefficient matrices. 
\eqref{5.2.1} can also be viewed as a multi-step ahead prediction model, since we are predicting the next twelve months simultaneously.

When applying TSRGA to \eqref{5.2.1}, we use a hold-out validation set from the training sample to select the just-in-time threshold $t_{n}$ from the grid $(0.1, 0.2, \ldots, 1.0) / \log n$.
In addition to TSRGA, we employ several benchmark prediction methods, including the vector autoregression (VAR), reduced rank regression (RR; see, e.g., \citealp{chen2013}), the integrative reduced rank regression (iRRR, \citealp{li2019}), and the Lasso.
For VAR, we concatenate all response variables and estimate the model
\begin{align*}
    \mathbf{z}_{it} = \mathbf{A}^{\top}\mathbf{z}_{i,t-1} + \mathbf{e}_{it},
\end{align*}
where $\mathbf{z}_{it} = (\mathbf{v}_{it}^{\top}, \mathbf{m}_{it}^{\top}, \mathbf{b}_{it}^{\top}, \mathbf{r}_{it}^{\top})^{\top} \in \mathbb{R}^{48}$.
Alternatively, we can implement VAR in a group-wise fashion (gVAR henceforth). Specifically, we separately estimate the model
\begin{align} \label{5.2.2}
    \mathbf{y}_{it} = \mathbf{A}^{\top}\mathbf{y}_{i,t-1} + \mathbf{e}_{it},
\end{align}
for each response variable $\mathbf{y}_{it} \in \{\mathbf{v}_{it}, \mathbf{m}_{it}, \mathbf{b}_{it}, \mathbf{r}_{it}\}$.  
The reduced rank regression also estimates \eqref{5.2.2} with an intercept term and an additional rank constraint on the coefficient matrix $\mathbf{A}$ in \eqref{5.2.2}.
We use the generalized cross validation (GCV, \citealp{Golub1979}) to select the optimal rank.
For Lasso, it is applied separately to each row of \eqref{5.2.1}; namely, it estimates 
\begin{align*}
    y_{it}(m) = \beta_{0} + \sum_{j=1}^{12} & \alpha_{j,1}V_{i,t-1}(j) + \sum_{j=1}^{12}\alpha_{j,2}\Delta M_{i,t-1}(j)  \\
    & + \sum_{j=1}^{12}\alpha_{j,3}B_{i,t-1}(j) + \sum_{j=1}^{12}\alpha_{j,4}R_{i,t-1}(j) + \epsilon_{it},
\end{align*}
for $m=1,2,\ldots,12$. Finally, we also apply the iRRR method of \cite{li2019} to \eqref{5.2.1}.

Table \ref{tab::5.2} presents the root mean squared prediction errors (RMSE) for different methods on the test set, for which we reserved the last year of data. 
The results show that gVAR consistently outperformed the usual VAR in all four financial variables, suggesting using simple least squares could be harmful in prediction when including many financial series as predictors.
RR provides a slight improvement in predicting volatility, but performs similarly as VAR and gVAR in predicting other targets. 
In the case of predicting volatility, the text data proved to be quite useful, and TSRGA, iRRR, and Lasso have all outperformed gVAR by more than 5\% with different number of textual features $K$ (except for Lasso with $K=50$).
TSRGA and iRRR, utilizing both the text information and low-rank coefficient estimates, yielded the smallest prediction errors.
In some cases, they achieved 10\% reduction in RMSE compared with gVAR and RR.
For the rest of the targets, the methods did not perform very differently from gVAR and RR.

\begin{table}[htb!]
\centering
\begin{tabular}{@{}lllll@{}}
\toprule
      & Volatility      & Volume      & Beta      & Return      \\
VAR   & 0.782 & 0.323 & 0.583 & 0.077  \\
gVAR  & 0.750 & 0.319 & 0.556 & 0.073 \\ 
RR    & 0.732 & 0.325 & 0.555 & 0.071 \\ \midrule
\multicolumn{5}{l}{$K=50$}                  \\
Lasso & 0.718 & 0.310 & 0.574  & 0.075  \\
iRRR  & \bf{0.688} & 0.318 & 0.568 & 0.072 \\
TSRGA & \bf{0.702} & 0.345 & 0.572 & 0.072 \\
\multicolumn{5}{l}{$K=100$}                 \\
Lasso & \bf{0.700} & 0.308 & 0.574 & 0.074 \\
iRRR  & \bf{0.677} & 0.316 & 0.566 & 0.072 \\
TSRGA & \bf{0.678} & 0.330 & 0.571 & 0.072 \\
\multicolumn{5}{l}{$K=150$}                 \\
Lasso & \bf{0.693} & 0.308 & 0.571 & 0.073 \\
iRRR  & \bf{0.667}$^a$ & 0.314 & 0.566 & 0.072 \\
TSRGA & \bf{0.681} & 0.332 & 0.573 & 0.072 \\
\multicolumn{5}{l}{$K=200$}                 \\
Lasso & \bf{0.684} & 0.309 & 0.574 & 0.073 \\
iRRR  & \bf{0.663}$^a$ & 0.314 & 0.567 & 0.072 \\ 
TSRGA & \bf{0.654}$^{a,b}$ & 0.345 & 0.574 & 0.072 \\ \bottomrule
\end{tabular}
\caption{Root mean squared prediction errors on the test dataset. Entries in boldface are at least $5\%$ below gVAR; $^{a}$ means $10\%$ below gVAR, and $^{b}$ means $10\%$ below RR.}
\label{tab::5.2}
\end{table}

In addition to the prediction performance, we make two more remarks on the empirical results.
First, our finding that textual features are useful in predicting volatility is consistent with previous studies. 
For instance, \citet{kogan2009} reported that one-hot text features are already effective in predicting volatility in a scalar linear regression, and \cite{Yeh2020} also observed gains of using neural word embedding to predict volatility. 
Our results suggest an alternative modeling choice: text data could explain each month's volatility via a low-rank channel. 
Second, low-rank models may not be suited for the trading volume series. 
The RR selected a full-rank model and TSRGA iterated more steps before the just-in-time stopping criterion was triggered.

The data set used in the application is relatively small, and can fit in most personal computer's memory. 
However, incorporating more sections of the 10-K reports or other financial corpus may pose computational challenges due to the increased number of dense text feature matrices.
TSRGA can easily handle such cases when feature-distributed data are inevitable.

\section{Horizontal partition for big feature-distributed data} \label{Sec::horizontal}

In this section, we briefly discuss the usage of TSRGA when the sample size $n$, in addition to the dimension $p_{n}$, is also large so that storing $(\mathbf{Y}, \mathbf{X}_{j})$ in one machine is infeasible. In this case, we also horizontally partition the (feature-distributed) data matrices and employ more computing nodes. 

To fix ideas, for $h=1,2,\ldots,H$, let
\begin{align*}
    \mathbf{Y}_{(h)} = (\mathbf{y}_{m_{h-1}+1}, \ldots, \mathbf{y}_{m_{h}})^{\top} \mbox{, and} \quad
    \mathbf{X}_{j,(h)} = (\mathbf{x}_{m_{h-1}+1,j}, \ldots, \mathbf{x}_{m_{h},j})^{\top}
\end{align*}
be horizontal partitions of $\mathbf{Y}$ and $\mathbf{X}_{j}$, $j=1,\ldots,p_{n}$, where $0 = m_{0} < m_{1} < \ldots < m_{H} = n$.
In the distributed computing system, we label the nodes by $(h, c)$, so that the $(h,c)$-th node owns data $\mathbf{Y}_{(h)}$ and $\{\mathbf{X}_{j,(h)}: j \in \mathcal{I}_{c}\}$, where $h \in [H]$, $c \in [M]$ and $\cup_{c \in [M]} \mathcal{I}_{c} = [p_{n}]$.
For ease in illustration, we further assume $\{\mathcal{I}_{c}: c\in [M]\}$ forms a partition of $[p_{n}]$.
Therefore, each computing node only owns a slice of the samples on a subset $\mathcal{I}_{c}$ of the predictors as well as the same slice of the response variables. 
Moreover, let $I(j) = \{(h,c): j \in \mathcal{I}_{c}\}$ be the indices of the nodes that have some observations of predictor $j$.

We call the nodes that own the $h$-th slice of data ``segment $h$''. That is, $\{(k, c): k = h\}$.
Note that each segment is essentially the feature-distributed framework discussed in the previous sections.
In what follows, quantities computed at nodes in segment $h$ carry a subscript $(h)$. For example, $\hat{\mathbf{\Sigma}}_{j,(h)} = n_{h}^{-1} \mathbf{X}_{j,(h)}^{\top}\mathbf{X}_{j,(h)}$, where $n_{h} = m_{h} - m_{h-1}$. For simplicity, we also assume $n_{1} = \ldots = n_{H}$ in this section.
Finally, we again assume there is at least one master node to coordinate all the computing nodes $\{(h,c): h\in [H], c\in[M]\}$.

To estimate \eqref{Model1Mtx} with the horizontally partitioned feature-distributed data described above, we suggest the following procedure.
First, we obtain a set of potentially relevant predictors $\hat{J}$ and their respective upper bounds on the coefficient ranks $\hat{r}_{j}$ by running the first-stage RGA with the just-in-time stopping criterion.
This can be done by applying Algorithm \ref{alg:dRGA} to one segment. 
Alternatively, one can apply it to multiple segments in parallel and set $\hat{J} = \cap_{h} \hat{J}_{(h)}$ and $\hat{r}_{j} = \min_{h} \hat{r}_{j,(h)}$.
In either case, Theorem \ref{3.thm1} ensures the sure-screening property as $n_{1} \rightarrow \infty$ if (C1)-(C4) hold in each of the segments.
By Lemma \ref{3.lem1}, this step costs $O_{p}(s_{n}^{2}(n_{1} + d_{n}))$ bytes of communication per node in the segment(s) involved. 

Next, for each $j \in \hat{J}$, each node $(h,c) \in I(j)$ computes $\mathbf{X}_{j,(h)}^{\top}\mathbf{X}_{j,(h)}$ and, if $q_{n,j} \wedge d_{n} > \hat{r} = \sum_{j} \hat{r}_{j}$, additionally computes $\mathbf{X}_{j,(h)}^{\top}\mathbf{Y}_{(h)}$. Then, send these matrices to the master node.
The master node computes $\hat{\mathbf{\Sigma}}_{j}^{-1} = (\sum_{h=1}^{H} \mathbf{X}_{j,(h)}^{\top}\mathbf{X}_{j,(h)})^{-1}$ and the leading $\hat{r}$ singular vectors of $\sum_{h=1}^{H} \mathbf{X}_{j,(h)}^{\top}\mathbf{Y}_{(h)}$, which form the column vectors of $\mathbf{U}_{j}$ and $\mathbf{V}_{j}$. 
Then $(\hat{\mathbf{\Sigma}}_{j}^{-1}, \mathbf{U}_{j}, \mathbf{V}_{j})$ (or just $\hat{\mathbf{\Sigma}}_{j}^{-1}$ if $q_{n,j} \wedge d_{n} \leq \hat{r}$) are sent back to $I(j)$.
This step costs $O_{p}(\sum_{j \in \hat{J}} \{q_{n,j}^{2} + (q_{n,j}d_{n} + \hat{r}(q_{n,j} + d_{n}))\mathbf{1}\{q_{n,j} \wedge d_{n} > \hat{r}\}\}$ bytes of communication per node.  

Now we can start the second-stage RGA iterations. Initializing $\hat{\mathbf{G}}_{(h)}^{(0)} = \mathbf{0}$ and $\hat{\mathbf{U}}_{(h)}^{(0)} = \mathbf{Y}_{(h)}$ for each computing nodes. 
At iteration $k$, for each $j \in \hat{J}$, nodes in $I(j)$ send $\mathbf{U}_{j}^{\top}\hat{\mathbf{\Sigma}}_{j}^{-1}\mathbf{X}_{j,(h)}^{\top}\hat{\mathbf{U}}_{(h)}^{(k-1)}\mathbf{V}_{j}$ to the master. The master aggregates the matrices
\begin{align*}
    \left\{ \mathbf{P}_{j} = \sum_{h=1}^{H} \mathbf{U}_{j}^{\top}\hat{\mathbf{\Sigma}}_{j}^{-1}\mathbf{X}_{j,(h)}^{\top}\hat{\mathbf{U}}_{(h)}^{(k-1)}\mathbf{V}_{j} : j \in \hat{J} \right\},
\end{align*}
and decides $\hat{j}_{k} = \arg\max_{j \in \hat{J}} \sigma_{1}(\mathbf{P}_{j})$ and $\hat{\mathbf{S}}_{k} = L_{n}\mathbf{u}\mathbf{v}^{\top}$, where $(\mathbf{u}, \mathbf{v})$ are the leading singular vectors of $\mathbf{P}_{\hat{j}_{k}}$.
The master node sends $\hat{\mathbf{S}}_{k}$ to the nodes in $I(\hat{j}_{k})$.
Sending the matrix $\mathbf{U}_{j}^{\top}\hat{\mathbf{\Sigma}}_{j}^{-1}\mathbf{X}_{j,(h)}^{\top}\hat{\mathbf{U}}_{(h)}^{(k-1)}\mathbf{V}_{j}$ requires $O(\hat{r}^{2})$ bytes of communication if $q_{n,j} \wedge d_{n} > \hat{r}$, and $O(q_{n,j}d_{n})$ bytes otherwise. Each computing node also receives $O(\hat{r})$ or $O(q_{n,j}+d_{n})$ bytes of data from the master, depending on whether $q_{n,\hat{j}_{k}} \wedge d_{n}$ is greater than $\hat{r}$.

To compute $\hat{\lambda}_{k}$, each node $(h,c) \in I(\hat{j}_{k})$ computes and sends to the master
\begin{align*}
    \mathbf{A}_{h} = \hat{\mathbf{U}}_{(h)}^{(k-1) \top} \mathbf{X}_{\hat{j}_{k},(h)}\hat{\mathbf{\Sigma}}_{\hat{j}_{k}}^{-1}\mathbf{U}_{\hat{j}_{k}}\hat{\mathbf{S}}_{k}\mathbf{V}_{\hat{j}_{k}}^{\top} - \hat{\mathbf{U}}_{(h)}^{(k-1) \top} \hat{\mathbf{G}}_{(h)}^{(k-1)},
\end{align*}
and 
\begin{align*}
    a_{h} = \Vert \mathbf{X}_{\hat{j}_{k},(h)}\hat{\mathbf{\Sigma}}_{\hat{j}_{k}}^{-1}\mathbf{U}_{\hat{j}_{k}}\hat{\mathbf{S}}_{k}\mathbf{V}_{\hat{j}_{k}}^{\top} - \hat{\mathbf{G}}_{(h)}^{(k-1)} \Vert_{F}^{2}.
\end{align*}
The master then is able to compute $\hat{\lambda}_{k} = \max\{\min\{ \hat{\lambda}_{k,uc} , 1\}, 0\}$, where
\begin{align*}
    \hat{\lambda}_{k, uc} = \frac{\mathrm{tr}(\sum_{h=1}^{H}\mathbf{A}_{h})}{\sum_{h=1}^{H}a_{h}}.
\end{align*}
Subsequently, $\hat{\lambda}_{k}$ is sent to all nodes. 
In this step, because $\hat{\mathbf{G}}_{h}^{(k-1)}$ is of rank at most $k-1$, sending $\mathbf{A}_{h}$ costs $O(d_{n}(k\wedge d_{n}))$ bytes of communication. 

Finally, each node $(h,c) \in I(\hat{j}_{k})$ updates
\begin{align*}
    \hat{\mathbf{G}}_{(h)}^{(k)} =& (1 - \hat{\lambda}_{k})\hat{\mathbf{G}}_{(h)}^{(k-1)} + \hat{\lambda}_{k} \mathbf{X}_{\hat{j}_{k},(h)}\hat{\mathbf{\Sigma}}_{\hat{j}_{k}}^{-1}\mathbf{U}_{\hat{j}_{k}}\hat{\mathbf{S}}_{k}\mathbf{V}_{\hat{j}_{k}}^{\top}, \\
    \hat{\mathbf{U}}_{(h)}^{(k)} =& \mathbf{Y}_{(h)} - \hat{\mathbf{G}}_{(h)}^{(k)}, \\
    \hat{\mathbf{B}}_{\hat{j}_{k}}^{(k)} =& (1 - \hat{\lambda}_{k})\hat{\mathbf{B}}_{\hat{j}_{k}}^{(k-1)} + \hat{\lambda}_{k} \hat{\mathbf{\Sigma}}_{\hat{j}_{k}}^{-1}\mathbf{U}_{\hat{j}_{k}}\hat{\mathbf{S}}_{k}\mathbf{V}_{\hat{j}_{k}}^{\top}, \\
    \hat{\mathbf{B}}_{j}^{(k)} =& (1 - \hat{\lambda}_{k})\hat{\mathbf{B}}_{j}^{(k-1)}, \quad j \in \mathcal{I}_{c}-\{\hat{j}_{k}\},
\end{align*}
and also sends (possibly via the master node) the matrix $\mathbf{X}_{\hat{j}_{k},(h)}\hat{\mathbf{\Sigma}}_{\hat{j}_{k}}^{-1}\mathbf{U}_{\hat{j}_{k}}\hat{\mathbf{S}}_{k}\mathbf{V}_{\hat{j}_{k}}^{\top}$ (which is of rank one and costs $O(n_{1} + d_{n})$ bytes of communication) to the nodes $\{(h, c'): c' \neq c\}$.
Then the node $(h, c') \notin I(\hat{j}_{k})$ is able to update $\hat{\mathbf{G}}_{(h)}^{(k)}$, $\hat{\mathbf{U}}_{(h)}^{(k)}$, and $\hat{\mathbf{B}}_{j}^{(k)}$ as above. 

It can be verified the above procedure implements the second-stage RGA. 
Moreover, the communication cost for node $(h,c)$ at the $k$-th iteration is at most
\begin{align*}
    O\left( \sum_{j \in \hat{J} \cap \mathcal{I}_{c}}\left(\hat{r}^{2}\mathbf{1}\{q_{n,j}\wedge d_{n}>\hat{r}\} + q_{n,j}d_{n}\mathbf{1}\{q_{n,j}\wedge d_{n}\leq\hat{r}\} \right) + d_{n}k + n_{1} \right).
\end{align*}
As a result, the above procedure to implement TSRGA has the following guarantee.

\begin{corollary} \label{hor.cor}
    Suppose $\hat{J}$ and $\{\hat{r}_{j}:j \in \hat{J}\}$ satisfy the sure-screening property \eqref{sure} as $n_{1} \rightarrow \infty$, and assume (C1)-(C6).
    If $\max_{1 \leq j \leq p_{n}} q_{n,j} = O(n_{1}^{\alpha})$, then the above procedure achieves an error of order 
    \begin{align*}
            \frac{1}{d_{n}} \sum_{j=1}^{p_{n}} \Vert \mathbf{B}_{j}^{*} - \hat{\mathbf{B}}_{j} \Vert_{F}^{2} = 
        O_{p}\left( \frac{\mathfrak{s}_{n}\xi_{n}^{2}}{n^{2}d_{n}}\log \frac{n^{2}d_{n}}{\xi_{n}^{2}} + \frac{\xi_{n}^{2}}{n^{2}\delta_{n}^{2}}\mathbf{1}\{J_{o} \neq \emptyset\} \right)
    \end{align*}
    with a communication complexity per computing node of order 
    \begin{align*}
        O_{p}\left(n_{1}^{\max\{2\alpha,1\}}s_{n}^{2} + (s_{n}^{2}n_{1}^{\alpha}d_{n} + n_{1}) \log \frac{n^{2}d_{n}}{\xi_{n}^{2}} + s_{n}^{10} \log \frac{n^{2}d_{n}}{\xi_{n}^{2}} + d_{n} s_{n}^{8}\left( \log \frac{n^{2}d_{n}}{\xi_{n}^{2}} \right)^{2} \right).
    \end{align*}
\end{corollary}
The proof of Corollary \ref{hor.cor} is an accounting on the communication costs shown above, whose details are relegated to Appendix \ref{App:details}. The communication complexity is still free of the ambient dimension $p_{n}$, but the dimension of the predictors $\max_{1\leq j \leq p_{n}}q_{n,j}$ comes into play, which was not a factor in the purely feature-distributed case. 
The additional communication between segments could inflate the communication complexity compared to the purely feature-distributed case. 
If $\alpha \leq 0.5$ and $s_{n}=O(1)$, the communication complexity, up to poly-logarithmic factors, reduces to $O_{p}(n_{1}+n_{1}^{\alpha}d_{n} + d_{n})$, which is no larger than the purely feature-distributed case $O_{p}(n_{1}+d_{n})$ if $d_{n} =O(n_{1}^{1-\alpha})$.
On the other hand, if $\alpha>0.5$ and $s_{n}=O(1)$, the communication complexity becomes $O_{p}(n_{1}^{2\alpha} + n_{1}^{\alpha}d_{n})$ (again ignoring poly-logarithmic terms), which is higher than the purely feature-distributed case.
These costs are incurred in the greedy search as well as in the determination of $\hat{\lambda}_{k}$.
Finally, we note that the above procedure is sequential, and certain improvements can be achieved with some carefully designed communication protocol.
However, methods or algorithms for speeding up convergence or lowering communication of the proposed TSRGA with horizontal partition is left for future research.

\section{Conclusion}

This paper presented a two-stage relaxed greedy algorithm (TSRGA) for estimating high-dimensional multivariate linear regression models with feature-distributed 
data. Our main contribution is that the communication complexity of TSRGA is independent of the feature dimension, which is often very large in feature-distributed data. Instead, the complexity depends 
on the sparsity of the underlying model, making the proposed approach a 
highly scalable and efficient method for analyzing large data sets.
We also briefly discussed applying TSRGA to huge data sets that require both vertical and horizontal partitions. 

We would like to point out a possible future extension.
In some applications, it is of paramount importance to protect the privacy of each node's data.
Thus, modifying TSRGA so that privacy can be guaranteed for feature-distributed data is an important direction 
for future research.


\acks{We acknowledge the University of Chicago Research Computing Center for support of this work.}


\newpage

\appendix
\section{Second-stage RGA with feature-distributed data} \label{sec::secondstage}
The following algorithm presents the pseudo-code for the implementation of the second-stage RGA with feature-distributed data.
\begin{algorithm}[h!] 
\DontPrintSemicolon
{\small 
  \KwInput{Number of required iterations $K_{n}$, $L_{n} > 0$, pre-selected $\hat{J}$.}
  \KwOutput{Each worker $1 \leq c \leq M$ has the coefficient matrices $\{ \hat{\mathbf{B}}_{j}: j \in\mathcal{I}_{c} \}$ to use for prediction.}
  \KwInit{$\hat{\mathbf{B}}_{j} = \mathbf{0}$, for all $j$, and $\hat{\mathbf{G}}^{(0)} = \mathbf{0}$}
  \For{$k = 1, 2, \ldots, K_{n}$}{
  	\Workers{
		\If{$k>1$}{
		Receive $(c^{*}, \hat{\lambda}_{k-1}, \sigma_{\hat{j}_{k-1}}, \mathbf{u}_{\hat{j}_{k-1}}, \mathbf{v}_{\hat{j}_{k-1}})$ from the master. \\
		$\hat{\mathbf{G}}^{(k-1)} = (1 - \hat{\lambda}_{k-1})\hat{\mathbf{G}}^{(k-2)} + \hat{\lambda}_{k-1}\sigma_{\hat{j}_{k-1}}\mathbf{u}_{\hat{j}_{k-1}}\mathbf{v}_{\hat{j}_{k-1}}^{\top}$. \\
		$\hat{\mathbf{B}}_{j} = (1 - \hat{\lambda}_{k-1}) \hat{\mathbf{B}}_{j}$ for $j \in \mathcal{I}_{c} \cap \hat{J}$. \\
		\If{$c = c^{*}$}{
			$\hat{\mathbf{B}}_{\hat{j}_{k-1}^{(c)}} = \hat{\mathbf{B}}_{\hat{j}_{k-1}} + \hat{\lambda}_{k-1} \hat{\mathbf{\Sigma}}_{\hat{j}_{k-1}^{(c)}}^{-1}\mathbf{U}_{\hat{j}_{k-1}^{(c)}}\hat{\mathbf{S}}_{k-1}^{(c)}\mathbf{V}_{\hat{j}_{k-1}^{(c)}}^{\top}$
			}
		}
		$\hat{\mathbf{U}}^{(k-1)} = \mathbf{Y} - \hat{\mathbf{G}}^{(k-1)}$ \\
		$(\hat{j}_{k}^{(c)}, \hat{\mathbf{S}}_{k}^{(c)}) \in \arg\max_{\substack{j \in \mathcal{I}_{c} \cap \hat{J} \\ \Vert \mathbf{S}_{k} \Vert_{*} \leq L_{n}}}|\langle \hat{\mathbf{U}}^{(k-1)}, \mathbf{X}_{j}\hat{\mathbf{\Sigma}}_{j}^{-1}\mathbf{U}_{j}\mathbf{S}_{k}\mathbf{V}_{j}^{\top} \rangle|$ \\
		$\rho_{c} = |\langle \hat{\mathbf{U}}^{(k-1)}, \mathbf{X}_{\hat{j}_{k}^{(c)}}\hat{\mathbf{\Sigma}}_{\hat{j}_{k}^{(c)}}^{-1}\mathbf{U}_{\hat{j}_{k}^{(c)}}\hat{\mathbf{S}}_{k}^{(c)}\mathbf{V}_{\hat{j}_{k}^{(c)}}^{\top} \rangle|$ \\
		Find the leading singular value decomposition:
        \begin{align*}
            \mathbf{X}_{\hat{j}_{k}^{(c)}}\hat{\mathbf{\Sigma}}_{\hat{j}_{k}^{(c)}}^{-1}\mathbf{U}_{\hat{j}_{k}^{(c)}}\hat{\mathbf{S}}_{k^{(c)}}\mathbf{V}_{\hat{j}_{k}^{(c)}}^{\top}  = \sigma_{\hat{j}_{k}^{(c)}}\mathbf{u}_{\hat{j}_{k}^{(c)}}\mathbf{v}_{\hat{j}_{k}^{(c)}}^{\top}
		\end{align*} \\
		Send $(\sigma_{\hat{j}_{k}^{(c)}}, \mathbf{u}_{\hat{j}_{k}^{(c)}}, \mathbf{v}_{\hat{j}_{k}^{(c)}}, \rho_{c})$ to the master.
	}
	\Master{
		Receives $\{(\sigma_{\hat{j}_{k}^{(c)}}, \mathbf{u}_{\hat{j}_{k}^{(c)}}, \mathbf{v}_{\hat{j}_{k}^{(c)}}, \rho_{c}): c=1,2,\ldots,M\}$ from the workers. \\
		$c^{*} = \arg\max_{1\leq c \leq M} \rho_{c}$ \\
		$\sigma_{\hat{j}_{k}} = \sigma_{\hat{j}_{k}^{(c^{*})}}, \mathbf{u}_{\hat{j}_{k}} = \mathbf{u}_{\hat{j}_{k}^{(c^{*})}}, \mathbf{v}_{\hat{j}_{k}} = \mathbf{v}_{\hat{j}_{k}^{(c^{*})}}$ \\
		$\hat{\mathbf{G}}^{(k)} = (1 - \hat{\lambda}_{k}) \hat{\mathbf{G}}^{(k-1)} + \hat{\lambda}_{k} \sigma_{\hat{j}_{k}} \mathbf{u}_{\hat{j}_{k}}\mathbf{v}_{\hat{j}_{k}}^{\top}$,
		where $\hat{\lambda}_{k}$ is determined by 
		\begin{align*} 
    			\hat{\lambda}_{k} \in \arg\min_{0 \leq \lambda \leq 1} \Vert \mathbf{Y} - (1-\lambda) \hat{\mathbf{G}}^{(k-1)} - \lambda\sigma_{\hat{j}_{k}} \mathbf{u}_{\hat{j}_{k}}\mathbf{v}_{\hat{j}_{k}}^{\top} \Vert_{F}^{2}.
		\end{align*}
		\\
		Broadcast $(c^{*}, \hat{\lambda}_{k}, \sigma_{\hat{j}_{k}}, \mathbf{u}_{\hat{j}_{k}}, \mathbf{v}_{\hat{j}_{k}})$ to all workers.
	}
  }
\caption{Feature-distributed second-stage RGA}\label{alg:d2RGA}
}
\end{algorithm}

\section{Proofs} \label{sec::proofs}
This section presents the essential elements of the proofs of our main results. 
Further technical details are relegated to Appendix \ref{App:details}.

The analysis of TSRGA relies on what we call the ``noiseless updates,'' a theoretical device constructed as follows. 
Initialize $\mathbf{G}^{(0)} = \mathbf{0}$ and $\mathbf{U}^{(0)} = \tilde{\mathbf{Y}}$. 
For $1 \leq k \leq K_{n}$, suppose $(\hat{j}_{k}, \tilde{\mathbf{B}}_{\hat{j}_{k}})$ is chosen according to \eqref{2.2.1} by the first-stage RGA.
The noiseless updates are defined as
\begin{align} \label{6.1}
    \mathbf{G}^{(k)} =& (1-\lambda_{k})\hat{\mathbf{G}}^{(k-1)} + \lambda_{k} \mathbf{X}_{\hat{j}_{k}}\tilde{\mathbf{B}}_{\hat{j}_{k}},  
\end{align}
where
\begin{align} \label{6.2} 
    \lambda_{k} \in \arg\min_{0 \leq \lambda \leq 1} \Vert \tilde{\mathbf{Y}} - (1-\lambda)\hat{\mathbf{G}}^{(k-1)} - \lambda \mathbf{X}_{\hat{j}_{k}}\tilde{\mathbf{B}}_{\hat{j}_{k}} \Vert_{F}^{2}.
\end{align}
Recall that $\tilde{\mathbf{Y}} = \sum_{j=1}^{p_{n}}\mathbf{X}_{j}\mathbf{B}_{j}^{*}$ is the noise-free part of the response. 
Thus $\mathbf{G}^{(k)}$ is unattainable in practice. 
Similarly, we can define the noiseless updates for the second-stage RGA, with $\tilde{\mathbf{B}}_{\hat{j}_{k}}$ replaced by $\hat{\mathbf{\Sigma}}_{\hat{j}_{k}}^{-1}\mathbf{U}_{\hat{j}_{k}}\hat{\mathbf{S}}_{k}\mathbf{V}_{\hat{j}_{k}}^{\top}$ in (\ref{6.1}) and (\ref{6.2}).
By definition of the updates, for first- and second-stage RGA,
\begin{align} 
    \Vert \tilde{\mathbf{Y}} - \hat{\mathbf{G}}^{(k)} \Vert_{F}^{2} \leq& \Vert \tilde{\mathbf{Y}} - \mathbf{G}^{(k)} \Vert_{F}^{2} + 2 \langle \mathbf{E}, \hat{\mathbf{G}}^{(k)} - \mathbf{G}^{(k)} \rangle \notag \\
    \leq& \Vert \tilde{\mathbf{Y}} - \hat{\mathbf{G}}^{(k-1)} \Vert_{F}^{2} + 2 \langle \mathbf{E}, \hat{\mathbf{G}}^{(k)} - \mathbf{G}^{(k)} \rangle \label{6.3}
\end{align}
Recursively applying \eqref{6.3}, we have for any $1 \leq l \leq k$,
\begin{align} \label{6.4} 
    \Vert \tilde{\mathbf{Y}} - \hat{\mathbf{G}}^{(k)}  \Vert_{F}^{2} \leq& \Vert \tilde{\mathbf{Y}} - \hat{\mathbf{G}}^{(k-l)} \Vert_{F}^{2} + 2 \sum_{j=1}^{l} \langle \mathbf{E}, \hat{\mathbf{G}}^{(k-j+1)} - \mathbf{G}^{(k-j+1)} \rangle. 
\end{align}
\eqref{6.4} bounds the empirical prediction error at step $k$ by the empirical prediction error at step $k-l$ and a remainder term involving the noise and the noiseless updates up to step $l$. 
This will be handy in numerous places throughout the proofs.

Two other useful identities are
\begin{align} \label{6.5.0}
	\max_{\substack{1 \leq j \leq p_{n} \\ \Vert \mathbf{B}_{j} \Vert_{*} \leq L_{n}}} \langle \mathbf{A}, \mathbf{X}_{j}\mathbf{B}_{j} \rangle = 
	\sup_{\substack{\mathbf{B}_{j} \in \mathbb{R}^{q_{n,j} \times d_{n}}, j = 1,2, \ldots, ,p_{n} \\ \sum_{j} \Vert \mathbf{B}_{j} \Vert_{*} \leq L_{n}} }  \left\langle \mathbf{A}, \sum_{j=1}^{p_{n}} \mathbf{X}_{j}\mathbf{B}_{j} \right\rangle
\end{align}
and 
\begin{align} \label{6.5}
    \max_{\substack{j \in \hat{J} \\ \Vert \mathbf{S} \Vert_{*} \leq L_{n}}} \langle \mathbf{A}, \mathbf{X}_{j}\hat{\mathbf{\Sigma}}_{j}^{-1}\mathbf{U}_{j}\mathbf{S}\mathbf{V}_{j}^{\top} \rangle = \sup_{\substack{\sum_{j \in \hat{J}} \Vert \mathbf{S}_{j} \Vert_{*} \leq L_{n}}} \left\langle \mathbf{A}, \sum_{j \in \hat{J}}\mathbf{X}_{j}\hat{\mathbf{\Sigma}}_{j}^{-1} \mathbf{U}_{j}\mathbf{S}_{j}\mathbf{V}_{j}^{\top} \right\rangle, 
\end{align}
where $\mathbf{A} \in \mathbb{R}^{n \times d_{n}}$ is arbitrary. 
These identities hold because the maximum of the inner product is attained at the extreme points in the $\ell_{1}$ ball. 
The proofs are omitted for brevity.

We first prove an auxiliary lemma which guarantees sub-linear convergence of the empirical prediction error, whose proof makes use of the noiseless updates introduced above.
\begin{lemma} \label{6.lem1}
Assume (C1)-(C2) and that $\sum_{j=1}^{p_{n}}\Vert \mathbf{B}_{j}^{*} \Vert_{*} \leq d_{n}^{1/2}L$. RGA has the following uniform rate of convergence.
\begin{align} \label{6.lem1-0}
    \max_{1 \leq k \leq K_{n}} \frac{(nd_{n})^{-1}\Vert \tilde{\mathbf{Y}} - \hat{\mathbf{G}}^{(k)} \Vert_{F}^{2}}{k^{-1}} = O_{p}(1).
\end{align}
\end{lemma}

\begin{proof}
Let $1 \leq m \leq K_{n}$ be arbitrary. 
Note that for any $1 \leq k \leq K_{n}$,
\begin{align} \label{6.lem1.1}
    &\langle \tilde{\mathbf{Y}} - \hat{\mathbf{G}}^{(k-1)}, \mathbf{X}_{\hat{j}_{k}}\tilde{\mathbf{B}}_{\hat{j}_{k}} - \hat{\mathbf{G}}^{(k-1)} \rangle \notag \\
    =& \langle \mathbf{Y} - \hat{\mathbf{G}}^{(k-1)}, \mathbf{X}_{\hat{j}_{k}}\tilde{\mathbf{B}}_{\hat{j}_{k}} - \hat{\mathbf{G}}^{(k-1)} \rangle - \langle \mathbf{E}, \mathbf{X}_{\hat{j}_{k}}\tilde{\mathbf{B}}_{\hat{j}_{k}} - \hat{\mathbf{G}}^{(k-1)} \rangle \notag \\
    \geq& \max_{\substack{1 \leq j \leq p_{n} \\ \Vert \mathbf{B}_{j} \Vert_{*} \leq L_{n}}} \{ \langle \mathbf{Y} - \hat{\mathbf{G}}^{(k-1)}, \mathbf{X}_{j}\mathbf{B}_{j} - \hat{\mathbf{G}}^{(k-1)} \rangle \} - 2L_{n}\xi_{E} \notag \\
    \geq& \max_{\substack{1 \leq j \leq p_{n} \\ \Vert \mathbf{B}_{j} \Vert_{*} \leq L_{n}}} \{ \langle \tilde{\mathbf{Y}} - \hat{\mathbf{G}}^{(k-1)}, \mathbf{X}_{j}\mathbf{B}_{j} - \hat{\mathbf{G}}^{(k-1)} \rangle \} - 4L_{n}\xi_{E}.
\end{align}
Put
\begin{align} \label{keyEvents}
    \mathcal{E}_{n}(m) =& \left\{ \min_{1 \leq l \leq m}\max_{\substack{1 \leq j \leq p_{n} \\ \Vert \mathbf{B}_{j} \Vert_{*} \leq L_{n}}} \langle \tilde{\mathbf{Y}} - \hat{\mathbf{G}}^{(l-1)}, \mathbf{X}_{j}\mathbf{B}_{j} - \hat{\mathbf{G}}^{(l-1)} \rangle > \tilde{\tau}d_{n}^{1/2}\xi_{E}  \right\},
\end{align}
for some $\tilde{\tau} > 4L_{0}$. It follows from \eqref{6.5.0} and \eqref{6.lem1.1} that on $\mathcal{E}_{n}(m)$, for all $1 \leq k \leq m$,
\begin{align} \label{6.lem1.2}
    &\langle \tilde{\mathbf{Y}} - \hat{\mathbf{G}}^{(k-1)}, \mathbf{X}_{\hat{j}_{k}}\tilde{\mathbf{B}}_{\hat{j}_{k}} - \hat{\mathbf{G}}^{(k-1)} \rangle \notag \\
    \geq& (1 - \frac{4L_{0}}{\tilde{\tau}})  \max_{\substack{1 \leq j \leq p_{n} \\ \Vert \mathbf{B}_{j} \Vert_{*} \leq L}} \{ \langle \tilde{\mathbf{Y}} - \hat{\mathbf{G}}^{(k-1)}, \mathbf{X}_{j}\mathbf{B}_{j} - \hat{\mathbf{G}}^{(k-1)} \rangle \} \notag \\
    \geq& (1 - \frac{4L_{0}}{\tilde{\tau}})  \Vert \tilde{\mathbf{Y}} - \hat{\mathbf{G}}^{(k-1)} \Vert_{F}^{2} \notag \\
    :=&\tau \Vert \tilde{\mathbf{Y}} - \hat{\mathbf{G}}^{(k-1)} \Vert_{F}^{2} \\
    \geq& 0 \notag,
\end{align}
where $\tau =  1 - 4L_{0}/\tilde{\tau}$. This, together with Lemma \ref{aux-lem1}(iii) in Appendix \ref{App:details}, implies
\begin{align*}
    \lambda_{k} = \frac{\langle \tilde{\mathbf{Y}} - \hat{\mathbf{G}}^{(k-1)}, \mathbf{X}_{\hat{j}_{k}}\tilde{\mathbf{B}}_{\hat{j}_{k}} - \hat{\mathbf{G}}^{(k-1)} \rangle}{\Vert \mathbf{X}_{\hat{j}_{k}}\tilde{\mathbf{B}}_{\hat{j}_{k}} - \hat{\mathbf{G}}^{(k-1)} \Vert_{F}^{2}}
\end{align*}
for $1 \leq k \leq m$ on $\mathcal{E}_{n}(m)$ except for a vanishing event. 
This, combined with \eqref{6.3} and \eqref{6.lem1.2}, yields
\begin{align}
    \Vert \tilde{\mathbf{Y}} - \hat{\mathbf{G}}^{(k)} \Vert_{F}^{2} \leq& \Vert \tilde{\mathbf{Y}} - \mathbf{G}^{(k)} \Vert_{F}^{2} + 2 \langle \mathbf{E}, \hat{\mathbf{G}}^{(k)} - \mathbf{G}^{(k)} \rangle \notag \\
    =& \Vert \tilde{\mathbf{Y}} - \hat{\mathbf{G}}^{(k-1)} - \lambda_{k} (\mathbf{X}_{\hat{j}_{k}}\tilde{\mathbf{B}}_{\hat{j}_{k}} - \hat{\mathbf{G}}^{(k-1)} ) \Vert_{F}^{2} + 2 \langle \mathbf{E}, \hat{\mathbf{G}}^{(k)} - \mathbf{G}^{(k)} \rangle \notag \\
    =& \Vert \tilde{\mathbf{Y}} - \hat{\mathbf{G}}^{(k-1)} \Vert_{F}^{2} - \frac{\langle \tilde{\mathbf{Y}} - \hat{\mathbf{G}}^{(k-1)}, \mathbf{X}_{\hat{j}_{k}}\tilde{\mathbf{B}}_{\hat{j}_{k}} - \hat{\mathbf{G}}^{(k-1)} \rangle^{2}}{\Vert \mathbf{X}_{\hat{j}_{k}}\tilde{\mathbf{B}}_{\hat{j}_{k}} - \hat{\mathbf{G}}^{(k-1)}\Vert_{F}^{2}} + 2 \langle \mathbf{E}, \hat{\mathbf{G}}^{(k)} - \mathbf{G}^{(k)} \rangle \notag \\
    \leq& \Vert \tilde{\mathbf{Y}} - \hat{\mathbf{G}}^{(k-1)} \Vert_{F}^{2}\left\{ 1 - \frac{\tau^{2}\Vert \tilde{\mathbf{Y}} - \hat{\mathbf{G}}^{(k-1)} \Vert_{F}^{2}}{\Vert \mathbf{X}_{\hat{j}_{k}}\tilde{\mathbf{B}}_{\hat{j}_{k}} - \hat{\mathbf{G}}^{(k-1)}\Vert_{F}^{2}}\right\} + 2 \langle \mathbf{E}, \hat{\mathbf{G}}^{(k)} - \mathbf{G}^{(k)} \rangle \label{6.lem1.3}
\end{align}
for all $1\leq k \leq m$ on $\mathcal{E}_{n}(m)$ except for a vanishing event.
By (C1), with probability tending to one, $\Vert \mathbf{X}_{\hat{j}_{k}}\tilde{\mathbf{B}}_{\hat{j}_{k}} - \hat{\mathbf{G}}^{(k-1)}\Vert_{F}^{2} \leq 4L_{n}^{2}n\mu$ and $\Vert \tilde{\mathbf{Y}} \Vert_{F}^{2} \leq (1-\epsilon_{L})^{2}L_{n}^{2}n \mu$.
Now by Lemma \ref{aux-lem2} and Lemma \ref{aux-lem1}(ii) in Appendix \ref{App:details}, we have
\begin{align} \label{6.lem1.4}
    \frac{1}{nd_{n}}\Vert \tilde{\mathbf{Y}} - \hat{\mathbf{G}}^{(m)} \Vert_{F}^{2} \leq& \frac{4L_{0}^{2}\mu}{1+m\tau^{2}} + 2 \sum_{l=1}^{m} \frac{|\langle \mathbf{E}, \hat{\mathbf{G}}^{(l)} - \mathbf{G}^{(l)} \rangle|}{nd_{n}} \notag \\
    =& \frac{4L_{0}^{2}\mu}{1+m\tau^{2}} + 2 \sum_{l=1}^{m} |\hat{\lambda}_{l} - \lambda_{l}| \frac{|\langle \mathbf{E}, \mathbf{X}_{\hat{j}_{l}}\tilde{\mathbf{B}}_{\hat{j}_{l}} - \hat{\mathbf{G}}^{(l-1)} \rangle|}{nd_{n}} \notag \\
    \leq& \frac{4L_{0}^{2}\mu}{1+m\tau^{2}} + \frac{8}{1-\epsilon_{L}} \frac{m\xi_{E}^{2}}{n^{2}d_{n}},
\end{align}
on $\mathcal{E}_{n}(m)$ except for a vanishing event. 
Note that by (C2), $m\xi_{E}^{2}/(n^{2}d_{n}) \leq m^{-1}(K_{n}\xi_{E}/(nd_{n}^{1/2}))^{2} = O_{p}(m^{-1})$.
Furthermore, it is shown in Appendix~\ref{App:details} that on $\mathcal{E}_{n}^{c}(m)$ except for a vanishing event, 
\begin{align} \label{6.lem1.5}
    \frac{1}{nd_{n}}\Vert \tilde{\mathbf{Y}} - \hat{\mathbf{G}}^{(m)} \Vert_{F}^{2} \leq \frac{\tilde{\tau}\xi_{E}}{n \sqrt{d_{n}}} + \frac{8m\xi_{E}^{2}}{(1-\epsilon_{L})n^{2}d_{n}}.
\end{align}
Combining \eqref{6.lem1.4} and \eqref{6.lem1.5} yields the desired result. 
\end{proof}

Now we are ready to prove the main results.

\begin{proof}[\textsc{Proof of Theorem \ref{3.thm1}}]
Since $d_{n}^{1/2}L \geq \sum_{j=1}^{p_{n}}\Vert \mathbf{B}_{j}^{*} \Vert_{*} \geq \sharp(J_{n})\min_{j \in J_{n}}\sigma_{r_{j}^{*}}(\mathbf{B}_{j}^{*})$ and $s_{n} = o(K_{n}^{2})$, it follows that $\sharp(J_{n}) = o(K_{n})$, and by (C3), with probability tending to one, $\lambda_{\min}(\mathbf{X}(\hat{J}_{k} \cup J_{n})^{\top}\mathbf{X}(\hat{J}_{k} \cup J_{n})) \geq n\mu^{-1}$, for all $1 \leq k \leq K_{n}$, where $\hat{J}_{k} = \{\hat{j}_{1},\hat{j}_{2},\ldots,\hat{j}_{k}\}$.
Let 
$\mathcal{G}_{n} = \{\mbox{there exists }$ $\mbox{some } j \mbox{ such that } \mathrm{rank}(\mathbf{B}_{j}^{*}) > \mathrm{rank}(\hat{\mathbf{B}}_{j}^{(\hat{k})})\}$. 
Then on $\mathcal{G}_{n}$ except for a vanishing event, it follows from \eqref{6.5.0}, (C3), Eckart-Young theorem and (C4) that
\begin{align}
    \min_{1 \leq m \leq \hat{k}} \max_{\substack{1 \leq j \leq p_{n} \\ \Vert \mathbf{B}_{j} \Vert_{*} \leq L}} \langle \tilde{\mathbf{Y}} - \hat{\mathbf{G}}^{(m)}, \mathbf{X}_{j}\mathbf{\mathbf{B}}_{j} - \hat{\mathbf{G}}^{(m)} \rangle \geq& 
    \min_{1 \leq m \leq \hat{k}} \Vert \tilde{\mathbf{Y}} - \hat{\mathbf{G}}^{(m)} \Vert_{F}^{2} \notag \\ 
    \geq& n\mu^{-1} \min_{1 \leq m \leq \hat{k} }\Vert \mathbf{B}_{j}^{*} - \hat{\mathbf{B}}_{j}^{(m)} \Vert_{F}^{2} \notag \\
    \geq& n \mu^{-1} \min_{\mathrm{rank}(\mathbf{B}) < r_{j}^{*}} \Vert \mathbf{B}_{j}^{*} - \mathbf{B} \Vert_{F}^{2} \notag \\
    \geq& \frac{nd_{n}}{\mu s_{n}}. \label{6.thm1.0}
\end{align}
By \eqref{6.thm1.0}, (C2) and (C4), we have $\lim_{n \rightarrow \infty} \mathbb{P}\left( \mathcal{G}_{n} \cap \mathcal{E}_{n}^{c}(\hat{k}) \right)  \leq \lim_{n \rightarrow \infty} \mathbb{P}(nd_{n}^{1/2} \leq \tilde{\tau}\mu s_{n}\xi_{E}) = 0$,
where $\mathcal{E}_{n}(\cdot)$ is defined in \eqref{keyEvents}. 
Hence it suffices to show $\lim_{n \rightarrow \infty}\mathbb{P}(\mathcal{G}_{n} \cap \mathcal{E}_{n}(\hat{k})) = 0$. 
By \eqref{6.thm1.0} and the same argument as in \eqref{6.lem1.3}, on $\mathcal{G}_{n} \cap \mathcal{E}_{n}(\hat{k})$ except for a vanishing event, 
\begin{align*}
    \Vert \tilde{\mathbf{Y}} - \hat{\mathbf{G}}^{(k)} \Vert_{F}^{2} \leq& \Vert \tilde{\mathbf{Y}} - \hat{\mathbf{G}}^{(k - 1)} \Vert_{F}^{2}
    \left\{ 1 - \frac{\tau^{2}\Vert \tilde{\mathbf{Y}} - \hat{\mathbf{G}}^{(k-1)} \Vert_{F}^{2}}{\Vert \mathbf{X}_{\hat{j}_{k}}\tilde{\mathbf{B}}_{\hat{j}_{k}} - \hat{\mathbf{G}}^{(k-1)} \Vert_{F}^{2}} \right\} + 2 \langle \mathbf{E}, \hat{\mathbf{G}}^{(k)} - \mathbf{G}^{(k)} \rangle \\
    \leq& \Vert \tilde{\mathbf{Y}} - \hat{\mathbf{G}}^{(k - 1)} \Vert_{F}^{2} \left\{ 1 - \frac{\tau^{2}s_{n}^{-1}}{4L_{0}^{2}\mu^{2}} \right\} + 2 \langle \mathbf{E}, \hat{\mathbf{G}}^{(k)} - \mathbf{G}^{(k)} \rangle,
\end{align*}
and thus
\begin{align*} 
    nd_{n}\hat{\sigma}_{k}^{2} \leq \Vert \tilde{\mathbf{Y}} - \hat{\mathbf{G}}^{(k-1)} \Vert_{F}^{2} \left(1 - \frac{\tau^{2}s_{n}^{-1}}{4L_{0}^{2}\mu^{2}}\right) + \Vert \mathbf{E} \Vert_{F}^{2} + 2 \langle \mathbf{E}, \tilde{\mathbf{Y}} - \hat{\mathbf{G}}^{(k)}, \rangle 
\end{align*}
for $1 \leq k \leq \hat{k}$. 
It follows that
\begin{align} \label{6.thm1.1}
    \frac{\hat{\sigma}_{k}^{2}}{\hat{\sigma}_{k-1}^{2}} \leq& \frac{(nd_{n})^{-1} \Vert \tilde{\mathbf{Y}} - \hat{\mathbf{G}}^{(k-1)} \Vert_{F}^{2}  + (nd_{n})^{-1} \Vert \mathbf{E} \Vert_{F}^{2} + 4L_{0}\xi_{E}/(nd_{n}^{1/2})}{(nd_{n})^{-1}\Vert \tilde{\mathbf{Y}} - \hat{\mathbf{G}}^{(k-1)} \Vert_{F}^{2}  + (nd_{n})^{-1} \Vert \mathbf{E} \Vert_{F}^{2} - 4L_{0}\xi_{E}/(nd_{n}^{1/2})} \notag \\
    &- \frac{\tau^{2}s_{n}^{-1}}{4L_{0}^{2}\mu^{2}} \frac{(nd_{n})^{-1}\Vert \tilde{\mathbf{Y}} - \hat{\mathbf{G}}^{(k-1)} \Vert_{F}^{2}}{(nd_{n})^{-1}\Vert \tilde{\mathbf{Y}} - \hat{\mathbf{G}}^{(k-1)} \Vert_{F}^{2} + (nd_{n})^{-1} \Vert \mathbf{E} \Vert_{F}^{2} - 4L_{0}\xi_{E}/(nd_{n}^{1/2})} \notag \\
    := & A_{k} - B_{k},
\end{align}
for $1 \leq k \leq \hat{k}$ on $\mathcal{G}_{n} \cap \mathcal{E}_{n}(\hat{k})$ except for a vanishing event.
We show in Appendix \ref{App:details} that on $\mathcal{G}_{n} \cap \mathcal{E}_{n}(\hat{k})$ except for a vanishing event, for all $1 \leq k \leq \hat{k}$,
\begin{align} \label{6.thm1.2}
    A_{k} \leq 1 + \frac{12ML_{0} \xi_{E}}{nd_{n}^{1/2}},
\end{align}
and 
\begin{align} \label{6.thm1.3}
    B_{k} \geq \frac{\tau^{2}}{4L_{0}^{2}\mu^{2}}s_{n}^{-1} \frac{1}{1+\mu M s_{n}}\left(1 - \frac{4ML_{0}\xi_{E}}{nd_{n}^{1/2}} \right).
\end{align}
By \eqref{6.thm1.1}-\eqref{6.thm1.3}, 
$\max_{1 \leq k \leq \hat{k}} \hat{\sigma}_{k}^{2}/\hat{\sigma}_{k-1}^{2} \leq 1 - s_{n}^{-2} C_{n}$,
where 
\begin{align*}
    C_{n} = \frac{\tau^{2}}{4L_{0}^{2}\mu^{2}} \frac{1}{\mu M + s_{n}^{-1}}\left( 1 - \frac{4 M L_{0} \xi_{E}}{nd_{n}^{1/2}} \right) - 12 M L_{0} \frac{s_{n}^{2}\xi_{E}}{nd_{n}^{1/2}}.
\end{align*}
By (C2) and (C4), it can be shown that there exists some $v > 0$ such that $C_{n} \geq v$ with probability tending to one.
Therefore, by the definition of $\hat{k}$,
\begin{align}
    \mathbb{P}(\mathcal{G}_{n} \cap 
    \mathcal{E}_{n}(\hat{k})) \leq& \mathbb{P}(\hat{k} < K_{n}, \mathcal{G}_{n} \cap 
    \mathcal{E}_{n}(\hat{k})) + \mathbb{P}(\hat{k} = K_{n}, \mathcal{G}_{n} \cap 
    \mathcal{E}_{n}(\hat{k})) \notag \\
    \leq& \mathbb{P}(\max_{1 \leq k \leq \hat{k}} \hat{\sigma}_{k}^{2} / \hat{\sigma}_{k-1}^{2}
    \leq 1 - vs_{n}^{-2}, \hat{k} < K_{n}) + \mathbb{P}(\hat{k} = K_{n}, \mathcal{G}_{n} \cap \mathcal{E}_{n}(\hat{k})) + o(1) \notag \\
    =& \mathbb{P}(\hat{k} = K_{n}, \mathcal{G}_{n} \cap \mathcal{E}_{n}(\hat{k})) + o(1), \label{6.thm1.5}
\end{align}
if $t_{n} = Cs_{n}^{-2}$ in \eqref{jit} is chosen with $C < v$.
In view of \eqref{6.thm1.5}, it remains to show $\mathbb{P}(\hat{k} = K_{n}, \mathcal{G}_{n} \cap \mathcal{E}_{n}(\hat{k})) = o(1)$.
Since $s_{n}=o(K_{n})$ by (C4), it follows from \eqref{6.thm1.0} and Lemma \ref{6.lem1} that 
\begin{align*}
    \mathbb{P}(\hat{k} = K_{n}, \mathcal{G}_{n}) \leq& \mathbb{P}\left(\frac{1}{nd_{n}} \Vert \tilde{\mathbf{Y}} - \hat{\mathbf{G}}^{(K_{n})} \Vert_{F}^{2} \geq \frac{1}{\mu s_{n}}\right)  + o(1) \\
    =& \mathbb{P}\left(\frac{(nd_{n})^{-1} \Vert \tilde{\mathbf{Y}} - \hat{\mathbf{G}}^{(K_{n})} \Vert_{F}^{2}}{K_{n}^{-1}} \geq \frac{K_{n}}{\mu s_{n}} \right) + o(1) \\
    =& o(1),
\end{align*}
which completes the proof.
\end{proof}

\begin{proof}[\textsc{Proof of Lemma \ref{3.lem1}}]
Letting $a_{n} = \lfloor D s_{n}^{2} \rfloor$ for some arbitrary $D > 0$, we have
\begin{align}
    \mathbb{P}(\hat{k} > a_{n}) \leq& \mathbb{P}\left( \frac{\hat{\sigma}_{a_{n}}^{2}}{\hat{\sigma}_{a_{n}-1}^{2}} < 1 - Cs_{n}^{2} \right) \notag \\
    =& \mathbb{P}\left( Cs_{n}^{-2} < \frac{\hat{\sigma}_{a_{n}-1}^{2} - \zeta_{n}^{2} - (\hat{\sigma}_{a_{n}}^{2} - \zeta_{n}^{2})}{\zeta_{n}^{2} + \hat{\sigma}_{a_{n}-1}^{2} - \zeta_{n}^{2}} \right) \notag \\
    \leq& \mathbb{P}\left( Cs_{n}^{-2} < \frac{\hat{\sigma}_{a_{n}-1}^{2} - \zeta_{n}^{2}}{M^{-1}+\hat{\sigma}_{a_{n}-1}^{2} - \zeta_{n}^{2}} + \frac{4L_{0}\xi_{E}n^{-1}d_{n}^{-1/2}}{M^{-1}+\hat{\sigma}_{a_{n}-1}^{2} - \zeta_{n}^{2}} \right) + o(1). \label{6.3lem1.1}
\end{align}
Put $A_{n} = \{\hat{\sigma}_{a_{n}-1}^{2} - \zeta_{n}^{2} > 0 \}$. Then \eqref{6.3lem1.1} implies
\begin{align*}
    \mathbb{P}(\hat{k} > a_{n}, A_{n}) \leq& \mathbb{P}\left( M^{-1} + \hat{\sigma}_{a_{n}-1}^{2} - \zeta_{n}^{2} < \frac{\hat{\sigma}_{a_{n}-1}^{2}-\zeta_{n}^{2}}{Cs_{n}^{-2}} + \frac{4L_{0}s_{n}^{2}\xi_{E}}{Cnd_{n}^{1/2}}, A_{n} \right) + o(1) \notag \\
    \leq& \mathbb{P}\left( M^{-1} < Z_{n} \frac{s_{n}^{2}}{C(a_{n}-1)}  + \frac{4L_{0}}{C} \frac{s_{n}^{2}\xi_{E}}{nd_{n}^{1/2}} \right) + o(1), \notag 
\end{align*}
where
\begin{align*}
    Z_{n} := \max_{1 \leq k \leq K_{n}} \frac{|(nd_{n})^{-1}\Vert \mathbf{Y} - \hat{\mathbf{G}}^{(k)} \Vert_{F}^{2} - \zeta_{n}^{2}|}{k^{-1}}.
\end{align*}
Since $|(nd_{n})^{-1}\Vert \mathbf{Y} - \hat{\mathbf{G}}^{(k)} \Vert_{F}^{2} - \zeta_{n}^{2}| \leq (nd_{n})^{-1}\Vert \tilde{\mathbf{Y}} - \hat{\mathbf{G}}^{(k)} \Vert_{F}^{2} + 4L_{0}\xi_{E}n^{-1}d_{n}^{-1/2}$, where $\zeta_{n}^{2} = (nd_{n})^{-1}\Vert \mathbf{E} \Vert_{F}^{2}$, it follows from Lemma \ref{6.lem1} that $Z_{n}=O_{p}(1)$.
Thus $\limsup_{n \rightarrow \infty}\mathbb{P}(\hat{k}>a_{n},A_{n}) \rightarrow 0$ as $D \rightarrow 0$.
On $A_{n}^{c}$, it is not difficult to show that
\begin{align*}
    \hat{\sigma}_{a_{n}}^{2} - \zeta_{n}^{2} \leq \hat{\sigma}_{a_{n}-1}^{2} - \zeta_{n}^{2} \leq 0
\end{align*}
and
\begin{align*}
    \max\left\{ \frac{1}{nd_{n}}\Vert \tilde{\mathbf{Y}} - \hat{\mathbf{G}}^{(a_{n}-1)} \Vert_{F}^{2}, \frac{1}{nd_{n}}\Vert \tilde{\mathbf{Y}} - \hat{\mathbf{G}}^{(a_{n})} \Vert_{F}^{2} \right\} \leq \frac{4L_{0}\xi_{E}}{nd_{n}^{1/2}}.
\end{align*}
It follows that on $A_{n}^{c}$, 
\begin{align*}
    \frac{\hat{\sigma}_{a_{n}}^{2}}{\hat{\sigma}_{a_{n}-1}^{2}} =& 1 - \frac{\hat{\sigma}_{a_{n}-1}^{2} - \hat{\sigma}_{a_{n}}^{2}}{\hat{\sigma}_{a_{n}-1}^{2}} \\
    \geq& 1 - \frac{\hat{\sigma}_{a_{n}-1}^{2} - \zeta_{n}^{2} - (\hat{\sigma}_{a_{n}}^{2} - \zeta_{n}^{2})}{\zeta_{n}^{2} - 4L_{0}\xi_{E}n^{-1}d_{n}^{-1/2}} \\
    \geq& 1 - \frac{1}{\zeta_{n}^{2} - 4L_{0}\xi_{E}n^{-1}d_{n}^{-1/2}}\frac{16L_{0}\xi_{E}}{nd_{n}^{1/2}}.
\end{align*}
By (C4), we have 
\begin{align*}
    \mathbb{P}(\hat{k} > a_{n},A_{n}^{c}) \leq& \mathbb{P}\left(Cs_{n}^{-2} \leq \frac{1}{\zeta_{n}^{2} - 4L_{0}\xi_{E}n^{-1}d_{n}^{-1/2}}\frac{16L_{0}\xi_{E}}{nd_{n}^{1/2}} \right) = o(1),
\end{align*}
which completes the proof.
\end{proof}

Before proving Theorem \ref{3.thm2}, we introduce the following uniform convergence rate for the second-stage RGA, which is also of independent interest.

\begin{theorem} \label{6.thm3}
Assume the same as Theorem \ref{3.thm1}, and additionally (C5) and (C6) hold. 
The second-stage RGA satisfies
\begin{align} \label{6.thm3.0}
	\max_{1 \leq m \leq K_{n} } \frac{(nd_{n})^{-1}\Vert \tilde{\mathbf{Y}} - \hat{\mathbf{G}}^{(m)} \Vert_{F}^{2}}{\left(1 - \frac{\tau^{2}}{64\mu^{5}\kappa_{n}}\right)^{m} + \frac{(m + \kappa_{n}) \xi_{E}^{2}}{n^{2}d_{n}} + \frac{\xi_{E}^{2}}{\delta_{n}^{2}n^{2}} \mathbf{1}\{J_{o} \neq \emptyset\}} = O_{p}(1),
\end{align}
where $\tau < 1$ is an absolute constant.
\end{theorem}
\begin{proof}
By Theorem \ref{3.thm1}, we can assume $\mathrm{rank}(\mathbf{B}_{j}^{*}) \leq \hat{r}_{j}$ holds for all $j$ in the following analysis. 
Let $1 \leq m \leq K_{n}$ be arbitrary. 
Observe that for the second-stage RGA, each $\hat{\mathbf{G}}^{(k)}$, $k=1,2,\ldots$, lies in the set
\begin{align} \label{setCL}
    \mathcal{C}_{L} = \left\{ \mathbf{H}=\sum_{j\in\hat{J}}\mathbf{X}_{j}\hat{\mathbf{\Sigma}}_{j}^{-1}\mathbf{U}_{j}\mathbf{D}_{j}\mathbf{V}_{j}^{\top}: \sum_{j \in \hat{J}} \Vert \mathbf{D}_{j} \Vert_{*} \leq L_{n} \right\}.
\end{align}
By \eqref{6.5} and a similar argument as \eqref{6.lem1.1}-\eqref{6.lem1.2}, we have, for all $1 \leq k \leq m$,
\begin{align} \label{6.thm3.1}
	&\langle \tilde{\mathbf{Y}} - \hat{\mathbf{G}}^{(k-1)}, \mathbf{\mathbf{X}}_{\hat{j}_{k}}\hat{\mathbf{\Sigma}}_{\hat{j}_{k}}^{-1}\mathbf{U}_{\hat{j}_{k}}\hat{\mathbf{S}}_{k}\mathbf{V}_{\hat{j}_{k}}^{T} - \hat{\mathbf{G}}^{(k-1)} \rangle \notag \\
	\geq& \tau \max_{\substack{j \in \hat{J}_{\hat{k}} \\ \Vert \mathbf{S} \Vert_{*} \leq L_{n} }}
		\langle \tilde{\mathbf{Y}} - \hat{\mathbf{G}}^{(k-1)}, \mathbf{\mathbf{X}}_{j}\hat{\mathbf{\Sigma}}_{j}^{-1}\mathbf{U}_{j}\mathbf{S}\mathbf{V}_{j}^{\top} - \hat{\mathbf{G}}^{(k-1)} \rangle \notag \\
	=& \tau \sup_{\mathbf{H} \in \mathcal{C}_{L}} 
		\langle \tilde{\mathbf{Y}} - \hat{\mathbf{G}}^{(k-1)}, \mathbf{H} - \hat{\mathbf{G}}^{(k-1)} \rangle,
\end{align}
where $\tau = 1 - 4\mu L_{0}/\tilde{\tau}$ and $\tilde{\tau} > 4\mu L_{0}$ on the event
\begin{align*}
	\mathcal{F}_{n}(m) = \left\{ \min_{1 \leq k \leq m} \max_{\substack{j \in \hat{J}_{\hat{k}} \\ \Vert \mathbf{S} \Vert_{*} \leq L_{n}}} \langle \tilde{\mathbf{Y}} - \hat{\mathbf{G}}^{(k-1)}, \mathbf{X}_{j}\hat{\mathbf{\Sigma}}_{j}^{-1}\mathbf{U}_{j}\mathbf{S}\mathbf{V}_{j}^{\top} - \hat{\mathbf{G}}^{(k-1)} \rangle > \tilde{\tau}d_{n}^{1/2}\xi_{E} \right\}.
\end{align*}
Define 
\begin{align*}
	\mathcal{B} =& \left\{ \mathbf{H} = \sum_{j \in \hat{J}_{\hat{k}} } \mathbf{X}_{j}\hat{\mathbf{\Sigma}}_{j}^{-1}\mathbf{U}_{j}\mathbf{D}_{j}\mathbf{V}_{j}^{\top}: \Vert \bar{\mathbf{Y}} - \mathbf{H} \Vert_{F}^{2} \leq \frac{9nd_{n}L_{0}^{2}}{16 \mu^{3}\kappa_{n}} \right\},
\end{align*}
where 
\begin{align} \label{ybardef}
    \bar{\mathbf{Y}} = \sum_{j \in \hat{J}_{o}} \mathbf{X}_{j}\hat{\mathbf{\Sigma}}_{j}^{-1}\mathbf{U}_{j}\mathbf{L}_{j}\mathbf{\Lambda}_{j}\mathbf{R}_{j}^{\top}\mathbf{V}_{j}^{\top} + \sum_{j \in \hat{J} - \hat{J}_{o}} \mathbf{X}_{j} \mathbf{B}_{j}^{*},
\end{align}
in which $\hat{J}_{o} = \{j \in \hat{J}: \hat{r} < \min\{q_{n,j}, d_{n}\}\}$, $\mathbf{\Lambda}_{j}$ are defined in (C6), and $\mathbf{L}_{j}$, $\mathbf{R}_{j}$ are $\hat{r} \times \bar{r}_{j}$ matrices such that $\mathbf{L}_{j}^{T}\mathbf{L}_{j} = \mathbf{I}_{\bar{r}_{j}} = \mathbf{R}_{j}^{T}\mathbf{R}_{j}$ to be specified later (recall that $\hat{r} \geq \bar{r}_{j} = \mathrm{rank}(\mathbf{X}_{j}^{\top}\tilde{\mathbf{Y}})$ because of Theorem \ref{3.thm1}). 
We claim that 
\begin{align} \label{6.thm3.2}
	\lim_{n\rightarrow\infty}\mathbb{P}( \mathcal{B} \subseteq \mathcal{C}_{L} ) = 1,
\end{align}
whose proof is relegated to Appendix \ref{App:details}.
Now put $\mathbf{H}^{(l)} = \hat{\mathbf{G}}^{(l)} + (1+\alpha_{l}) (\bar{\mathbf{Y}} - \hat{\mathbf{G}}^{(l)})$ for $l = 1, 2, \ldots$, where
\begin{align*}
	\alpha_{l} = \frac{3\sqrt{nd_{n}}L_{0}}{ 4 \mu^{3/2} \sqrt{\kappa_{n}} \Vert \bar{\mathbf{Y}} - \hat{\mathbf{G}}^{(l)} \Vert_{F} } \geq 0.
\end{align*}
Then \eqref{6.thm3.2} implies that $\mathbb{P}(\mathbf{H}^{(l)} \in \mathcal{C}_{L}, l=1,2,\ldots) \rightarrow 1$.
Thus by \eqref{6.thm3.1},
\begin{align} \label{6.thm3.3}
	&\langle \tilde{\mathbf{Y}} - \hat{\mathbf{G}}^{(k-1)}, \mathbf{\mathbf{X}}_{\hat{j}_{k}}\hat{\mathbf{\Sigma}}_{\hat{j}_{k}}^{-1}\mathbf{U}_{\hat{j}_{k}}\hat{\mathbf{S}}_{k}\mathbf{V}_{\hat{j}_{k}}^{\top} - \hat{\mathbf{G}}^{(k-1)} \rangle
	\notag \\
    &\geq \tau \langle \tilde{\mathbf{Y}} - \hat{\mathbf{G}}^{(k-1)}, \mathbf{H}^{(k-1)} - \hat{\mathbf{G}}^{(k-1)} \rangle 
\end{align}
for all $1\leq k \leq m$ on $\mathcal{F}_{n}(m)$ except for a vanishing event.
Put $\mathcal{H}_{n}(m) = \{\Vert \tilde{\mathbf{Y}} - \bar{\mathbf{Y}} \Vert_{F} < 2^{-1}\min_{1 \leq l \leq m} \Vert \bar{\mathbf{Y}} - \hat{\mathbf{G}}^{(l-1)} \Vert_{F} \}$. On $\mathcal{F}_{n}(m) \cap \mathcal{H}_{n}(m)$ except for a vanishing event, \eqref{6.thm3.3} and Cauchy-Schwarz inequality yield
\begin{align*}
    &\langle \tilde{\mathbf{Y}} - \hat{\mathbf{G}}^{(k-1)}, \mathbf{\mathbf{X}}_{\hat{j}_{k}}\hat{\mathbf{\Sigma}}_{\hat{j}_{k}}^{-1}\mathbf{U}_{\hat{j}_{k}}\hat{\mathbf{S}}_{k}\mathbf{V}_{\hat{j}_{k}}^{\top} - \hat{\mathbf{G}}^{(k-1)} \rangle \\
    \geq& \tau \langle \tilde{\mathbf{Y}} - \hat{\mathbf{G}}^{(k-1)}, \mathbf{H}^{(k-1)} - \hat{\mathbf{G}}^{(k-1)} \rangle \\
    \geq& \tau (1 + \alpha_{k-1}) \left\{ \Vert \bar{\mathbf{Y}} - \hat{\mathbf{G}}^{(k-1)} \Vert_{F}^{2} - \Vert \tilde{\mathbf{Y}} - \bar{\mathbf{Y}} \Vert_{F} \Vert \bar{\mathbf{Y}} - \hat{\mathbf{G}}^{(k-1)} \Vert_{F}\right\} \\
    \geq& \frac{\tau (1 + \alpha_{k-1})}{2} \Vert \bar{\mathbf{Y}} - \hat{\mathbf{G}}^{(k-1)} \Vert_{F}^{2} \geq 0
\end{align*}
for all $1 \leq k \leq m$.
Notice that $\Vert \bar{\mathbf{Y}} - \hat{\mathbf{G}}^{(k-1)} \Vert_{F} \geq (2/3)\Vert \tilde{\mathbf{Y}} - \hat{\mathbf{G}}^{(k-1)} \Vert_{F}$ for all $1 \leq k \leq m$ on $\mathcal{H}_{n}(m)$.
Hence, by Lemma \ref{aux-lem1}(ii), (iii), and a similar argument used in \eqref{6.lem1.3},
\begin{align*} 
	\Vert \tilde{\mathbf{Y}} - \hat{\mathbf{G}}^{(k)} \Vert_{F}^{2}
	\leq& \Vert \tilde{\mathbf{Y}} - \hat{\mathbf{G}}^{(k-1)} \Vert_{F}^{2} - \frac{\langle \tilde{\mathbf{Y}} - \hat{\mathbf{G}}^{(k-1)}, \mathbf{X}_{\hat{j}_{k}}\hat{\mathbf{\Sigma}}_{\hat{j}_{k}}^{-1} \mathbf{U}_{\hat{j}_{k}}\hat{\mathbf{S}}_{k}\mathbf{V}_{\hat{j}_{k}}^{\top} - \hat{\mathbf{G}}^{(k-1)} \rangle^{2}}{\Vert \mathbf{X}_{\hat{j}_{k}}\hat{\mathbf{\Sigma}}_{\hat{j}_{k}}^{-1} \mathbf{U}_{\hat{j}_{k}}\hat{\mathbf{S}}_{k}\mathbf{V}_{\hat{j}_{k}}^{\top} - \hat{\mathbf{G}}^{(k-1)} \Vert_{F}^{2}} \\
    &+ 2 \langle \mathbf{E}, \hat{\mathbf{G}}^{(k)} - \mathbf{G}^{(k)} \rangle \notag \\
	\leq& \Vert \tilde{\mathbf{Y}} - \hat{\mathbf{G}}^{(k-1)} \Vert_{F}^{2} - \frac{\tau^{2}\langle \tilde{\mathbf{Y}} - \hat{\mathbf{G}}^{(k-1)}, \mathbf{H}^{(k-1)} - \hat{\mathbf{G}}^{(k-1)} \rangle^{2}}{4n\mu L_{n}^{2}} + 2 \langle \mathbf{E}, \hat{\mathbf{G}}^{(k)} - \mathbf{G}^{(k)} \rangle \notag \\
	\leq& \Vert \tilde{\mathbf{Y}} - \hat{\mathbf{G}}^{(k-1)} \Vert_{F}^{2} - \frac{\tau^{2}(1+\alpha_{k-1})^{2}}{16n\mu L_{n}^{2}}\Vert \bar{\mathbf{Y}} - \hat{\mathbf{G}}^{(k-1)} \Vert_{F}^{4} + 2 \langle \mathbf{E}, \hat{\mathbf{G}}^{(k)} - \mathbf{G}^{(k)} \rangle \notag \\
	\leq& \Vert \tilde{\mathbf{Y}} - \hat{\mathbf{G}}^{(k-1)} \Vert_{F}^{2} - \frac{\tau^{2}\Vert \tilde{\mathbf{Y}} - \hat{\mathbf{G}}^{(k-1)} \Vert_{F}^{2}}{64\mu^{4}\kappa_{n}} \\
    &+ 2 (\hat{\lambda}_{k} - \lambda_{k})\langle \mathbf{E}, \mathbf{X}_{\hat{j}_{k}}\hat{\mathbf{\Sigma}}_{\hat{j}_{k}}^{-1}\mathbf{U}_{\hat{j}_{k}}\hat{\mathbf{S}}_{k}\mathbf{V}_{\hat{j}_{k}}^{\top} - \hat{\mathbf{G}}^{(k-1)} \rangle \notag \\
	\\
	\leq& \Vert \tilde{\mathbf{Y}} - \hat{\mathbf{G}}^{(k-1)} \Vert_{F}^{2} \left(1 - \frac{\tau^{2}}{64\mu^{4}\kappa_{n}} \right) + \frac{8\mu}{1 - \epsilon_{L}} \frac{\xi_{E}^{2}}{n} 
\end{align*}
for all $1 \leq k \leq m$ on $\mathcal{F}_{n}(m) \cap \mathcal{H}_{n}(m)$ except for a vanishing event.
It follows that, on the same event, 
\begin{align} \label{6.thm3.4}
    \Vert \tilde{\mathbf{Y}} - \hat{\mathbf{G}}^{(m)} \Vert_{F}^{2} \leq \Vert \tilde{\mathbf{Y}} \Vert_{F}^{2} \left(1 - \frac{\tau^{2}}{64\mu^{4}\kappa_{n}} \right)^{m} + \frac{8\mu}{1 - \epsilon_{L}} \frac{m\xi_{E}^{2}}{n}. 
\end{align}
By \eqref{6.5}, on $\mathcal{F}_{n}^{c}(m) \cap \mathcal{H}_{n}(m)$ there exists some $1 \leq k \leq m$ such that
\begin{align*}
    \tilde{\tau} d_{n}^{1/2} \xi_{E} \geq& \langle \tilde{\mathbf{Y}} - \hat{\mathbf{G}}^{(k-1)}, \mathbf{H}^{(k-1)} - \hat{\mathbf{G}}^{(k-1)} \rangle \\
    \geq& (1 + \alpha_{k-1})\langle \tilde{\mathbf{Y}} - \hat{\mathbf{G}}^{(k-1)}, \bar{\mathbf{Y}} - \hat{\mathbf{G}}^{(k-1)} \rangle \\
    \geq& \frac{1}{2}(1 + \alpha_{k-1})\Vert \bar{\mathbf{Y}} - \hat{\mathbf{G}}^{(k-1)} \Vert_{F}^{2} \\
    \geq& \frac{3 \sqrt{nd_{n}} L_{0}}{8\mu^{3/2}\sqrt{\kappa_{n}}} \Vert \bar{\mathbf{Y}} - \hat{\mathbf{G}}^{(k-1)} \Vert_{F},
\end{align*}
which implies
\begin{align}
    \Vert \tilde{\mathbf{Y}} - \hat{\mathbf{G}}^{(m)} \Vert_{F}^{2} \leq& \Vert \tilde{\mathbf{Y}} - \hat{\mathbf{G}}^{(k-1)} \Vert_{F}^{2} + \frac{8\mu}{1 - \epsilon_{L}} \frac{(m - k)\xi_{E}^{2}}{n} \notag \\
    \leq& 2\Vert \tilde{\mathbf{Y}} - \bar{\mathbf{Y}} \Vert_{F}^{2} + 2 \Vert \bar{\mathbf{Y}} - \hat{\mathbf{G}}^{(k-1)} \Vert_{F}^{2} + \frac{8\mu}{1 - \epsilon_{L}} \frac{(m - k)\xi_{E}^{2}}{n} \notag \\
    \leq& \frac{5}{2} \Vert \bar{\mathbf{Y}} - \hat{\mathbf{G}}^{(k-1)} \Vert_{F}^{2} + \frac{8\mu}{1 - \epsilon_{L}} \frac{(m - k)\xi_{E}^{2}}{n} \notag \\
    \leq& \left(\frac{160 \tilde{\tau}^{2} \mu^{3}}{9L^{2}} \kappa_{n} + \frac{8\mu}{1-\epsilon_{L}}(m - k) \right)
    \frac{\xi_{E}^{2}}{n}. \label{6.thm3.6}
\end{align}
Next, on $\mathcal{H}_{n}^{c}(m)$, there exists some $1 \leq k \leq m$ such that $\Vert \bar{\mathbf{Y}} - \hat{\mathbf{G}}^{(k-1)} \Vert_{F}^{2} \leq 4 \Vert \tilde{\mathbf{Y}} - \bar{\mathbf{Y}} \Vert_{F}^{2}$. By \eqref{6.4} and the parallelogram law,
\begin{align}
    \Vert \tilde{\mathbf{Y}} - \hat{\mathbf{G}}^{(m)} \Vert_{F}^{2} \leq& \Vert \tilde{\mathbf{Y}} - \hat{\mathbf{G}}^{(k-1)} \Vert_{F}^{2} + 2 \sum_{j=k}^{m}\langle \mathbf{E}, \hat{\mathbf{G}}^{(j)} - \mathbf{G}^{(j)} \rangle \notag \\
    \leq& 10 \Vert \tilde{\mathbf{Y}} - \bar{\mathbf{Y}} \Vert_{F}^{2} + \frac{8\mu}{1-\epsilon_{L}} \frac{(m - k)\xi_{E}^{2}}{n} \label{6.thm3.5}
\end{align}
on $\mathcal{H}_{n}^{c}(m)$ except for a vanishing event. 
Finally, note that \eqref{6.thm3.4}-\eqref{6.thm3.5} are valid for any choice of $\mathbf{L}_{j}$ and $\mathbf{R}_{j}$ so long as $\mathbf{L}_{j}^{\top}\mathbf{L}_{j} = \mathbf{I}_{\bar{r}_{j}} = \mathbf{R}_{j}^{\top}\mathbf{R}_{j}$, $j \in \hat{J}$.
In Appendix \ref{App:details}, we show that $\mathbf{L}_{j}$, $\mathbf{R}_{j}$, $j \in \hat{J}_{o}$, can be chosen so that
\begin{align} \label{6.thm3.7}
    \frac{1}{nd_{n}}\Vert \tilde{\mathbf{Y}} - \bar{\mathbf{Y}} \Vert_{F}^{2} \leq 8\mu L^{2} \frac{\xi_{E}^{2}}{(n\delta_{n} - \xi_{E})^{2}} = O_{p}\left( \frac{\xi_{E}^{2}}{n^{2}\delta_{n}^{2
    }} \right).
\end{align}
Hence, by \eqref{6.thm3.4}-\eqref{6.thm3.7}, the desired result follows.
\end{proof}

Now we are ready to prove our last main result. 

\begin{proof}[\textsc{Proof of Theorem \ref{3.thm2}}]
Note first that $\mathcal{C}_{L}$ (defined in \eqref{setCL}) is a convex compact set almost surely. 
Thus we can define $\mathbf{Y}^{*}$ to be the  orthogonal projection of $\mathbf{Y}$ onto $\mathcal{C}_{L}$.
Since $\hat{\mathbf{G}}^{(m)} \in \mathcal{C}_{L}$ and $\hat{\sigma}_{m}^{2} \leq \hat{\sigma}_{m_{n}}^{2}$ for $m \geq m_{n}$, it follows that for $m \geq m_{n}$,
\begin{align}
    \Vert \mathbf{Y}^{*} - \hat{\mathbf{G}}^{(m)} \Vert_{F}^{2} =& \Vert \mathbf{Y} - \hat{\mathbf{G}}^{(m)} \Vert_{F}^{2} - \Vert \mathbf{Y} - \mathbf{Y}^{*} \Vert_{F}^{2} + 2 \langle \mathbf{Y}^{*} - \mathbf{Y}, \mathbf{Y}^{*} - \hat{\mathbf{G}}^{(m)} \rangle \notag \\
    \leq& \Vert \mathbf{Y} - \hat{\mathbf{G}}^{(m_{n})} \Vert_{F}^{2} - \Vert \mathbf{Y} - \mathbf{Y}^{*} \Vert_{F}^{2} \notag  \\
    =& \Vert \mathbf{Y}^{*} - \hat{\mathbf{G}}^{(m_{n})} \Vert_{F}^{2} - 2 \langle \tilde{\mathbf{Y}} - \mathbf{Y}^{*}, \hat{\mathbf{G}}^{(m_{n})} - \mathbf{Y}^{*} \rangle - 2 \langle \mathbf{E}, \hat{\mathbf{G}}^{(m_{n})} - \mathbf{Y}^{*} \rangle \notag \\
    \leq& 2 \Vert \mathbf{Y}^{*} - \hat{\mathbf{G}}^{(m_{n})} \Vert_{F}^{2} + \Vert \mathbf{Y}^{*} - \tilde{\mathbf{Y}} \Vert_{F}^{2} - 2 \langle \mathbf{E}, \hat{\mathbf{G}}^{(m_{n})} - \mathbf{Y}^{*} \rangle. \label{6.thm2.1}
\end{align}
Note that if $\mathbf{H}, \mathbf{G}$ are in $\mathcal{C}_{L}$ with $\mathbf{H} = \sum_{j \in \hat{J}}\mathbf{X}_{j}\hat{\mathbf{\Sigma}}_{j}^{-1}\mathbf{U}_{j}\mathbf{S}_{j}^{H} \mathbf{V}_{j}^{\top}$ and $\mathbf{G} = \sum_{j \in \hat{J}}\mathbf{X}_{j}\hat{\mathbf{\Sigma}}_{j}^{-1}\mathbf{U}_{j}\mathbf{S}_{j}^{G} \mathbf{V}_{j}^{\top}$, then by Proposition \ref{aux-prop0} and (C3) we have
\begin{align*} 
    \Vert \mathbf{H} - \mathbf{G} \Vert_{F}^{2} \geq \frac{n}{\mu^{3}\kappa_{n}}\left\{ \sum_{j \in \hat{J}}\Vert \mathbf{S}_{j}^{H} - \mathbf{S}_{j}^{G} \Vert_{*} \right\}^{2}.
\end{align*}
Hence
\begin{align} \label{6.thm2.2}
    |\langle \mathbf{E}, \mathbf{H} - \mathbf{G} \rangle| \leq \mu \xi_{E} \sum_{j \in \hat{J}} \Vert \mathbf{S}_{j}^{H} - \mathbf{S}_{j}^{G} \Vert_{*} \leq  \xi_{E} \sqrt{\frac{\mu^{5}\kappa_{n}}{n}} \Vert \mathbf{H} - \mathbf{G} \Vert_{F}.
\end{align}
Combining \eqref{6.thm2.1} and \eqref{6.thm2.2} yields
\begin{align*}
    \Vert \mathbf{Y}^{*} - \hat{\mathbf{G}}^{(m)} \Vert_{F}^{2} \leq 2 \Vert \mathbf{Y}^{*} - \hat{\mathbf{G}}^{(m_{n})} \Vert_{F}^{2} + \Vert \mathbf{Y}^{*} - \tilde{\mathbf{Y}} \Vert_{F}^{2} + 2 \xi_{E} \sqrt{\frac{\mu^{5}\kappa_{n}}{n}} \Vert \mathbf{Y}^{*} - \hat{\mathbf{G}}^{(m)} \Vert_{F}.
\end{align*}
Since $x^{2} \leq c + bx$ ($x, b, c \geq 0$) implies $x \leq (b + \sqrt{b^{2} + 4c})/2$, we have
\begin{align} \label{6.thm2.3}
    \Vert \mathbf{Y}^{*} - \hat{\mathbf{G}}^{(m)} \Vert_{F}^{2} \leq 2 \Vert \mathbf{Y}^{*} - \tilde{\mathbf{Y}} \Vert_{F}^{2} + 4 \Vert \mathbf{Y}^{*} - \hat{\mathbf{G}}^{(m_{n})} \Vert_{F}^{2} + 4\mu^{5} \frac{\kappa_{n}\xi_{E}^{2}}{n}.
\end{align}
By \eqref{6.thm2.3} and repeated applications of the parallelogram law, it is straightforward to show
\begin{align*}
    \frac{1}{nd_{n}}\Vert \tilde{\mathbf{Y}} - \hat{\mathbf{G}}^{(m)} \Vert_{F}^{2} \leq \frac{C_{1}}{nd_{n}}\left\{\Vert \tilde{\mathbf{Y}} - \mathbf{Y}^{*} \Vert_{F}^{2} + \Vert \tilde{\mathbf{Y}} - \hat{\mathbf{G}}^{(m_{n})} \Vert_{F}^{2} + \frac{\mu^{5}\kappa_{n}\xi_{E}^{2}}{n} \right\}
\end{align*}
for some absolute constant $C_{1}$.
The right-hand side does not depend on $m$, so the inequality still holds if we take supremum over $m \geq m_{n}$ on the left-hand side. 
Moreover, by (C3) and Theorem \ref{3.thm1}, we have
\begin{align} \label{6.thm2.4}
    \sup_{m \geq m_{n}} \frac{1}{d_{n}}\sum_{j=1}^{p_{n}}\Vert \mathbf{B}_{j}^{*} - \hat{\mathbf{B}}_{j}^{(m)} \Vert_{F}^{2} = 
    O_{p}\left( \frac{1}{nd_{n}}\left\{ \Vert \tilde{\mathbf{Y}} - \mathbf{Y}^{*} \Vert_{F}^{2} + \Vert \tilde{\mathbf{Y}} - \hat{\mathbf{G}}^{(m_{n})} \Vert_{F}^{2} + \frac{\mu^{5}\kappa_{n}\xi_{E}^{2}}{n} \right\} \right)     
\end{align}
By Theorem \ref{6.thm3} and the choice of $m_{n}$, we have
\begin{align} \label{6.thm2.5}
    \frac{1}{nd_{n}} \Vert \tilde{\mathbf{Y}} - \hat{\mathbf{G}}^{(m_{n})} \Vert_{F}^{2} = O_{p}\left( \frac{\kappa_{n}\xi_{n}^{2}}{n^{2}d_{n}} \log \frac{n^{2}d_{n}}{\xi_{n}^{2}} + \frac{\xi_{n}^{2}}{n^{2}\delta_{n}^{2}} \right).
\end{align}
By (C6), it is not difficult to show $\bar{\mathbf{Y}}$, defined in \eqref{ybardef}, is in $\mathcal{C}_{L}$.
It follows from the definition of $\mathbf{Y}^{*}$ that
\begin{align}
    \Vert \tilde{\mathbf{Y}} - \mathbf{Y}^{*} \Vert_{F}^{2} =& \Vert \mathbf{Y} - \mathbf{Y}^{*} \Vert_{F}^{2} - \Vert \mathbf{E} \Vert_{F}^{2} - 2 \langle \mathbf{E}, \tilde{\mathbf{Y}} - \mathbf{Y}^{*} \rangle \notag \\
    \leq& \Vert \mathbf{Y} - \bar{\mathbf{Y}} \Vert_{F}^{2} - \Vert \mathbf{E} \Vert_{F}^{2} - 2 \langle \mathbf{E}, \tilde{\mathbf{Y}} - \mathbf{Y}^{*} \rangle \notag \\
    =& \Vert \tilde{\mathbf{Y}} - \bar{\mathbf{Y}} \Vert_{F}^{2} + 2 \langle \mathbf{E}, \mathbf{Y}^{*} - \bar{\mathbf{Y}} \rangle. \label{6.thm2.6}
\end{align}
By \eqref{6.thm2.2} again,
\begin{align} \label{6.thm2.7}
    |\langle \mathbf{E}, \mathbf{Y}^{*} - \bar{\mathbf{Y}} \rangle| \leq \xi_{E} \left( \frac{\mu^{5}\kappa_{n}}{n} \right)^{1/2} \Vert \bar{\mathbf{Y}} - \mathbf{Y}^{*} \Vert_{F}.
\end{align}
Now if $\Vert \bar{\mathbf{Y}} - \mathbf{Y}^{*} \Vert_{F} \geq 2 \Vert \tilde{\mathbf{Y}} - \mathbf{Y}^{*} \Vert_{F}$, then $\Vert \bar{\mathbf{Y}} - \mathbf{Y}^{*} \Vert_{F} \leq 2 \Vert \bar{\mathbf{Y}} - \tilde{\mathbf{Y}} \Vert_{F}$. This, together with \eqref{6.thm2.6}, \eqref{6.thm2.7}, and \eqref{6.thm3.7}, yields
\begin{align}
    \Vert \tilde{\mathbf{Y}} - \mathbf{Y}^{*} \Vert_{F}^{2} \leq& \Vert \tilde{\mathbf{Y}} - \bar{\mathbf{Y}} \Vert_{F}^{2} + 4 \xi_{E} \left( \frac{\mu^{5}\kappa_{n}}{n} \right)^{1/2} \Vert \bar{\mathbf{Y}} - \tilde{\mathbf{Y}} \Vert_{F} \notag \\
    \leq& 2\Vert \tilde{\mathbf{Y}} - \bar{\mathbf{Y}} \Vert_{F}^{2} + 4 \mu^{5} \frac{\kappa_{n}\xi_{E}^{2}}{n} \notag \\
    \leq& 16 \mu L_{0}^{2} \frac{nd_{n}\xi_{E}^{2}}{(n\delta_{n} - \xi_{E})^{2}} + 4 \mu^{5} \frac{\kappa_{n}\xi_{E}^{2}}{n}. \label{6.thm2.8}
\end{align}
On the other hand, if $\Vert \bar{\mathbf{Y}} - \mathbf{Y}^{*} \Vert_{F} < 2 \Vert \tilde{\mathbf{Y}} - \mathbf{Y}^{*} \Vert_{F}$, then \eqref{6.thm2.6} and \eqref{6.thm2.7} imply
\begin{align*}
    \Vert \tilde{\mathbf{Y}} - \mathbf{Y}^{*} \Vert_{F}^{2} \leq& \Vert \tilde{\mathbf{Y}} - \bar{\mathbf{Y}} \Vert_{F}^{2} + 4 \xi_{E} \left( \frac{\mu^{5}\kappa_{n}}{n} \right)^{1/2} \Vert \tilde{\mathbf{Y}} - \mathbf{Y}^{*} \Vert_{F}.
\end{align*}
By a similar argument used to obtain \eqref{6.thm2.3}, this and \eqref{6.thm3.7} yield
\begin{align}
    \Vert \tilde{\mathbf{Y}} - \mathbf{Y}^{*} \Vert_{F}^{2} \leq& 16 \mu^{5} \frac{\kappa_{n}\xi_{E}^{2}}{n} + 2 \Vert \tilde{\mathbf{Y}} - \bar{\mathbf{Y}} \Vert_{F}^{2} \notag \\
    \leq& 16 \mu^{5} \frac{\kappa_{n}\xi_{E}^{2}}{n} + 16 \mu L_{0}^{2} \frac{nd_{n}\xi_{E}^{2}}{(n \delta_{n} - \xi_{E})^{2}}. \label{6.thm2.9}
\end{align}
In view of \eqref{6.thm2.4}, \eqref{6.thm2.5}, \eqref{6.thm2.8}, \eqref{6.thm2.9} and (C5), the deisred result follows. 
\end{proof}

\section{Further technical details} \label{App:details}
In this section, we present some additional auxiliary results along with the proofs of \eqref{6.lem1.5}, \eqref{6.thm1.2}, \eqref{6.thm1.3}, \eqref{6.thm3.2}, \eqref{6.thm3.7}.
Some existing results that are useful in our proofs are also stated here for completeness with the references to their proofs in the literature. 
These results are stated in the forms that are most convenient for our use, which may not be in full generality.

\begin{proposition}[\citealp{ruhe1970}] \label{aux-prop0}
Let $\mathbf{A}, \mathbf{B}$ be matrices with size $m \times n$ and $n \times p$ respectively. Then
\begin{align*}
     \sum_{j=1}^{n} \sigma_{j}^{2}(\mathbf{A})\sigma_{j}^{2}(\mathbf{B}) \geq \Vert \mathbf{AB} \Vert_{F}^{2} \geq \sum_{j=1}^{n} \sigma_{n-j+1}^{2}(\mathbf{A})\sigma_{j}^{2}(\mathbf{B}).
\end{align*}
\end{proposition}
\begin{remark} \normalfont
    One consequence of this inequality we frequently use is $\sigma_{1}^{2}(\mathbf{A}) \Vert \mathbf{B} \Vert_{F}^{2} \geq \Vert \mathbf{AB} \Vert_{F}^{2} \geq \sigma_{n}^{2}(\mathbf{A}) \Vert \mathbf{B} \Vert_{F}^{2}$.
    Note also that by transposition the roles of $\mathbf{A}$ and $\mathbf{B}$ can be interchanged on the left- and right-most expressions. 
\end{remark}

\begin{lemma} \label{aux-lem1}
Assume (C1)-(C2) and that $\sum_{j=1}^{p_{n}} \Vert \mathbf{B}_{j}^{*} \Vert_{*} \leq L$. 
Suppose $L_{n} = d_{n}^{1/2}L_{0}$ is chosen so that $L_{0} \geq L / (1 - \epsilon_{L})$ with $1 - \epsilon_{L} \leq 1 / (4\mu^{2})$.
Then for first- and second-stage RGA, with probability tending to one,

    (i) 
    \begin{align} 
        \inf_{k \geq 1} \frac{1}{nd_{n}} \Vert \mathbf{X}_{\hat{j}_{k}}\tilde{\mathbf{B}}_{\hat{j}_{k}} - \hat{\mathbf{G}}^{(k-1)} \Vert_{F}^{2} \geq& (1 - \epsilon_{L})\mu L_{0}^{2} \label{lem1-1-1} \\
        \inf_{k \geq 1} \frac{1}{nd_{n}} \Vert \mathbf{X}_{\hat{j}_{k}}\hat{\mathbf{\Sigma}}_{\hat{j}_{k}}^{-1} \mathbf{U}_{\hat{j}_{k}} \hat{\mathbf{S}}_{k} \mathbf{V}_{\hat{j}_{k}}^{\top} - \hat{\mathbf{G}}^{(k-1)} \Vert_{F}^{2} \geq& (1 - \epsilon_{L})\mu L_{0}^{2} \label{lem1-1-2}
    \end{align}
    
    (ii) 
    \begin{align} 
        \sup_{k \geq 1} |\lambda_{k} - \hat{\lambda}_{k}| \leq& \frac{2}{(1 - \epsilon_{L})L_{0}} \frac{\xi_{E}}{n\sqrt{d_{n}}} \label{lem1-2-2}
    \end{align}

    (iii) 
    \begin{align} \label{lem1-3}
        \max_{1 \leq k \leq K_{n}}\lambda_{k} \leq 1.
    \end{align}
\end{lemma}
\begin{proof}
We shall prove the results for the second-stage RGA. The corresponding proofs for first-stage RGA follow similarly and thus are omitted. 
It is also sufficient to prove (i)-(iii) assuming the condition described in (C1) holds almost surely because the event that the condition holds has probability tending to one. It will greatly simplify the exposition (without repeating that the inequalities holds except on a vanishing event).
Note that
\begin{align*}
    &\langle \mathbf{X}_{\hat{j}_{k}} \hat{\mathbf{\Sigma}}_{\hat{j}_{k}}^{-1} \mathbf{U}_{\hat{j}_{k}} \hat{\mathbf{S}}_{k} \mathbf{V}_{\hat{j}_{k}}^{\top}, \tilde{\mathbf{Y}} - \hat{\mathbf{G}}^{(k-1)} \rangle \\
    =& \langle \mathbf{X}_{\hat{j}_{k}} \hat{\mathbf{\Sigma}}_{\hat{j}_{k}}^{-1} \mathbf{U}_{\hat{j}_{k}} \hat{\mathbf{S}}_{k} \mathbf{V}_{\hat{j}_{k}}^{\top}, \mathbf{Y} - \hat{\mathbf{G}}^{(k-1)} \rangle - \langle \mathbf{X}_{\hat{j}_{k}} \hat{\mathbf{\Sigma}}_{\hat{j}_{k}}^{-1} \mathbf{U}_{\hat{j}_{k}} \hat{\mathbf{S}}_{k} \mathbf{V}_{\hat{j}_{k}}^{\top}, \mathbf{E} \rangle \\
    \geq& -|\langle \mathbf{X}_{\hat{j}_{k}} \hat{\mathbf{\Sigma}}_{\hat{j}_{k}}^{-1} \mathbf{U}_{\hat{j}_{k}} \hat{\mathbf{S}}_{k} \mathbf{V}_{\hat{j}_{k}}^{\top}, \mathbf{E} \rangle| \\
    \geq& - \Vert \hat{\mathbf{\Sigma}}_{\hat{j}_{k}}^{-1} \mathbf{X}_{\hat{j}_{k}}^{\top}\mathbf{E}\Vert_{op} \Vert \mathbf{U}_{\hat{j}_{k}} \hat{\mathbf{S}}_{k} \mathbf{V}_{\hat{j}_{k}}^{\top} \Vert_{*} \\
    \geq& -\mu L_{n}\xi_{E},
\end{align*}
where the first inequality follows because $\langle \mathbf{X}_{\hat{j}_{k}} \hat{\mathbf{\Sigma}}_{\hat{j}_{k}}^{-1} \mathbf{U}_{\hat{j}_{k}} \hat{\mathbf{S}}_{k} \mathbf{V}_{\hat{j}_{k}}^{T}, \mathbf{Y} - \hat{\mathbf{G}}^{(k-1)} \rangle \geq 0$ with probability one and the second inequality follows because the dual norm of the nuclear norm is the operator norm. 
By Proposition \ref{aux-prop0}, we have
\begin{align*}
    &\Vert \mathbf{X}_{\hat{j}_{k}} \hat{\mathbf{\Sigma}}_{\hat{j}_{k}}^{-1} \mathbf{U}_{\hat{j}_{k}} \hat{\mathbf{S}}_{k} \mathbf{V}_{\hat{j}_{k}}^{\top} - \hat{\mathbf{G}}^{(k-1)} \Vert_{F}^{2} \\
    \geq& \Vert \mathbf{X}_{\hat{j}_{k}} \hat{\mathbf{\Sigma}}_{\hat{j}_{k}}^{-1} \mathbf{U}_{\hat{j}_{k}} \hat{\mathbf{S}}_{k} \mathbf{V}_{\hat{j}_{k}}^{\top} \Vert_{F}^{2} - 2 \langle \mathbf{X}_{\hat{j}_{k}} \hat{\mathbf{\Sigma}}_{\hat{j}_{k}}^{-1} \mathbf{U}_{\hat{j}_{k}} \hat{\mathbf{S}}_{k} \mathbf{V}_{\hat{j}_{k}}^{\top}, \hat{\mathbf{G}}^{(k-1)} \rangle \notag \\
    \geq& n \mu^{-1} \Vert \mathbf{U}_{\hat{j}_{k}}\hat{\mathbf{S}}_{k} \mathbf{V}_{\hat{j}_{k}}^{\top} \Vert_{F}^{2} \notag \\
    &+ 2\langle \mathbf{X}_{\hat{j}_{k}} \hat{\mathbf{\Sigma}}_{\hat{j}_{k}}^{-1} \mathbf{U}_{\hat{j}_{k}} \hat{\mathbf{S}}_{k} \mathbf{V}_{\hat{j}_{k}}^{\top}, \tilde{\mathbf{Y}} - \hat{\mathbf{G}}^{(k-1)} \rangle - 2 \langle \mathbf{X}_{\hat{j}_{k}} \hat{\mathbf{\Sigma}}_{\hat{j}_{k}}^{-1} \mathbf{U}_{\hat{j}_{k}} \hat{\mathbf{S}}_{k} \mathbf{V}_{\hat{j}_{k}}^{\top}, \tilde{\mathbf{Y}} \rangle \notag \\
    \geq& n \mu^{-1} L_{n}^{2} - 2\mu L_{n}\xi_{E} - 2 \langle \mathbf{X}_{\hat{j}_{k}} \hat{\mathbf{\Sigma}}_{\hat{j}_{k}}^{-1} \mathbf{U}_{\hat{j}_{k}} \hat{\mathbf{S}}_{k} \mathbf{V}_{\hat{j}_{k}}^{\top}, \tilde{\mathbf{Y}} \rangle,
\end{align*}
where the last inequality follows from the fact that $\hat{\mathbf{S}}_{k}$ is rank-one with singular value $L_{n}$. Thus, by writing $\hat{\mathbf{S}}_{k} = L_{n} \mathbf{a}\mathbf{b}^{T}$ for some unit vectors $\mathbf{a}, \mathbf{b}$, we have $\Vert \mathbf{U}_{\hat{j}_{k}}\hat{\mathbf{S}}_{k} \mathbf{V}_{\hat{j}_{k}}^{T} \Vert_{F}^{2} = L_{n}^{2} \Vert \mathbf{U}_{\hat{j}_{k}}\mathbf{a}\mathbf{b}^{T}\mathbf{V}_{\hat{j}_{k}}^{T} \Vert_{F}^{2} = L_{n}^{2}$.
Next, observe that
\begin{align*}
    |\langle \mathbf{X}_{\hat{j}_{k}} \hat{\mathbf{\Sigma}}_{\hat{j}_{k}}^{-1} \mathbf{U}_{\hat{j}_{k}} \hat{\mathbf{S}}_{k} \mathbf{V}_{\hat{j}_{k}}^{\top}, \tilde{\mathbf{Y}} \rangle | =& \left\vert \sum_{j=1}^{p_{n}}\langle \mathbf{X}_{\hat{j}_{k}} \hat{\mathbf{\Sigma}}_{\hat{j}_{k}}^{-1} \mathbf{U}_{\hat{j}_{k}} \hat{\mathbf{S}}_{k} \mathbf{V}_{\hat{j}_{k}}^{\top}, \mathbf{X}_{j}\mathbf{B}_{j}^{*} \rangle \right\vert \\
    \leq& \sum_{j=1}^{p_{n}}\Vert \mathbf{B}_{j}^{*} \Vert_{*} \Vert \mathbf{X}_{j}^{\top}\mathbf{X}_{\hat{j}_{k}} \hat{\mathbf{\Sigma}}_{\hat{j}_{k}}^{-1} \mathbf{U}_{\hat{j}_{k}} \hat{\mathbf{S}}_{k} \mathbf{V}_{\hat{j}_{k}}^{\top}\Vert_{op}  \\
    \leq& (1 - \epsilon_{L})L_{n}^{2}n\mu .
\end{align*}
Therefore, 
\begin{align*}
    (nd_{n})^{-1}\Vert \mathbf{X}_{\hat{j}_{k}} \hat{\mathbf{\Sigma}}_{\hat{j}_{k}}^{-1} \mathbf{U}_{\hat{j}_{k}} \hat{\mathbf{S}}_{k} \mathbf{V}_{\hat{j}_{k}}^{\top} - \hat{\mathbf{G}}^{(k-1)} \Vert_{F}^{2} \geq& \mu^{-1} L_{0}^{2} - 2 (1 - \epsilon_{L})L_{0}^{2}\mu - 2\mu L_{0} \frac{\xi_{E}}{n\sqrt{d_{n}}} \\
    \geq& 2(1 - \epsilon_{L})L_{0}^{2}\mu - 2\mu L_{0} \frac{\xi_{E}}{n\sqrt{d_{n}}}.
\end{align*}
Since $\xi_{E} = o_{p}(n\sqrt{d_{n}})$ by (C2), \eqref{lem1-1-2} follows.

For \eqref{lem1-2-2}, note first that if the solutions to the line search problems \eqref{2.3.3} and \eqref{6.2} (with $\tilde{\mathbf{B}}_{\hat{j}_{k}}$ replaced by $\hat{\mathbf{\Sigma}}_{\hat{j}_{k}}^{-1}\mathbf{U}_{\hat{j}_{k}}\hat{\mathbf{S}}_{k}\mathbf{V}_{\hat{j}_{k}}^{\top}$) for second-stage RGA are not constrained to be in $[0,1]$, then they are given by
\begin{align*}
    \hat{\lambda}_{k,uc} =& \frac{\langle \mathbf{Y} - \hat{\mathbf{G}}^{(k-1)}, \mathbf{X}_{\hat{j}_{k}} \hat{\mathbf{\Sigma}}_{\hat{j}_{k}}^{-1} \mathbf{U}_{\hat{j}_{k}} \hat{\mathbf{S}}_{k} \mathbf{V}_{\hat{j}_{k}}^{\top} - \hat{\mathbf{G}}^{(k-1)} \rangle}{\Vert \mathbf{X}_{\hat{j}_{k}} \hat{\mathbf{\Sigma}}_{\hat{j}_{k}}^{-1} \mathbf{U}_{\hat{j}_{k}} \hat{\mathbf{S}}_{k} \mathbf{V}_{\hat{j}_{k}}^{\top} - \hat{\mathbf{G}}^{(k-1)} \Vert_{F}^{2}}, \\
    \lambda_{k,uc} =& \frac{\langle \tilde{\mathbf{Y}} - \hat{\mathbf{G}}^{(k-1)}, \mathbf{X}_{\hat{j}_{k}} \hat{\mathbf{\Sigma}}_{\hat{j}_{k}}^{-1} \mathbf{U}_{\hat{j}_{k}} \hat{\mathbf{S}}_{k} \mathbf{V}_{\hat{j}_{k}}^{\top} - \hat{\mathbf{G}}^{(k-1)} \rangle}{\Vert \mathbf{X}_{\hat{j}_{k}} \hat{\mathbf{\Sigma}}_{\hat{j}_{k}}^{-1} \mathbf{U}_{\hat{j}_{k}} \hat{\mathbf{S}}_{k} \mathbf{V}_{\hat{j}_{k}}^{\top} - \hat{\mathbf{G}}^{(k-1)} \Vert_{F}^{2}}.
\end{align*}
Since $\hat{\mathbf{G}}^{(l)}$ can always be expressed as $
\hat{\mathbf{G}}^{(l)} = \sum_{j \in \hat{J}} \mathbf{X}_{j} \hat{\mathbf{\Sigma}}_{j}^{-1} \mathbf{U}_{j} \mathbf{A}_{j} \mathbf{V}_{j}^{\top}$ with $\sum_{j\in \hat{J}}\Vert \mathbf{A}_{j} \Vert_{*} \leq L_{n}$,  
it follows that
\begin{align*}
    |\hat{\lambda}_{k} - \lambda_{k}| \leq |\hat{\lambda}_{k,uc} - \lambda_{k,uc}| =& \frac{|\langle \mathbf{E}, \mathbf{X}_{\hat{j}_{k}} \hat{\mathbf{\Sigma}}_{\hat{j}_{k}}^{-1} \mathbf{U}_{\hat{j}_{k}} \hat{\mathbf{S}}_{k} \mathbf{V}_{\hat{j}_{k}}^{\top} - \hat{\mathbf{G}}^{(k-1)} \rangle|}{\Vert \mathbf{X}_{\hat{j}_{k}} \hat{\mathbf{\Sigma}}_{\hat{j}_{k}}^{-1} \mathbf{U}_{\hat{j}_{k}} \hat{\mathbf{S}}_{k} \mathbf{V}_{\hat{j}_{k}}^{\top} - \hat{\mathbf{G}}^{(k-1)} \Vert_{F}^{2}} \\
    \leq& \frac{2L_{n}\mu\xi_{E}}{\Vert \mathbf{X}_{\hat{j}_{k}} \hat{\mathbf{\Sigma}}_{\hat{j}_{k}}^{-1} \mathbf{U}_{\hat{j}_{k}} \hat{\mathbf{S}}_{k} \mathbf{V}_{\hat{j}_{k}}^{\top} - \hat{\mathbf{G}}^{(k-1)} \Vert_{F}^{2}} \\
    \leq& \frac{2\xi_{E}}{nd_{n}^{1/2}(1 - \epsilon_{L})L_{0} },
\end{align*}
with probability tending to one, where the last inequality follows from \eqref{lem1-1-2}.

For \eqref{lem1-3}, it suffices to prove that $\lim_{n \rightarrow \infty}\mathbb{P}(E_{n}) = 1$, where
$E_{n} = \left\{\max_{1 \leq k \leq K_{n}} \lambda_{k,uc} \leq 1 \right\}$.
On $E_{n}^{c}$, there exists some $k$ such that, by Cauchy-Schwarz inequality and \eqref{6.4},
\begin{align} \label{lem1-pf-1}
    \Vert \mathbf{X}_{\hat{j}_{k}}\hat{\mathbf{\Sigma}}_{\hat{j}_{k}}^{-1}\mathbf{U}_{\hat{j}_{k}}\hat{\mathbf{S}}_{k}\mathbf{V}_{\hat{j}_{k}}^{T} - \hat{\mathbf{G}}^{(k-1)} \Vert_{F}^{2} \leq& \Vert \tilde{\mathbf{Y}} - \hat{\mathbf{G}}^{(k-1)} \Vert_{F}^{2} \notag \\
    \leq& \Vert \tilde{\mathbf{Y}} \Vert_{F}^{2} + 2 \sum_{j=1}^{k-1} \langle \mathbf{E}, \hat{\mathbf{G}}^{(k-j)} - \mathbf{G}^{(k-j)} \rangle \notag \\
    =& \Vert \tilde{\mathbf{Y}} \Vert_{F}^{2} + 2 \sum_{l=1}^{k-1} (\hat{\lambda}_{l} - \lambda_{l}) \langle \mathbf{E}, \mathbf{X}_{\hat{j}_{l}}\hat{\mathbf{\Sigma}}_{\hat{j}_{l}}^{-1}\mathbf{U}_{\hat{j}_{l}}\hat{\mathbf{S}}_{l}\mathbf{V}_{\hat{j}_{l}}^{\top} - \hat{\mathbf{G}}^{(l-1)} \rangle \notag \\
    \leq& \Vert \tilde{\mathbf{Y}} \Vert_{F}^{2} + 4K_{n} L_{n} \mu \xi_{E} \max_{1 \leq l \leq k}|\hat{\lambda}_{l}-\lambda_{l}|.
\end{align}
It is easy to see that
\begin{align} \label{lem1-pf-2}
    \Vert \tilde{\mathbf{Y}} \Vert_{F} = \left\Vert \sum_{j=1}^{p_{n}}\mathbf{X}_{j}\mathbf{B}_{j}^{*} \right\Vert_{F}  \leq (1 - \epsilon_{L})L_{n}\sqrt{n \mu}. 
\end{align}
Thus, by \eqref{lem1-1-2}, \eqref{lem1-2-2} and \eqref{lem1-pf-1}-\eqref{lem1-pf-2}, we have
\begin{align*} 
    \mathbb{P}(E_{n}^{c}) \leq \mathbb{P}\left( (1 - \epsilon_{L})L_{0}^{2}\mu\{1 - (1 - \epsilon_{L})\} \leq \frac{8\mu}{1 - \epsilon_{L}} \frac{K_{n}\xi_{E}^{2}}{n^{2}d_{n}} \right) + o(1) = o(1),
\end{align*}
where the last equality follows from (C2).
\end{proof}

\begin{lemma} \label{aux-lem2}
Let $\{a_{m}\}$ be a nonnegative sequence of reals. If
\begin{align*}
    a_{0} \leq A, \mbox{ and } a_{m} \leq a_{m-1}\left(1 - \frac{\xi^{2}a_{m-1}}{A} \right) + b_{m},
\end{align*}
for $m=1,2,\ldots,$ where $b_{m} \geq 0$ with $b_{0}=0$, then for each $m$,
\begin{align} \label{lem2-1}
    a_{m} \leq \frac{A}{1+m\xi^{2}} + \sum_{k=0}^{m} b_{k}.
\end{align}
\end{lemma}
\begin{proof}
We prove by induction. When $m=0$, \eqref{lem2-1} holds by assumption. Suppose now that \eqref{lem2-1} holds for some $m \geq 1$. Then
\begin{align*}
    a_{m+1} \leq& a_{m}\left( 1-\frac{\xi^{2}a_{m}}{A} \right) + b_{m+1} \\
    \leq& \frac{1}{a_{m}^{-1}+ \xi^{2}/A} + b_{m+1} \\
    \leq& \frac{1}{\left( \frac{A}{1+m\xi^{2}} + \sum_{k=0}^{m}b_{k}\right)^{-1} + \xi^{2}/A} + b_{m+1} \\
    =& \frac{ \frac{A}{1+m\xi^{2}} + \sum_{k=0}^{m}b_{k}}{1+\frac{\xi^{2}}{A}\left(  \frac{A}{1+m\xi^{2}} + \sum_{k=0}^{m}b_{k} \right)} + b_{m+1} \\
    \leq& \frac{A}{1+(m+1)\xi^{2}} + \sum_{k=0}^{m+1}b_{k},
\end{align*}
where the second inequality follows from $1-x \leq 1/(1+x)$ for $x \geq 0$. 
\end{proof}
\begin{remark} \normalfont
Lemma \ref{aux-lem2} is a slight modification of Lemma 3.1 of \citet{temlyakov2000}.
\end{remark}

\begin{proof}[\textsc{Proof of \eqref{6.lem1.5}}]
On $\mathcal{E}_{n}^{c}(m)$, there exists some $l \leq m$ such that
\begin{align*}
    \tilde{\tau}d_{n}^{1/2}\xi_{E} \geq \max_{\substack{1 \leq j \leq p_{n} \\ \Vert \mathbf{B}_{j} \Vert_{*} \leq L_{n}}} \langle \tilde{\mathbf{Y}} - \hat{\mathbf{G}}^{(l - 1)}, \mathbf{X}_{j}\mathbf{B}_{j} - \hat{\mathbf{G}}^{(l - 1)} \rangle 
    \geq \Vert \tilde{\mathbf{Y}} - \hat{\mathbf{G}}^{(l - 1)} \Vert_{F}^{2}.
\end{align*}
By \eqref{6.4} and Lemma \ref{aux-lem1}(ii), it follows that, on $\mathcal{E}_{n}^{c}(m)$ except for a vanishing event, 
\begin{align*} 
    \Vert \tilde{\mathbf{Y}} - \hat{\mathbf{G}}^{(m)} \Vert_{F}^{2} \leq& \Vert \tilde{\mathbf{Y}} - \hat{\mathbf{G}}^{(l - 1)} \Vert_{F}^{2} + 2 \sum_{k=l}^{m} \langle \mathbf{E}, \hat{\mathbf{G}}^{(k)} -\mathbf{G}^{(k)} \rangle  \\
    \leq& \tilde{\tau}d_{n}^{1/2}\xi_{E} + 2 \sum_{k=l}^{m} (\hat{\lambda}_{k} - \lambda_{k}) \langle \mathbf{E}, \mathbf{X}_{\hat{j}_{k}}\tilde{\mathbf{B}}_{\hat{j}_{k}} - \hat{\mathbf{G}}^{(k-1)} \rangle  \\
    \leq& \tilde{\tau}d_{n}^{1/2}\xi_{E} + \frac{8m\xi_{E}^{2}}{n(1 - \epsilon_{L})},
\end{align*}
which is the desired result.
\end{proof}

\begin{proof}[\textsc{Proof of \eqref{6.thm1.2} and \eqref{6.thm1.3}}]
    Note first that for any $D>0$, $(D+x)/(D-x) \leq 1 + 3x/D$ for all $0 \leq x \leq (1 - \sqrt{2/3})D$.
It is not difficult to see that 
\begin{align*}
    &\mathbb{P}\left\{ \frac{4L_{0}\xi_{E}}{nd_{n}^{1/2}} \leq (1 - \sqrt{\frac{2}{3}})\left((nd_{n})^{-1} \Vert \tilde{\mathbf{Y}} - \hat{\mathbf{G}}^{(k)} \Vert_{F}^{2} + (nd_{n})^{-1} \Vert \mathbf{E} \Vert_{F}^{2} \right), 1 \leq k \leq \hat{k}, \mathcal{G}_{n} \right\} \\
    \geq& \mathbb{P}\left\{ \frac{4L_{0}\xi_{E}}{nd_{n}^{1/2}} \leq (1 - \sqrt{\frac{2}{3}})M^{-1} \right\} - o(1) \\
    \rightarrow& 1.
\end{align*}
Thus, on $\mathcal{G}_{n}$ except for a vanishing event, 
\begin{align*}
    A_{k} \leq& 1 + \frac{12 L_{0} \xi_{E} / (nd_{n}^{1/2})}{(nd_{n})^{-1} \Vert \tilde{\mathbf{Y}} - \hat{\mathbf{G}}^{(k)} \Vert_{F}^{2} + (nd_{n})^{-1} \Vert \mathbf{E} \Vert_{F}^{2} } \\
    \leq& 1 + 12 M L_{0} \frac{\xi_{E}}{n d_{n}^{1/2}},
\end{align*}
for all $1 \leq k \leq \hat{k}$. 
This proves \eqref{6.thm1.2}. 
We now turn to \eqref{6.thm1.3}. 
Since for any positive $A$ and $B$, $A/(B+x) \geq A(1 - x/B)/B$ for all $x \geq 0$, it follows from \eqref{6.thm1.0} that on $\mathcal{G}_{n}$ except for a vanishing event,
\begin{align*}
    B_{k} \geq& \frac{\tau^{2}s_{n}^{-1}}{4L_{0}^{2}\mu^{2}} \frac{(nd_{n})^{-1}\Vert \tilde{\mathbf{Y}} - \hat{\mathbf{G}}^{(k-1)}\Vert_{F}^{2}}{(nd_{n})^{-1}\Vert \tilde{\mathbf{Y}} - \hat{\mathbf{G}}^{(k-1)}\Vert_{F}^{2} + (nd_{n})^{-1} \Vert \mathbf{E} \Vert_{F}^{2}} \\
    & \times \left( 1 - \frac{4L_{0}\xi_{E}/(nd_{n}^{1/2})}{(nd_{n})^{-1}\Vert \tilde{\mathbf{Y}} - \hat{\mathbf{G}}^{(k-1)}\Vert_{F}^{2} + (nd_{n})^{-1} \Vert \mathbf{E} \Vert_{F}^{2}} \right) \\
    \geq& \frac{\tau^{2}s_{n}^{-1}}{4L_{0}^{2}\mu^{2}} \frac{1}{1 + \mu M s_{n}} \left(1 - \frac{4ML_{0}\xi_{E}}{nd_{n}^{1/2}} \right)
\end{align*}
for $1 \leq k \leq \hat{k}$, which proves \eqref{6.thm1.3}.
\end{proof}

\begin{proof}[\textsc{Proof of \eqref{6.thm3.2}}]
Let 
\begin{align*}
    \mathbf{H} = \sum_{j \in \hat{J}}\mathbf{X}_{j}\hat{\mathbf{\Sigma}}_{j}^{-1}\mathbf{U}_{j}\mathbf{D}_{j}\mathbf{V}_{j}^{\top} \in \mathcal{B}.
\end{align*}
Note that Proposition \ref{aux-prop0} and (C3) imply
\begin{align*}
    \Vert \bar{\mathbf{Y}} - \mathbf{H} \Vert_{F}^{2} \geq& n\mu^{-1} \left\{ \sum_{j \in \hat{J}_{o}} \Vert \hat{\mathbf{\Sigma}}_{j}^{-1}\mathbf{U}_{j}(\mathbf{L}_{j}\mathbf{\Lambda}_{j}\mathbf{R}_{j}^{\top} - \mathbf{D}_{j})\mathbf{V}_{j}^{\top} \Vert_{F}^{2} + \sum_{j \in \hat{J}-\hat{J}_{o}} \Vert \hat{\mathbf{\Sigma}}_{j}^{-1}\mathbf{U}_{j}\mathbf{D}_{j}\mathbf{V}_{j}^{\top} - \mathbf{B}_{j}^{*} \Vert_{F}^{2}\right\} \\
    \geq& n\mu^{-3} \left\{ \sum_{j \in \hat{J}_{o}} \Vert \mathbf{L}_{j}\mathbf{\Lambda}_{j}\mathbf{R}_{j}^{\top} - \mathbf{D}_{j} \Vert_{F}^{2} + \sum_{j \in \hat{J} - \hat{J}_{o}} \Vert \mathbf{U}_{j}^{\top}\hat{\mathbf{\Sigma}}_{j}\mathbf{B}_{j}^{*}\mathbf{V}_{j} - \mathbf{D}_{j} \Vert_{F}^{2} \right\} \\
    \geq& \frac{n}{\mu^{3}\kappa_{n}}\left\{ \sum_{j \in \hat{J}_{o}} \Vert \mathbf{L}_{j}\mathbf{\Lambda}_{j}\mathbf{R}_{j}^{\top} - \mathbf{D}_{j} \Vert_{*} + \sum_{j \in \hat{J} - \hat{J}_{o}} \Vert \mathbf{U}_{j}^{\top}\hat{\mathbf{\Sigma}}_{j}\mathbf{B}_{j}^{*}\mathbf{V}_{j} - \mathbf{D}_{j} \Vert_{*} \right\}^{2}.
\end{align*}
Since $\mathbf{H} \in \mathcal{B}$, we have 
\begin{align*}
    \left\{ \sum_{j \in \hat{J}_{o}} \Vert \mathbf{L}_{j}\mathbf{\Lambda}_{j}\mathbf{R}_{j}^{\top} - \mathbf{D}_{j} \Vert_{*} + \sum_{j \in \hat{J} - \hat{J}_{o}} \Vert \mathbf{U}_{j}^{\top}\hat{\mathbf{\Sigma}}_{j}\mathbf{B}_{j}^{*}\mathbf{V}_{j} - \mathbf{D}_{j} \Vert_{*}  \right\}^{2} \leq \frac{9d_{n}L_{0}^{2}}{16} = \frac{9L_{n}^{2}}{16}.
\end{align*}
By the triangle inequality, we have $\sum_{j \in \hat{J}} \Vert \mathbf{D}_{j} \Vert_{*} \leq 3L_{n}/4 + \sum_{j \in \hat{J}_{o}} \Vert \mathbf{\Lambda}_{j} \Vert_{*} + \sum_{j \in \hat{J}-\hat{J}_{o}}\Vert \hat{\mathbf{\Sigma}}_{j}\mathbf{B}_{j}^{*} \Vert_{*}$.
Because of (C6), and $\hat{J}_{o} \subset J_{o}$ (with probability tending to one), $\sum_{j \in \hat{J}_{o}} \Vert \mathbf{\Lambda}_{j} \Vert_{*} + \sum_{j \in \hat{J}-\hat{J}_{o}}\Vert \hat{\mathbf{\Sigma}}_{j}\mathbf{B}_{j}^{*} \Vert_{*} \leq \sum_{j \in \hat{J}} \Vert \hat{\mathbf{\Sigma}}_{j}\mathbf{B}_{j}^{*} \Vert_{*} \leq  \mu (1 - \epsilon_{L}) L_{n} \leq 4^{-1} \mu^{-1} L_{n} \leq L_{n} / 4$. Hence $\sum_{j \in \hat{J}_{\hat{k}}}\Vert \mathbf{D}_{j} \Vert_{*} \leq L_{n}$, which proves $\mathbf{H} \in \mathcal{C}_{L}$.
\end{proof}

\begin{proposition} \label{aux-prop}
Let $\mathbf{A}^{*}$ be an $m \times n$ matrix and $\mathbf{A} = \mathbf{A}^{*} + \mathbf{E}$ be its perturbed version. Let $\mathbf{U}_{*}\mathbf{\Sigma}_{*}\mathbf{V}_{*}^{\top}$ and $\mathbf{U}\mathbf{\Sigma}\mathbf{V}^{\top}$ be their truncated SVD of rank $r_{*}$, respectively. If $\sigma_{r_{*}}(\mathbf{A}^{*}) := \sigma_{r_{*}} > \sigma_{r_{*}+1}(\mathbf{A}^{*}) = 0$, and if $\Vert \mathbf{E} \Vert_{op} < \sigma_{r_{*}}$, then
\begin{align*}
    \max\{\mathrm{dist}(\mathbf{U}_{*}, \mathbf{U}), \mathrm{dist}(\mathbf{V}_{*}, \mathbf{V}) \} \leq \frac{\sqrt{2}\max\{\Vert \mathbf{E}^{\top}\mathbf{U}_{*}\Vert_{op}, \Vert \mathbf{E}\mathbf{V}_{*} \Vert_{op}\}}{\sigma_{r_{*}} - \Vert \mathbf{E} \Vert_{op}},
\end{align*}
where $\mathrm{dist}(\mathbf{Q}, \mathbf{Q}_{*}) = \min_{\mathbf{R}}\Vert \mathbf{QR} - \mathbf{Q}_{*} \Vert_{op}$ for any two orthogonal matrices $\mathbf{Q}$, $\mathbf{Q}^{*}$ with $r$ columns, where the minimum is taken over all $r \times r$ orthonormal matrices.
\end{proposition}
\begin{remark} \normalfont
    Proposition \ref{aux-prop} is a consequence of the perturbation bounds for singular values \citep{wedin1972}. A proof can be found in \citet{Ma2021}.
\end{remark}

\begin{proof}[\textsc{Proof of \eqref{6.thm3.7}}]
    Note first that
\begin{align*}
    \bar{\mathbf{Y}} - \tilde{\mathbf{Y}} =& \sum_{j \in \hat{J}_{o}} \mathbf{X}_{j} \hat{\mathbf{\Sigma}}_{j}^{-1}(\mathbf{U}_{j}\mathbf{L}_{j} - \tilde{\mathbf{U}}_{j})\mathbf{\Lambda}_{j}\tilde{\mathbf{V}}_{j}^{\top} \\
    &+ \sum_{j \in \hat{J}_{o}} \mathbf{X}_{j} \hat{\mathbf{\Sigma}}_{j}^{-1}\mathbf{U}_{j}\mathbf{L}_{j}\mathbf{\Lambda}_{j}(\mathbf{V}_{j}\mathbf{R}_{j} - \tilde{\mathbf{V}}_{j})^{\top}.
\end{align*}
By triangle inequality,
\begin{align}
    \Vert \bar{\mathbf{Y}} - \tilde{\mathbf{Y}} \Vert_{F} \leq& \sqrt{n \mu}\left( \sum_{j \in \hat{J}_{o}}\Vert \mathbf{\Lambda}_{j} \Vert_{F} \right)\left\{ \max_{j \in \hat{J}_{o}} \Vert \mathbf{U}_{j}\mathbf{L}_{j} - \tilde{\mathbf{U}}_{j} \Vert_{op} +  \max_{j \in \hat{J}_{o}} \Vert \mathbf{V}_{j}\mathbf{R}_{j} - \tilde{\mathbf{V}}_{j} \Vert_{op}\right\}. 
    \label{pert1}
\end{align}
Let $\mathbf{U}_{j, \bar{r}_{j}}$ and $\mathbf{V}_{j, \bar{r}_{j}}$ be sub-matrices of $\mathbf{U}_{j}$ and $\mathbf{V}_{j}$ consisting of column vectors that correspond to the leading $\bar{r}_{j}$ singular vectors. Write $\mathbf{U}_{j} = (\mathbf{U}_{j,\bar{r}_{j}},  \mathbf{U}_{j, -\bar{r}_{j}})$ and $\mathbf{V}_{j} = (\mathbf{V}_{j,\bar{r}_{j}},  \mathbf{V}_{j, -\bar{r}_{j}})$.
Since $\mathbf{X}_{j}^{\top}\tilde{\mathbf{Y}} = \mathbf{X}_{j}^{\top}\mathbf{Y} - \mathbf{X}_{j}^{\top}\mathbf{E}$, it follows from Proposition \ref{aux-prop} and (C5) that there exist $\bar{r}_{j} \times \bar{r}_{j}$ orthonormal matrices $\tilde{\mathbf{L}}_{j}$ and $\tilde{\mathbf{R}}_{j}$ such that with probability tending to one,
\begin{align}
    \max\left\{\Vert \mathbf{U}_{j,\bar{r}_{j}}\tilde{\mathbf{L}}_{j} - \tilde{\mathbf{U}}_{j} \Vert_{op}, \Vert \mathbf{V}_{j,\bar{r}_{j}}\tilde{\mathbf{R}}_{j} - \tilde{\mathbf{V}}_{j} \Vert_{op} \right\} \leq& \frac{\sqrt{2}\max\{\Vert \mathbf{E}^{\top}\mathbf{X}_{j}\tilde{\mathbf{U}}_{j} \Vert_{op}, \Vert \mathbf{X}_{j}^{\top}
    \mathbf{E} \tilde{\mathbf{V}}_{j}\Vert_{op} \}}{n\delta_{n} - \Vert \mathbf{X}_{j}^{\top} \mathbf{E} \Vert_{op}} \notag \\
    \leq& \frac{\sqrt{2}\xi_{E}}{n\delta_{n} - \xi_{E}}. \notag 
\end{align}
Set $\mathbf{L}_{j}^{\top} = (\tilde{\mathbf{L}}_{j}^{\top}, \mathbf{0}_{\bar{r}_{j}\times(\hat{r}-\bar{r}_{j})})$ and $\mathbf{R}_{j}^{\top} = (\tilde{\mathbf{R}}_{j}^{\top}, \mathbf{0}_{\bar{r}_{j}\times(\hat{r}-\bar{r}_{j})})$ for $j \in \hat{J}_{o}$ in \eqref{pert1}. Then by (C4) and (C6), it follows that 
\begin{align*}
    \Vert \bar{\mathbf{Y}} - \tilde{\mathbf{Y}} \Vert_{F}^{2} \leq n \mu \left( \sum_{j \in \hat{J}_{o}}\Vert \mathbf{\Lambda}_{j} \Vert_{F} \right)^{2} \left( \frac{2\sqrt{2}\xi_{E}}{n\delta_{n} - \xi_{E}} \right)^{2} 
    \leq 8\mu L^{2} n d_{n} \frac{\xi_{E}^{2}}{(n\delta_{n} - \xi_{E})^{2}}.
\end{align*}
\end{proof}

\begin{proof}[Proof of Corollary \ref{hor.cor}]
    By Lemma \ref{3.lem1}, $\sharp(\hat{J}) + \hat{r} = O_{p}(s_{n}^{2})$.
    Thus running the first-stage RGA with the just-in-time stopping criterion costs
    \begin{align} \label{hor.cor.1}
        O_{p}(s_{n}^{2}(n_{1}+d_{n}))
    \end{align}
    bytes of communication per computing node.
    In addition, preparing $\{\hat{\mathbf{\Sigma}}_{j}^{-1}: j\in \hat{J}\}$ and $(\mathbf{U}_{j}, \mathbf{V}_{j})$ for $j \in \hat{J}$ with $q_{n,j} \wedge d_{n} > \hat{r}$ costs
    \begin{align} \label{hor.cor.2}
        &O_{p}\left( \sum_{j \in \hat{J}} \{q_{n,j}^{2} + (q_{n,j}d_{n} + \hat{r}(q_{n,j} + d_{n}))\mathbf{1}\{q_{n,j} \wedge d_{n} > \hat{r}\}\} \right) \notag \\
        =& O_{p}(n_{1}^{2\alpha}s_{n}^{2} + n_{1}^{\alpha}d_{n}s_{n}^{2} + s_{n}^{4}(n_{1}^{\alpha}+d_{n})).
    \end{align}
    Since the communication costs per node at the $k$-th iteration of the second-stage RGA is at most
    \begin{align*}
        &O_{p}\left( \sum_{j \in \hat{J}}\left(\hat{r}^{2}\mathbf{1}\{q_{n,j}\wedge d_{n}>\hat{r}\} + q_{n,j}d_{n}\mathbf{1}\{q_{n,j}\wedge d_{n}\leq\hat{r}\} \right) + d_{n}k + n_{1} \right) \\
        =&O_{p}\left( s_{n}^{6} + n_{1}^{\alpha}d_{n}s_{n}^{2} + d_{n}k + n_{1} \right),
    \end{align*}
    running $m_{n} = O_{p}(s_{n}^{4}\log(n^{2}d_{n}/\xi_{n}^{2}))$ iterations (see Theorem \ref{3.thm2} for the definition of $m_{n}$) costs
    \begin{align} \label{hor.cor.3}
        O_{p}\left( (s_{n}^{6} + s_{n}^{2}n_{1}^{\alpha}d_{n} + n_{1}) s_{n}^{4} \log \frac{n^{2}d_{n}}{\xi_{n}^{2}} + d_{n}s_{n}^{8}\left( \log \frac{n^{2}d_{n}}{\xi_{n}^{2}} \right)^{2} \right).
    \end{align}
    Combining \eqref{hor.cor.1}-\eqref{hor.cor.3} yields the desired result.
\end{proof}

\section{TSRGA for high-dimensional generalized linear models} \label{App:GLM}

In this section, we apply the idea of TSRGA to and propose a modified algorithm for estimating the generalized linear model (GLM).
Focusing on the case of a scalar response $y_{t}$, the GLM postulates that the probability density function $f$ of $y_{t}$ (or the probability mass function if $y_{t}$ is discrete) belongs to the exponential family. In particular, 
\begin{align*}
    f(y; \theta) = \exp[y\theta - r(\theta) + h(y)],
\end{align*}
and
\begin{align*}
    \mathbb{E}(y_{t}|{x}_{t,1}, \ldots, x_{t,p_{n}}) = r'\left( \sum_{j=1}^{p_{n}}\beta_{j}^{*}x_{t,j} \right)
\end{align*}
where $\theta$ is called the natural parameter; $r$, $h$ are known functions, and $r'$ is the derivative of $r$, which is also known as the inverse of the link function (see, e.g., \citealp{dunn2018, Han2023}).
To maximize the log-likelihood function, scaled as $y\theta - r(\theta)$, one can minimize the following loss function
\begin{align*}
    \mathcal{L}_{n}(\mathbf{X}\bm{\beta}) = \frac{1}{n} \sum_{t=1}^{n} \left[ - y_{t}\left( \sum_{j=1}^{p_{n}}\beta_{j}x_{t,j} \right) + r\left( \sum_{j=1}^{p_{n}}\beta_{j}x_{t,j} \right) \right],
\end{align*}
where $\mathcal{L}_{n}(\bm{\tau}) = n^{-1}\sum_{t=1}^{n}(y_{t}\tau_{t} - r(\tau_{t}))$ for $\bm{\tau} = (\tau_{1}, \ldots, \tau_{n})^{\top}$.

Interpreting $y_{t} - r'(\sum_{j=1}^{p_{n}}\beta_{j}x_{t,j})$ as the residual, we can implement RGA as follows.
First initialize $\hat{\mathbf{G}}^{(0)} = 0$. Then for $k = 1, 2, \ldots, K_{n}$,
find
\begin{align} \label{appD1}
    \hat{j}_{k} \in \arg\max_{1 \leq j \leq p_{n}} \left|\frac{1}{n} \sum_{t=1}^{n}\left( y_{t} - r'(\hat{G}^{(k-1)}_{t}) \right) {x}_{t,j}\right|
\end{align}
and update
\begin{align} \label{appD2}
    \hat{\mathbf{G}}^{(k)} = (1-\hat{\lambda}_{k}) \hat{\mathbf{G}}^{(k-1)} + \hat{\lambda}_{k} L s_{k} \mathbf{z}_{\hat{j}_{k}},
\end{align}
where $\hat{\mathbf{G}}^{(k)} = (\hat{G}_{1}^{(k)}, \ldots, \hat{G}_{n}^{(k)})^{\top}$, $L>0$ is given, $\mathbf{z}_{j} = (x_{1,j}, \ldots, x_{n,j})^{\top}$, 
\begin{align*}
    s_{k} = \mathrm{sgn}\left( \frac{1}{n} \sum_{t=1}^{n}\left( y_{t} - r'(\hat{G}^{(k-1)}_{t}) \right) {x}_{t,\hat{j}_{k}}\right),
\end{align*}
and $\hat{\lambda}_{k}$ is determined by  
\begin{align*}
    \hat{\lambda}_{k} = \arg\min_{\lambda \in [0,1]} \mathcal{L}_{n}((1-\lambda) \hat{\mathbf{G}}^{(k-1)} + \lambda L s_{k} \mathbf{z}_{\hat{j}_{k}}).
\end{align*}
It is not difficult to see that \eqref{appD1} can be easily solved for feature-distributed data and constructing $\hat{\mathbf{G}}^{(k)}$ in each node requires a communication cost of $O(n)$ bytes. 
The second-stage RGA can be implemented similarly with the set of predictors considered in \eqref{appD1} restricted to $\hat{J}$, the set of predictors chosen by the first-stage when the just-in-time criterion is met. 
Finally, since $\mathcal{L}_{n}$ could take negative values, we modify the just-in-time criterion \eqref{jit} as
\begin{align} \label{appD3}
    \hat{k} = \min\left\{1 \leq k \leq K_{n}: \left\vert \frac{\mathcal{L}_{n}(\hat{\mathbf{G}}^{(k)})}{\mathcal{L}_{n}(\hat{\mathbf{G}}^{(k-1)})} - 1 \right\vert < t_{n} \right\}.
\end{align}
In the same spirit as \eqref{jit}, \eqref{appD3} terminates the first-stage RGA as soon as the improvement in the loss function is below certain threshold, which would save some communication costs and speed up the algorithm. 

Next, we examine the performance of this version of TSRGA (\eqref{appD1}-\eqref{appD3}) using simulations.
In the following experiments, the predictors $x_{t,j}$ are generated as in Specification 2.
We consider the following two specifications. 

\begin{spec} \normalfont (Logit model) The response $y_{t}$ takes only values in $\{0, 1\}$ and is generated via
\begin{align*}
    \mathbb{P}(y_{t} = y, \theta_{t}) = \theta_{t}^{y}(1-\theta_{t})^{1 - y}, \quad
    \theta_{t} = \frac{1}{1 + \exp(-\sum_{j=1}^{p_{n}} \beta_{j}^{*}x_{t,j})}
\end{align*}
where $(\beta_{1}^{*}, \beta_{2}^{*}, \beta_{3}^{*}, \beta_{4}^{*}, \beta_{5}^{*}) = (-2.4, 1.8, -1.9, 2.8, -2.2)$, $\beta_{j}^{*} = 0$ for $j > 5$.
For this model, we have $r(\theta) = \log(1 + \exp(\theta))$.
\end{spec}

\begin{spec} \normalfont (Poisson model) The response $y_{t}$ takes values in $\{0,1,2,\ldots\}$ and is generated via
\begin{align*}
    \mathbb{P}(y_{t} = y, \theta_{t}) = \frac{\theta_{t}^{y}e^{-\theta_{t}}}{y!}, \quad
    \theta_{t} = \exp(\sum_{j=1}^{p_{n}} \beta_{j}^{*}x_{t,j})
\end{align*}
and $(\beta_{1}^{*}, \beta_{2}^{*}, \beta_{3}^{*}, \beta_{4}^{*}, \beta_{5}^{*}) = (0.15, -0.25, 0.35, -0.45, 0.55)$, $\beta_{j}^{*} = 0$ for $j > 5$.
For this model, we have $r(\theta) = \exp(\theta)$.
\end{spec}

As a benchmark, we compare with the $\ell_{1}$-regularized GLM which solves
\begin{align} \label{appD4}
    \min_{\bm{\beta}} \mathcal{L}_{n}(\mathbf{X}\bm{\beta}) + \lambda \Vert \bm{\beta} \Vert_{1}
\end{align}
with $\lambda$ selected by 5-fold cross validation. 
Table \ref{table::GLM} reports the parameter estimation error $\Vert \hat{\bm{\beta}} - \bm{\beta}^{*} \Vert_{2}$, the number of irrelevant variables selected (false positives, FP) and the number of relevant variables not selected (false negatives, FN).
For the logit model, we additionally report the out-of-sample prediction accuracy on a test set of size 500.
For the Poisson model, the out-of-sample prediction error is measured by RMSE.
All these figures are averages over 500 independent simulations.

\begin{table}[]
\centering
\begin{tabular}{@{}lrrrr@{}}\toprule
Logit & \multicolumn{2}{l}{$n=800, p=1200$} & \multicolumn{2}{l}{$n = 1200, p = 2000$} \\ \cmidrule(lr){2-3} \cmidrule(lr){4-5} 
 & TSRGA           & $\ell_{1}$-GLM  & TSRGA              & $\ell_{1}$-GLM    \\
$\Vert \hat{\bm{\beta}} - \bm{\beta}^{*} \Vert_{2}$ & 0.698 & 2.185 & 0.689 & 2.036 \\
FP & 0.018 & 82.808 & 0 & 105.070 \\ 
FN & 0 & 0 & 0 & 0 \\ 
Accuracy & 0.901 & 0.888 & 0.901 & 0.892 \\ \midrule
Poisson &                 &                 &                    &                   \\
$\Vert \hat{\bm{\beta}} - \bm{\beta}^{*} \Vert_{2}$ & 0.135 & 0.190 & 0.060 & 0.144 \\
FP & 6.638 & 25.470 & 1.830 & 25.146 \\ 
FN & 0.114 & 0.008 & 0.020 & 0 \\ 
RMSE & 1.324 & 1.363 & 1.283 & 1.329 \\ \bottomrule
\end{tabular}
\caption{Simulation results for estimating high-dimensional GLMs. $\ell_{1}$-GLM is defined in \eqref{appD4}. The results are based on 500 simulations.}
\label{table::GLM}
\end{table}

The results show that for both the logit and Poisson models, TSRGA yields parsimonious and accurate coefficient estimates, with comparable out-of-sample prediction accuracy to the $\ell_{1}$-GLM defined by \eqref{appD4}. 
In particular, the low FP and FN values of TSRGA may be due to its variable selection properties.
Though we expect the general conclusions about TSRGA in this paper, such as the sure-screening property, to hold under the GLM framework, the rigorous mathematical treatment is left for future work. 

\section{Complementary simulation results} \label{App_simluation}

In this section, we present some additional simulation results regarding Specifications 1 and 2.
Figures \ref{fig:spec1_time} and \ref{fig:spec2_time} plot the parameter estimation error, as in Figures \ref{fig:spec1_b_error} and \ref{fig:spec2_b_error}, against the elapsed time. 
Clearly, TSRGA converges within the least amount of time. 
In particular, its second-stage only takes a very short amount of time, thanks to the dimension reduction after the just-in-time stopping criterion. 
Other methods behave similarly as those in Figures \ref{fig:spec1_b_error} and \ref{fig:spec2_b_error}, as their implementation cost scales directly with the number of iterations. 

Figures \ref{fig:spec1_RMSE} and \ref{fig:spec2_RMSE} plot the out-of-sample prediction error (measures by the root mean square prediction error on an independent test sample) of the methods under Specifications 1 and 2. For Specification 1, the final prediction accuracy of TSRGA, cross-validated Lasso, and Hydra are similar. 
However, for Specification 2, TSRGA clearly is the most desirable prediction tool among the methods under consideration. 

\begin{figure}[h!]
    \centering
    \begin{subfigure}[t]{0.45\textwidth}
        \centering
        \includegraphics[width=\linewidth, height=50mm]{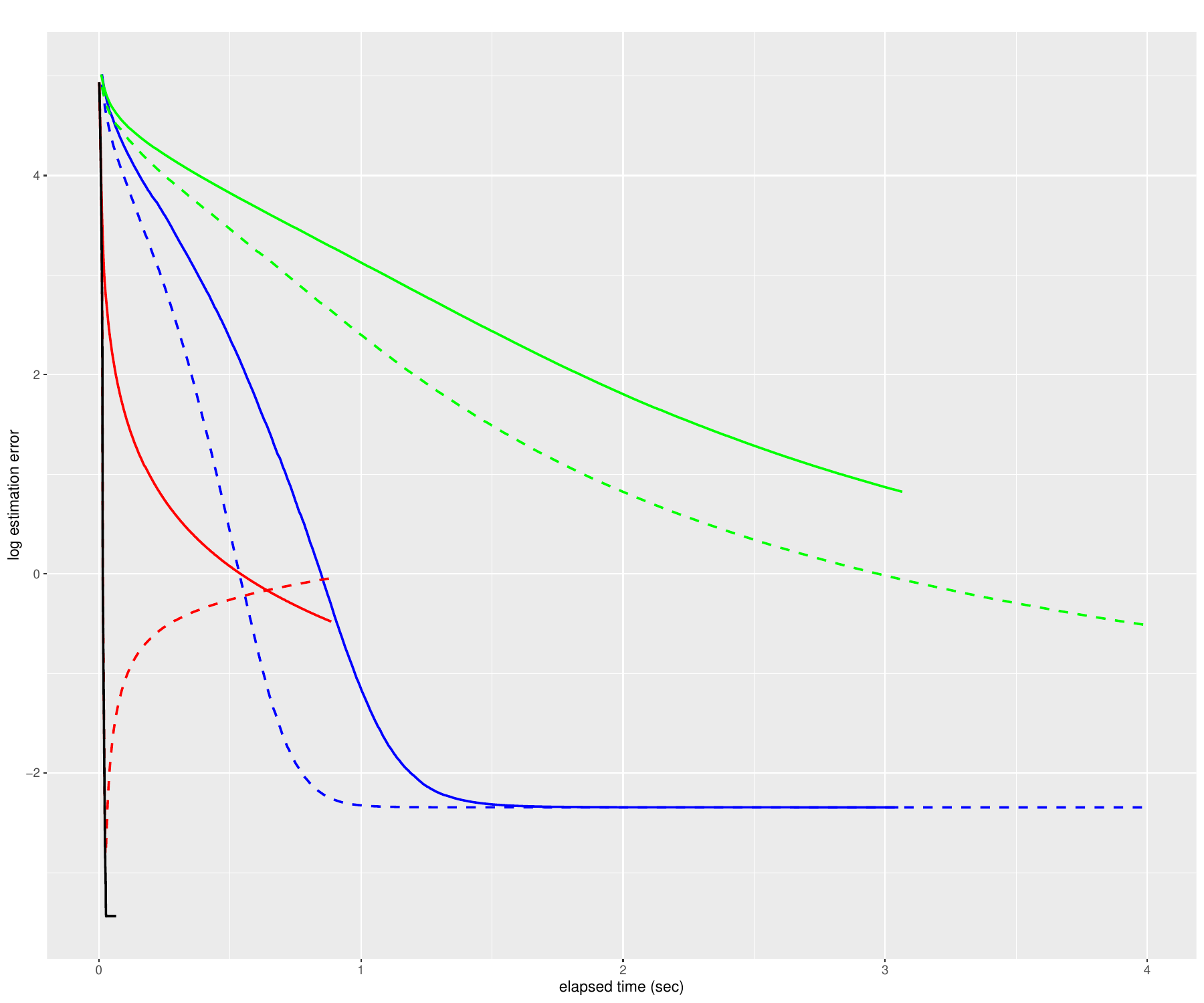}
        \caption{$n=800$, $p_{n}=1200$}    
    \end{subfigure}
    \hfill
    \begin{subfigure}[t]{0.45\textwidth}
        \centering
        \includegraphics[width=\linewidth, height=50mm]{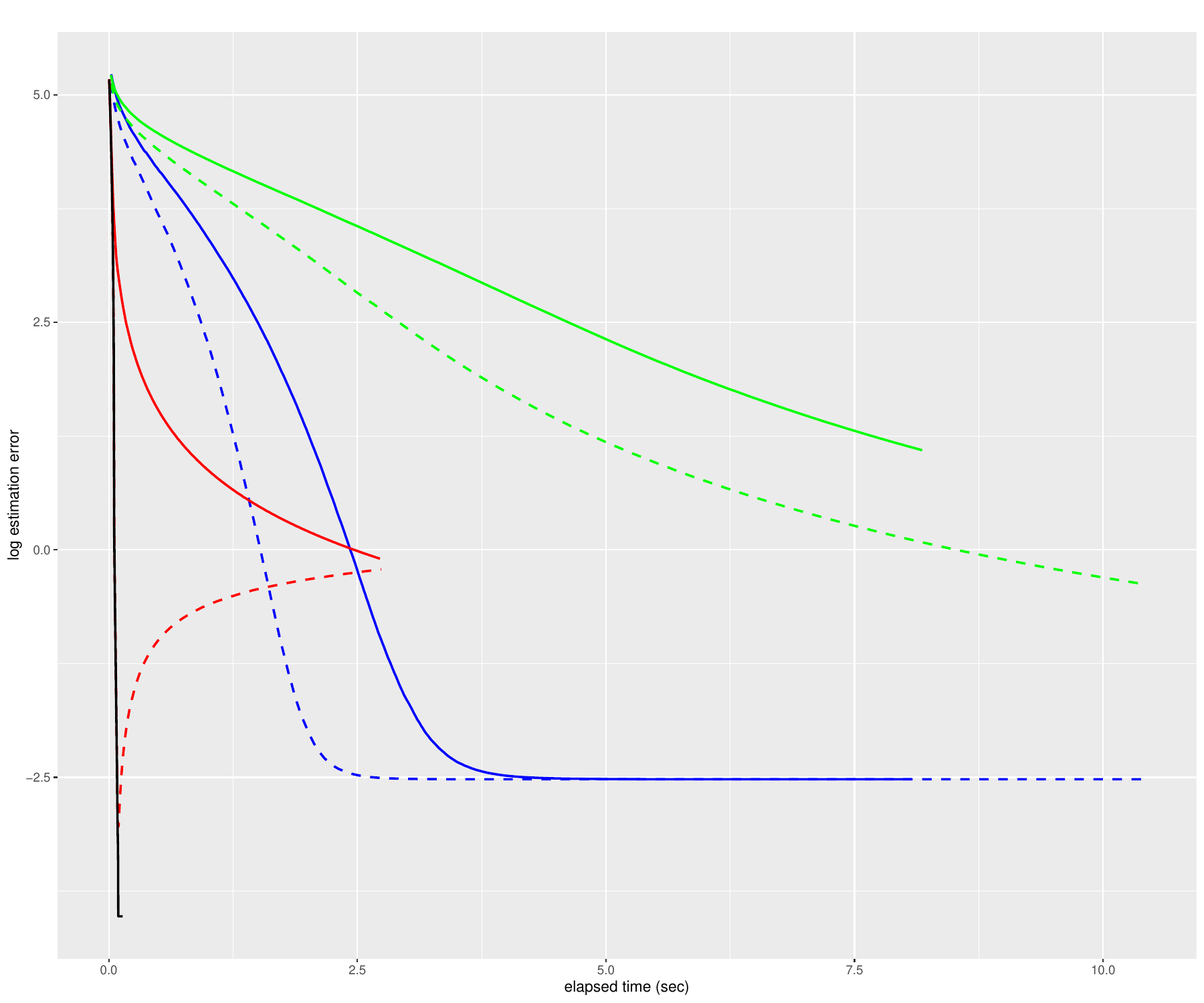}
        \caption{$n=1200$, $p_{n}=2000$}    
    \end{subfigure}
    \hfill
    \begin{subfigure}[t]{\textwidth}
        \centering
        \includegraphics[width=0.8\textwidth, height=85mm]{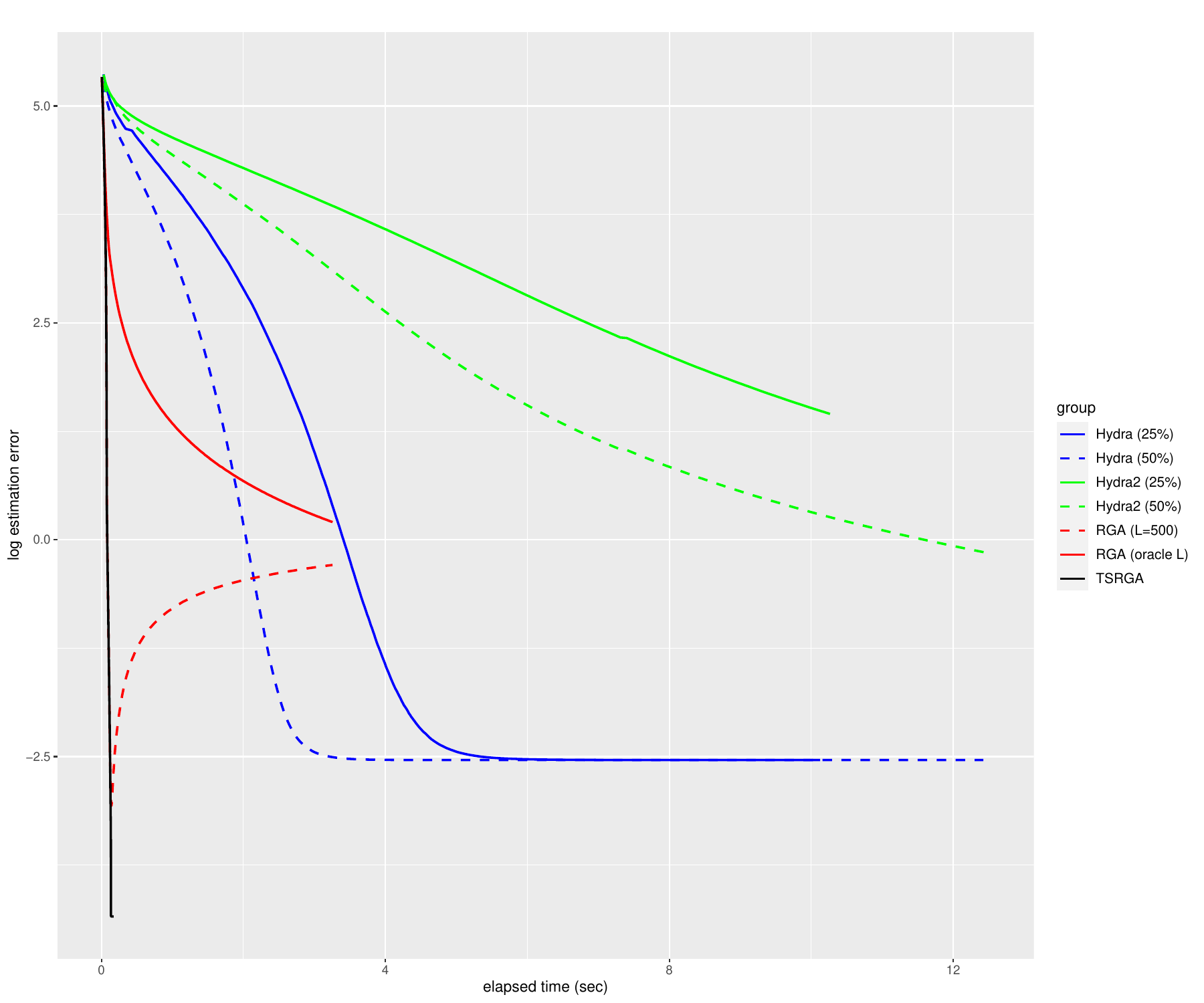}
        \caption{$n=1500$, $p_{n}=3000$}    
    \end{subfigure}
    \caption{Logarithm of parameter estimation errors of various methods against the elapsed time under Specification 1, where $n$ is the sample size 
    and $p_n$ is the dimension of predictors. The results are based on 100 simulations.}
    \label{fig:spec1_time}
\end{figure}

\begin{figure}[h!]
    \centering
    \begin{subfigure}[t]{0.45\textwidth}
        \centering
        \includegraphics[width=\linewidth, height=50mm]{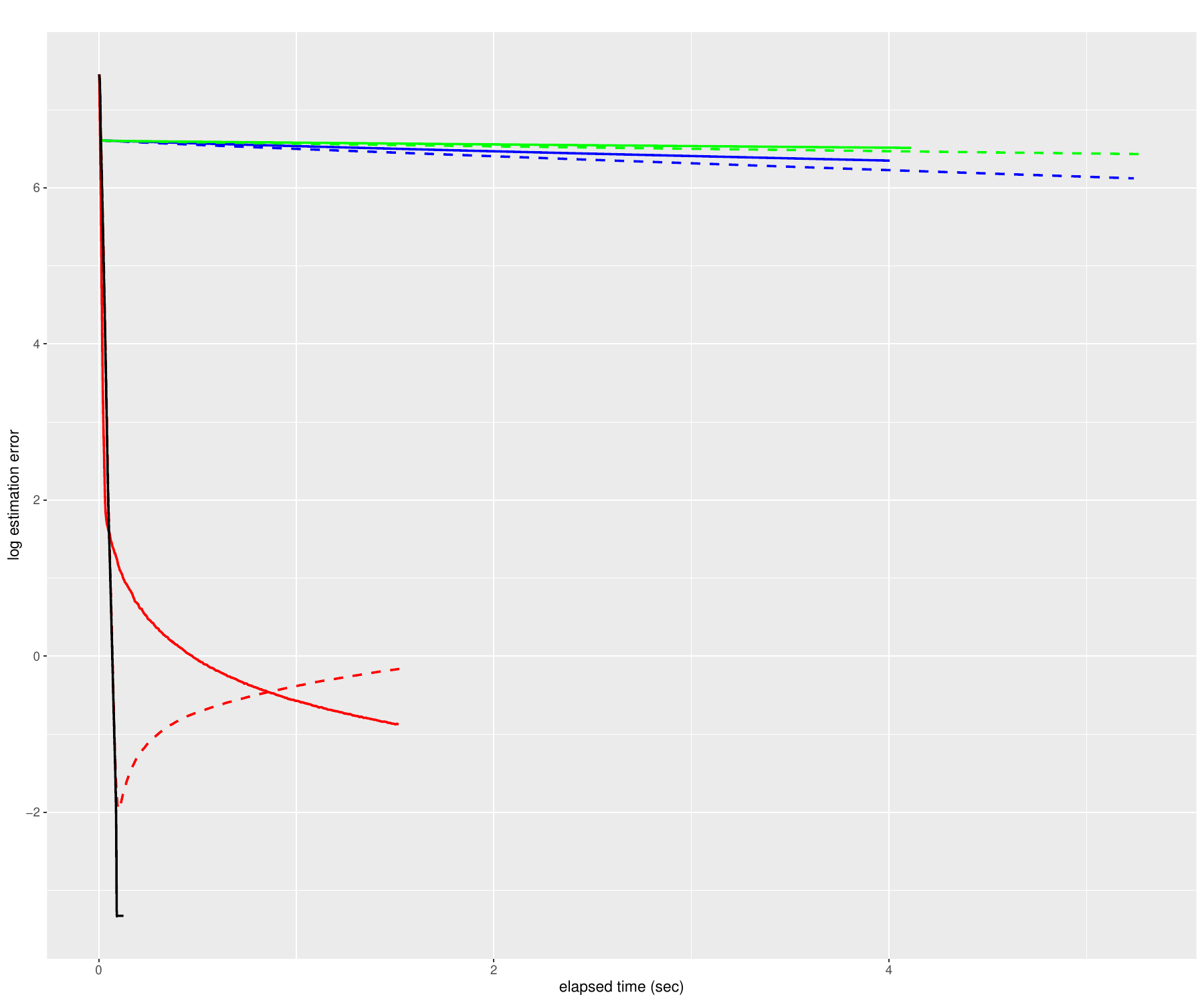}
        \caption{$n=800$, $p_{n}=1200$}    
    \end{subfigure}
    \hfill
    \begin{subfigure}[t]{0.45\textwidth}
        \centering
        \includegraphics[width=\linewidth, height=50mm]{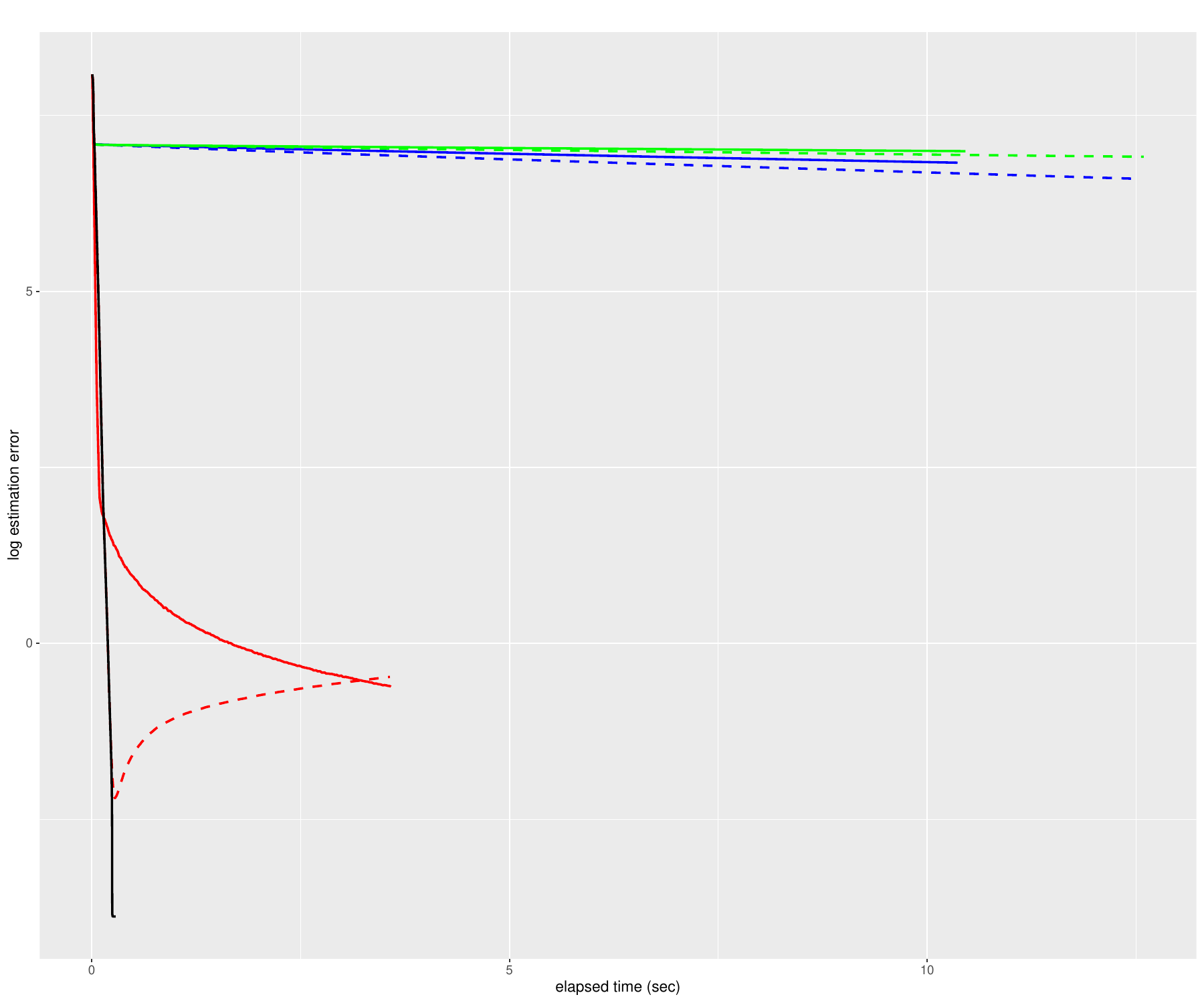}
        \caption{$n=1200$, $p_{n}=2000$}    
    \end{subfigure}
    \hfill
    \begin{subfigure}[t]{\textwidth}
        \centering
        \includegraphics[width=0.8\textwidth, height=85mm]{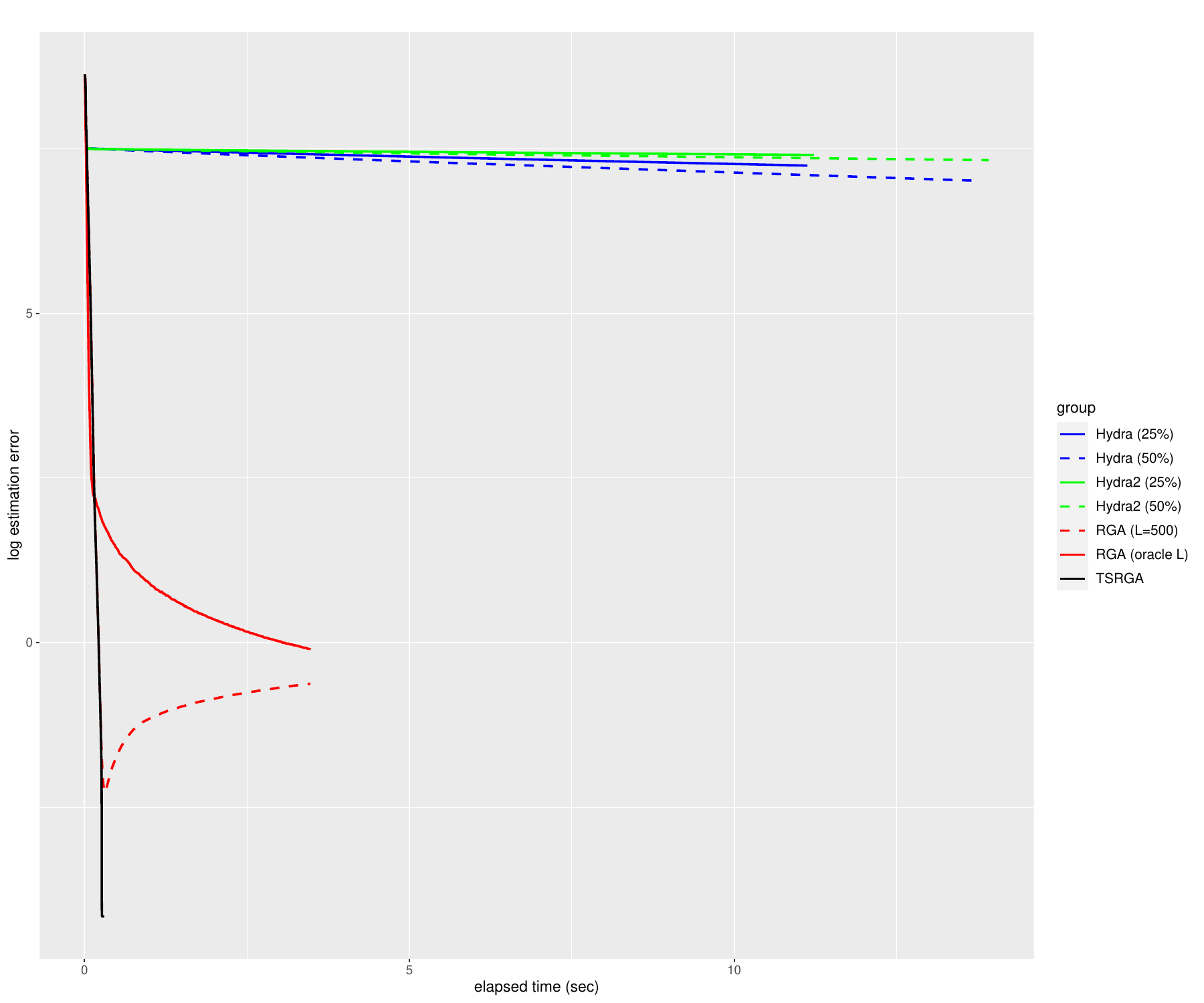}
        \caption{$n=1500$, $p_{n}=3000$}    
    \end{subfigure}
    \caption{Logarithm of parameter estimation errors of various methods against the elapsed time under Specification 2, where $n$ is the sample size 
    and $p_n$ is the dimension of predictors. The results are based on 100 simulations.}
    \label{fig:spec2_time}
\end{figure}

\begin{figure}[h!]
    \centering
    \begin{subfigure}[t]{0.45\textwidth}
        \centering
        \includegraphics[width=\linewidth, height=50mm]{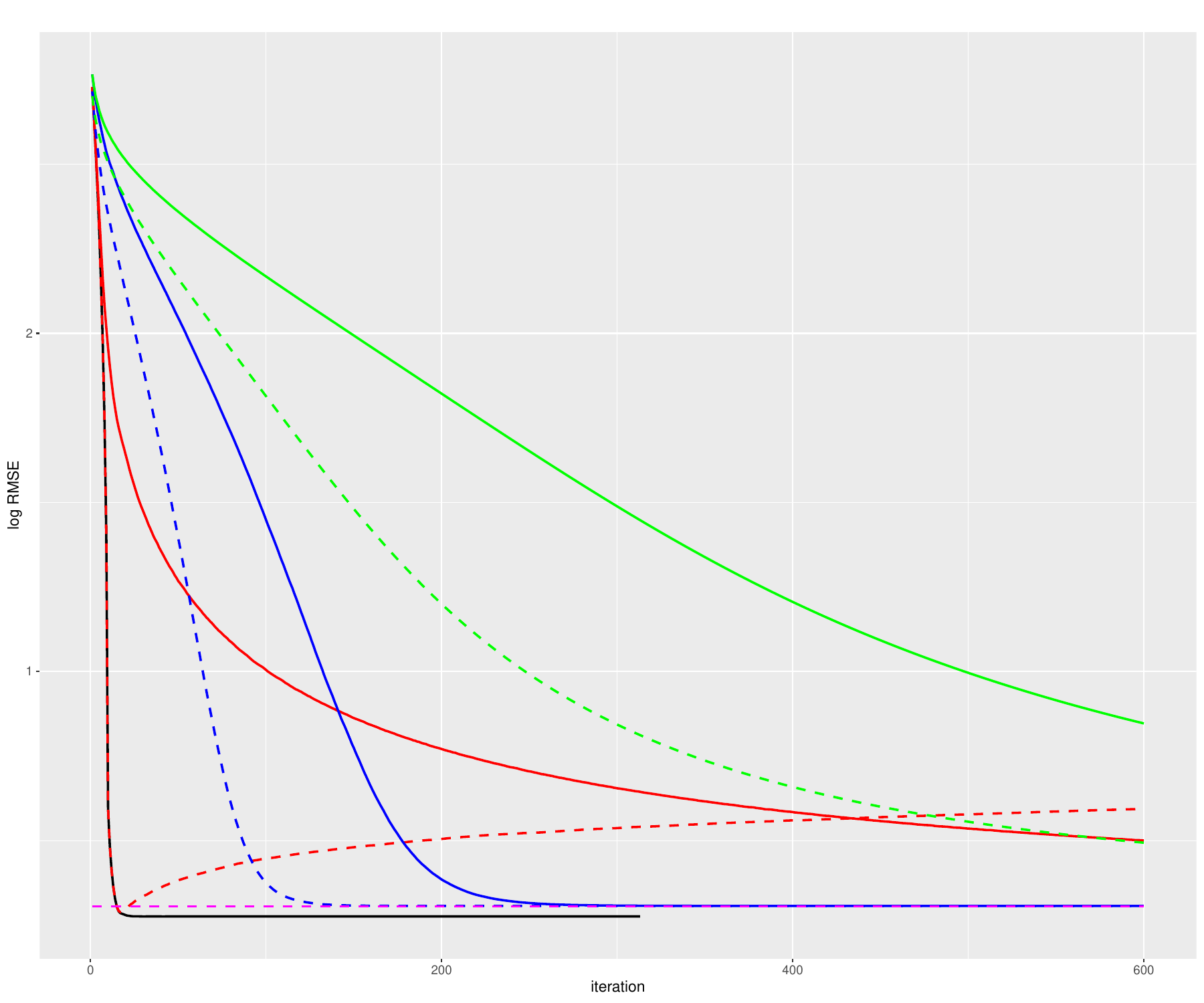}
        \caption{$n=800$, $p_{n}=1200$}    
    \end{subfigure}
    \hfill
    \begin{subfigure}[t]{0.45\textwidth}
        \centering
        \includegraphics[width=\linewidth, height=50mm]{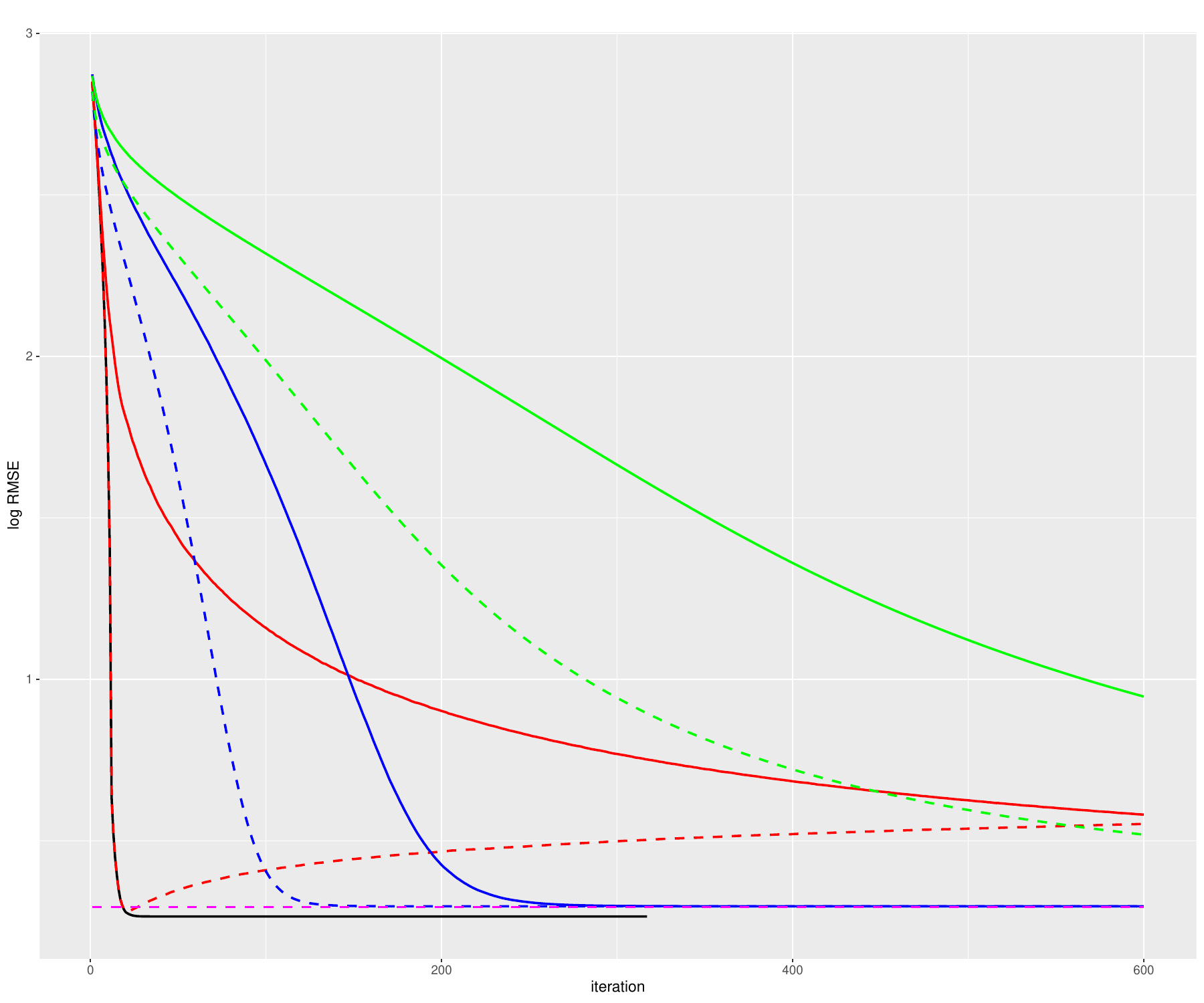}
        \caption{$n=1200$, $p_{n}=2000$}    
    \end{subfigure}
    \hfill
    \begin{subfigure}[t]{\textwidth}
        \centering
        \includegraphics[width=0.8\textwidth, height=85mm]{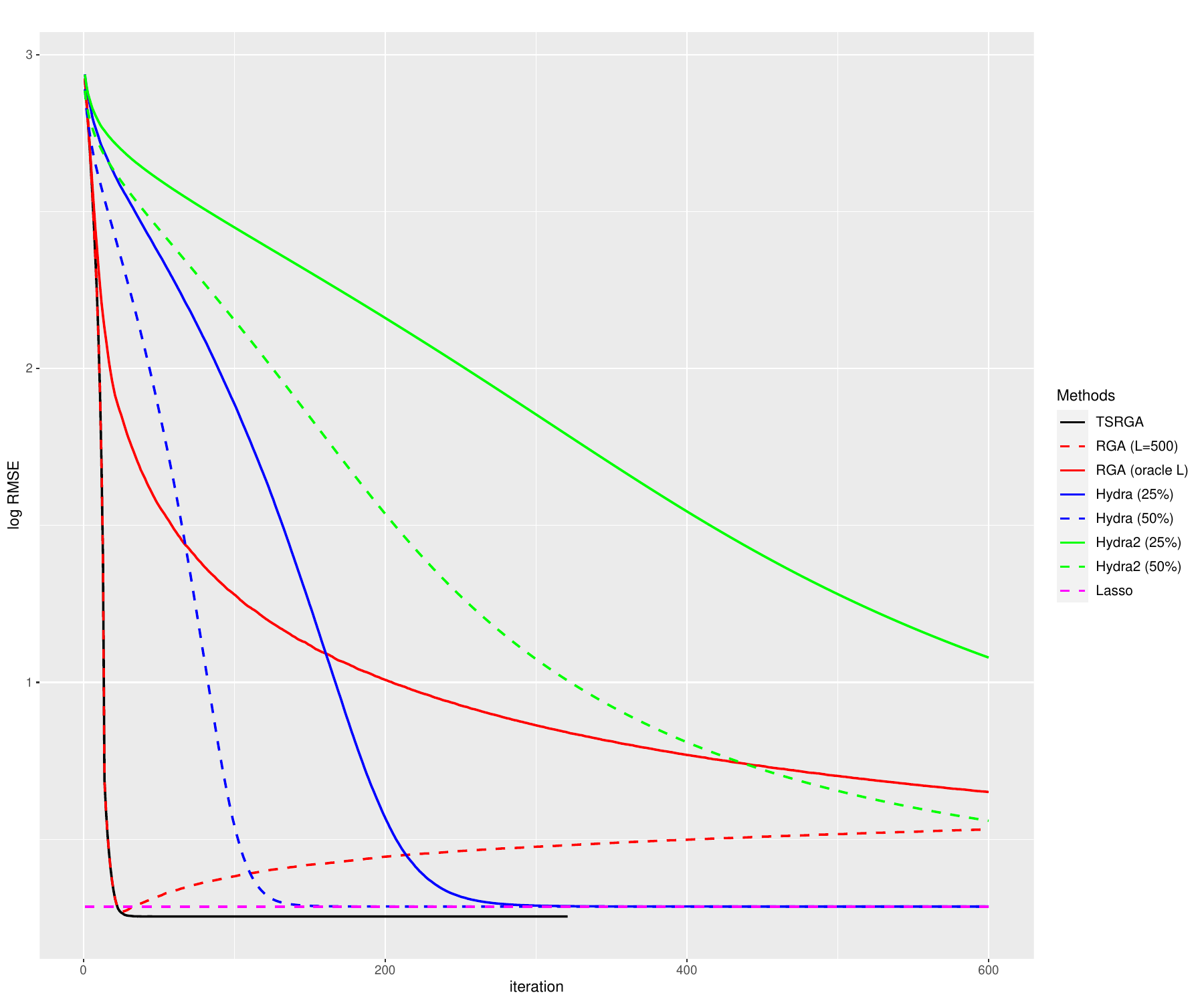}
        \caption{$n=1500$, $p_{n}=3000$}    
    \end{subfigure}
    \caption{Logarithm of out-of-sample prediction errors of various methods under Specification 1, where $n$ is the sample size 
    and $p_n$ is the dimension of predictors. The results are based on 100 simulations.}
    \label{fig:spec1_RMSE}
\end{figure}

\begin{figure}[h!]
    \centering
    \begin{subfigure}[t]{0.45\textwidth}
        \centering
        \includegraphics[width=\linewidth, height=50mm]{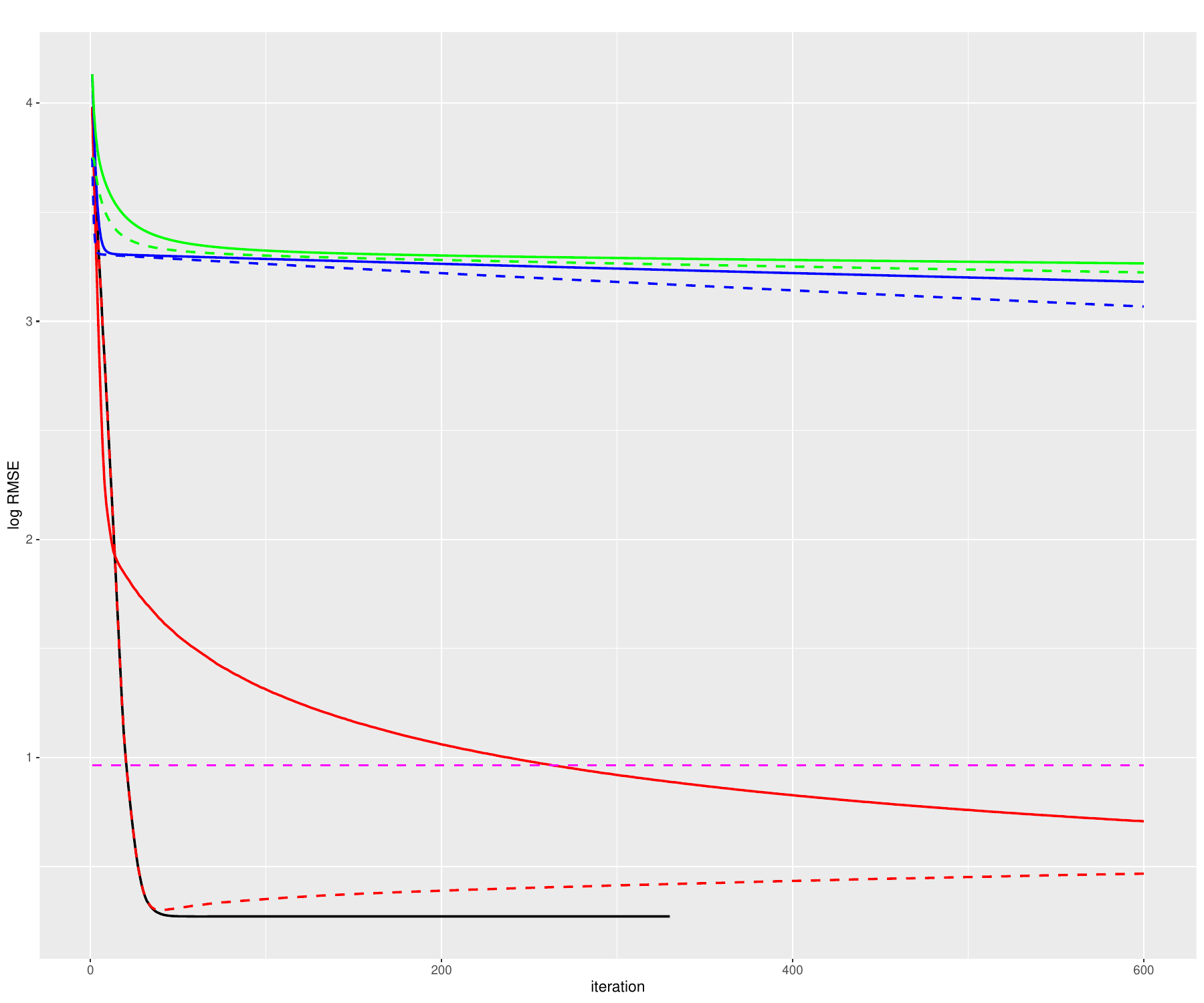}
        \caption{$n=800$, $p_{n}=1200$}    
    \end{subfigure}
    \hfill
    \begin{subfigure}[t]{0.45\textwidth}
        \centering
        \includegraphics[width=\linewidth, height=50mm]{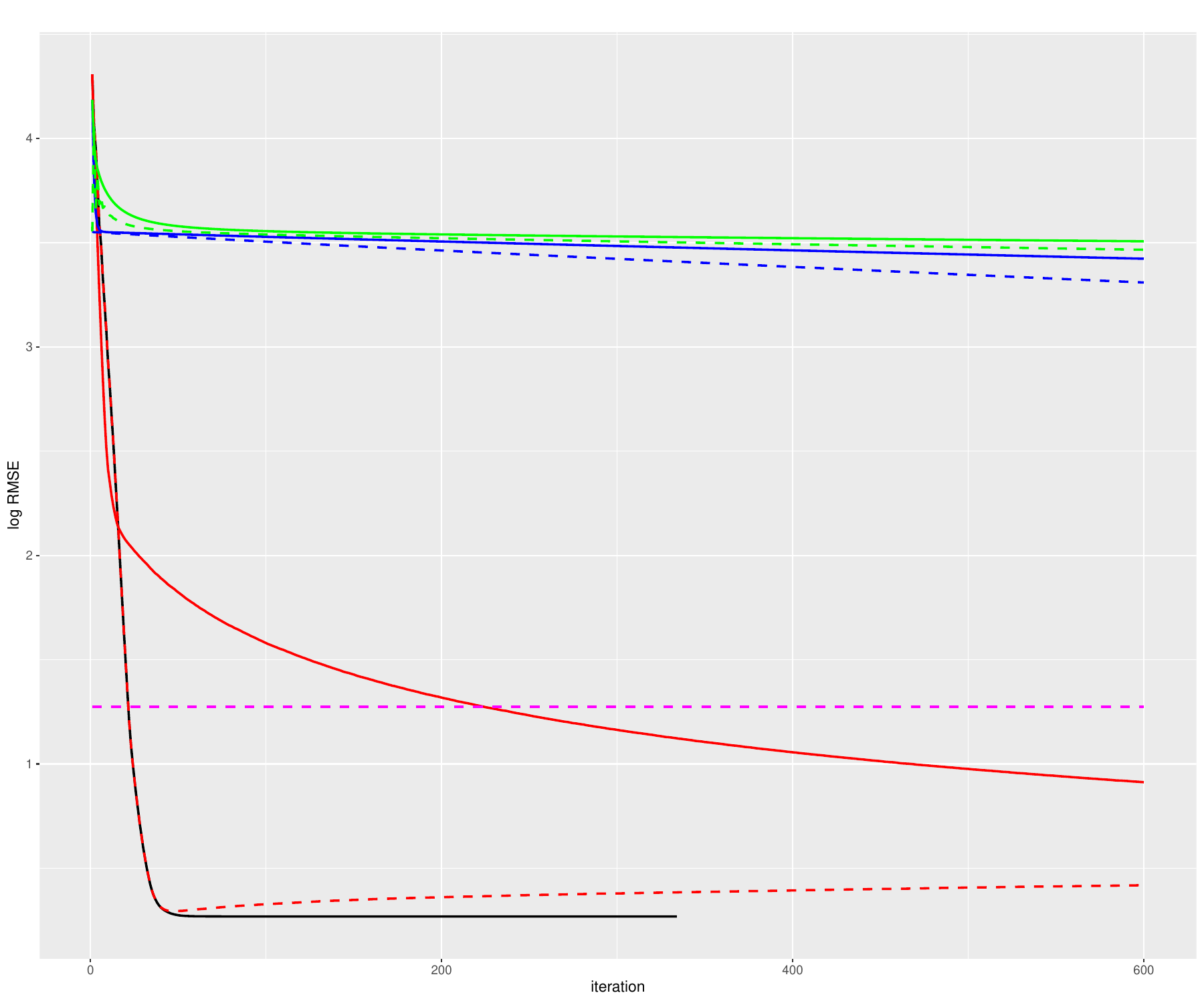}
        \caption{$n=1200$, $p_{n}=2000$}    
    \end{subfigure}
    \hfill
    \begin{subfigure}[t]{\textwidth}
        \centering
        \includegraphics[width=0.8\textwidth, height=85mm]{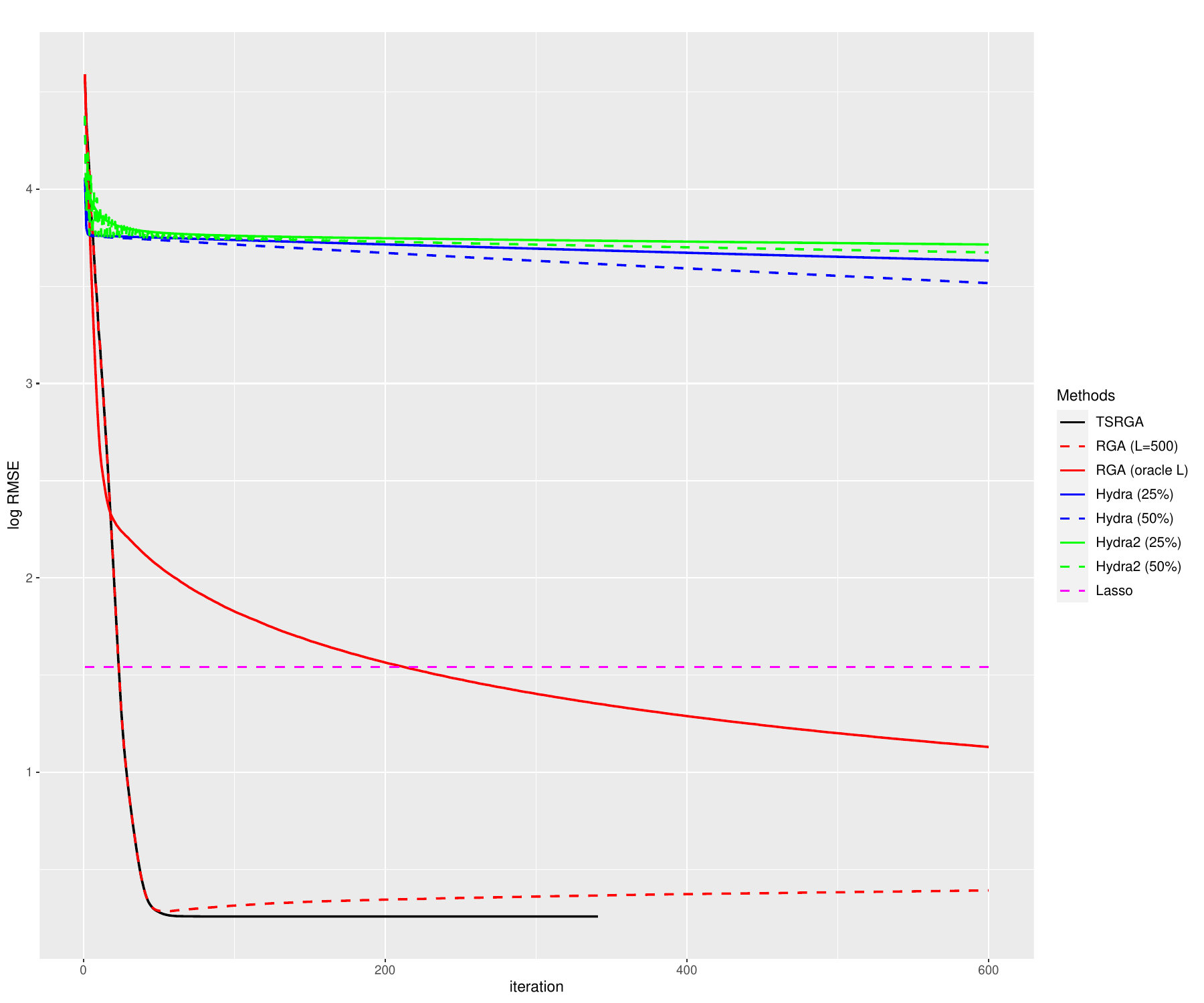}
        \caption{$n=1500$, $p_{n}=3000$}    
    \end{subfigure}
    \caption{Logarithm of out-of-sample prediction errors of various methods under Specification 2, where $n$ is the sample size 
    and $p_n$ is the dimension of predictors. The results are based on 100 simulations.}
    \label{fig:spec2_RMSE}
\end{figure}

\clearpage
\bibliography{sample}

\begin{thebibliography}{51}
\providecommand{\natexlab}[1]{#1}
\providecommand{\url}[1]{\texttt{#1}}
\expandafter\ifx\csname urlstyle\endcsname\relax
  \providecommand{\doi}[1]{doi: #1}\else
  \providecommand{\doi}{doi: \begingroup \urlstyle{rm}\Url}\fi

\bibitem[Bellet et~al.(2015)Bellet, Liang, Garakani, Balcan, and
  Sha]{bellet2015}
Aur{\'e}lien Bellet, Yingyu Liang, Alireza~Bagheri Garakani, Maria-Florina
  Balcan, and Fei Sha.
\newblock A distributed frank-wolfe algorithm for communication-efficient
  sparse learning.
\newblock In \emph{Proceedings of the 2015 SIAM International Conference on
  Data Mining}, pages 478--486, 2015.

\bibitem[Bertrand and Moonen(2010)]{Bertrand2010}
A.~Bertrand and M.~Moonen.
\newblock Distributed adaptive node-specific signal estimation in fully
  connected sensor networks---part i: Sequential node updating.
\newblock \emph{IEEE Transactions on Signal Processing}, 58\penalty0
  (10):\penalty0 5277--5291, 2010.

\bibitem[Bertrand and Moonen(2015)]{Bertrand2015}
A.~Bertrand and M.~Moonen.
\newblock Distributed canonical correlation analysis in wireless sensor
  networks with application to distributed blind source separation.
\newblock \emph{IEEE Transactions on Signal Processing}, 63\penalty0
  (18):\penalty0 4800--4813, 2015.

\bibitem[Bertrand and Moonen(2014)]{Bertrand2014}
Alexander Bertrand and Marc Moonen.
\newblock Distributed adaptive estimation of covariance matrix eigenvectors in
  wireless sensor networks with application to distributed pca.
\newblock \emph{Signal Processing}, 104:\penalty0 120--135, 2014.

\bibitem[Bunea et~al.(2011)Bunea, She, and Wegkamp]{Bunea2011}
Florentina Bunea, Yiyuan She, and Marten~H. Wegkamp.
\newblock {Optimal selection of reduced rank estimators of high-dimensional
  matrices}.
\newblock \emph{The Annals of Statistics}, 39\penalty0 (2):\penalty0
  1282--1309, 2011.

\bibitem[Bybee et~al.(2021)Bybee, Kelly, Manela, and Xiu]{Bybee2021}
Leland Bybee, Bryan~T Kelly, Asaf Manela, and Dacheng Xiu.
\newblock Business news and business cycles.
\newblock Working Paper 29344, National Bureau of Economic Research, October
  2021.

\bibitem[Chen et~al.(2013)Chen, Dong, and Chan]{chen2013}
Kun Chen, Hongbo Dong, and Kung-Sik Chan.
\newblock Reduced rank regression via adaptive nuclear norm penalization.
\newblock \emph{Biometrika}, 100\penalty0 (4):\penalty0 901--920, 09 2013.

\bibitem[Chen et~al.(2021)Chen, Chi, Fan, and Ma]{Ma2021}
Yuxin Chen, Yuejie Chi, Jianqing Fan, and Cong Ma.
\newblock Spectral methods for data science: A statistical perspective.
\newblock \emph{Foundations and Trends\textregistered in Machine Learning},
  14\penalty0 (5):\penalty0 566--806, 2021.

\bibitem[Dalcín and Fang(2021)]{mpi4py2}
Lisandro Dalcín and Yao-Lung~L. Fang.
\newblock mpi4py: Status update after 12 years of development.
\newblock \emph{Computing in Science and Engineering}, 23\penalty0
  (4):\penalty0 47--54, 2021.

\bibitem[Dalcín et~al.(2005)Dalcín, Paz, and Storti]{mpi4py1}
Lisandro Dalcín, Rodrigo Paz, and Mario Storti.
\newblock Mpi for python.
\newblock \emph{Journal of Parallel and Distributed Computing}, 65\penalty0
  (9):\penalty0 1108--1115, 2005.

\bibitem[Deerwester et~al.(1990)Deerwester, Dumais, Furnas, Landauer, and
  Harshman]{Deerwester1990}
Scott Deerwester, Susan~T. Dumais, George~W. Furnas, Thomas~K. Landauer, and
  Richard Harshman.
\newblock Indexing by latent semantic analysis.
\newblock \emph{Journal of the American Society for Information Science},
  41\penalty0 (6):\penalty0 391--407, 1990.

\bibitem[Ding et~al.(2020)Ding, Fei, Xu, and Yang]{spec20}
Lijun Ding, Yingjie Fei, Qiantong Xu, and Chengrun Yang.
\newblock Spectral frank-wolfe algorithm: Strict complementarity and linear
  convergence.
\newblock In \emph{Proceedings of the 37th International Conference on Machine
  Learning}, volume 119, pages 2535--2544, 13--18 Jul 2020.

\bibitem[Ding et~al.(2021)Ding, Fan, and Udell]{ding2020}
Lijun Ding, Jicong Fan, and Madeleine Udell.
\newblock $k$ fw: A frank-wolfe style algorithm with stronger subproblem
  oracles.
\newblock \emph{arXiv preprint arXiv:2006.16142}, 2021.

\bibitem[Dunn and Smyth(2018)]{dunn2018}
P.K. Dunn and G.K. Smyth.
\newblock \emph{Generalized Linear Models With Examples in R}.
\newblock Springer Texts in Statistics. Springer New York, 2018.

\bibitem[Fan and Lv(2008)]{fan2008sure}
Jianqing Fan and Jinchi Lv.
\newblock Sure independence screening for ultrahigh dimensional feature space.
\newblock \emph{Journal of the Royal Statistical Society: Series B (Statistical
  Methodology)}, 70\penalty0 (5):\penalty0 849--911, 2008.

\bibitem[Fan et~al.(2015)Fan, James, and Radchenko]{Fan2015}
Yingying Fan, Gareth~M. James, and Peter Radchenko.
\newblock Functional additive regression.
\newblock \emph{The Annals of Statistics}, 43\penalty0 (5):\penalty0
  2296--2325, 2015.

\bibitem[Fercoq et~al.(2014)Fercoq, Qu, Richt{\'a}rik, and Tak{\'a}{\v
  c}]{hydra2}
Olivier Fercoq, Zheng Qu, Peter Richt{\'a}rik, and Martin Tak{\'a}{\v c}.
\newblock Fast distributed coordinate descent for non-strongly convex losses.
\newblock In \emph{2014 IEEE International Workshop on Machine Learning for
  Signal Processing (MLSP)}, pages 1--6, 2014.

\bibitem[Frank and Wolfe(1956)]{fw1956}
Marguerite Frank and Philip Wolfe.
\newblock An algorithm for quadratic programming.
\newblock \emph{Naval Research Logistics Quarterly}, 3:\penalty0 95--110, 1956.

\bibitem[Gandhi et~al.(2019)Gandhi, Loughran, and McDonald]{gandhi2019}
Priyank Gandhi, Tim Loughran, and Bill McDonald.
\newblock Using annual report sentiment as a proxy for financial distress in
  u.s. banks.
\newblock \emph{Journal of Behavioral Finance}, 20\penalty0 (4):\penalty0
  424--436, 2019.

\bibitem[Gao and Tsay(2022)]{Gao2022}
Zhaoxing Gao and Ruey~S. Tsay.
\newblock Divide-and-conquer: A distributed hierarchical factor approach to
  modeling large-scale time series data.
\newblock \emph{Journal of the American Statistical Association}, 2022.

\bibitem[Garber(2020)]{garber2020}
Dan Garber.
\newblock Revisiting frank-wolfe for polytopes: Strict complementarity and
  sparsity.
\newblock In \emph{Advances in Neural Information Processing Systems},
  volume~33, pages 18883--18893, 2020.

\bibitem[Gene H.~Golub and Wahba(1979)]{Golub1979}
Michael~Heath Gene H.~Golub and Grace Wahba.
\newblock Generalized cross-validation as a method for choosing a good ridge
  parameter.
\newblock \emph{Technometrics}, 21\penalty0 (2):\penalty0 215--223, 1979.

\bibitem[Han et~al.(2023)Han, Tsay, and Wu]{Han2023}
Yuefeng Han, Ruey~S. Tsay, and Wei~Biao Wu.
\newblock {High dimensional generalized linear models for temporal dependent
  data}.
\newblock \emph{Bernoulli}, 29\penalty0 (1):\penalty0 105--131, 2023.

\bibitem[Hanley and Hoberg(2019)]{hanley2019}
Kathleen~Weiss Hanley and Gerard Hoberg.
\newblock Dynamic interpretation of emerging risks in the financial sector.
\newblock \emph{The Review of Financial Studies}, 32\penalty0 (12):\penalty0
  4543--4603, 02 2019.

\bibitem[Heinze et~al.(2016)Heinze, McWilliams, and Meinshausen]{Heinze16}
Christina Heinze, Brian McWilliams, and Nicolai Meinshausen.
\newblock Dual-loco: Distributing statistical estimation using random
  projections.
\newblock In \emph{Proceedings of the 19th International Conference on
  Artificial Intelligence and Statistics}, volume~51, pages 875--883, Cadiz,
  Spain, 2016.

\bibitem[Hu et~al.(2019)Hu, Niu, Yang, and Zhou]{hu2019}
Yaochen Hu, Di~Niu, Jianming Yang, and Shengping Zhou.
\newblock Fdml: A collaborative machine learning framework for distributed
  features.
\newblock In \emph{Proceedings of the 25th ACM SIGKDD International Conference
  on Knowledge Discovery \& Data Mining}, pages 2232--2240, 2019.

\bibitem[Ing(2020)]{ing2020}
Ching-Kang Ing.
\newblock Model selection for high-dimensional linear regression with dependent
  observations.
\newblock \emph{The Annals of Statistics}, 48\penalty0 (4):\penalty0
  1959--1980, 2020.

\bibitem[Ing and Lai(2011)]{ing2011}
Ching-Kang Ing and Tze~Leung Lai.
\newblock A stepwise regression method and consistent model selection for
  high-dimensional sparse linear models.
\newblock \emph{Statistica Sinica}, 21\penalty0 (4):\penalty0 1473--1513, 2011.

\bibitem[Jaggi(2013)]{jaggi2013}
Martin Jaggi.
\newblock Revisiting frank-wolfe: Projection-free sparse convex optimization.
\newblock \emph{Proceedings of the 30th International Conference on Machine
  Learning}, 28\penalty0 (1):\penalty0 427--435, 2013.

\bibitem[Jaggi and Lacoste-Julien(2015)]{jaggi2015}
Martin Jaggi and Simon Lacoste-Julien.
\newblock On the global linear convergence of frank-wolfe optimization
  variants.
\newblock \emph{Advances in Neural Information Processing Systems}, 28, 2015.

\bibitem[Jegadeesh and Wu(2013)]{Jegadeesh2013}
Narasimhan Jegadeesh and Di~Wu.
\newblock Word power: A new approach for content analysis.
\newblock \emph{Journal of Financial Economics}, 110\penalty0 (3):\penalty0
  712--729, 2013.

\bibitem[Kogan et~al.(2009)Kogan, Levin, Routledge, Sagi, and Smith]{kogan2009}
Shimon Kogan, Dimitry Levin, Bryan~R Routledge, Jacob~S Sagi, and Noah~A Smith.
\newblock Predicting risk from financial reports with regression.
\newblock In \emph{Proceedings of Human Language Technologies: The 2009 Annual
  Conference of the North American Chapter of the Association for Computational
  Linguistics}, pages 272--280, 2009.

\bibitem[Lei et~al.(2019)Lei, Zhuo, Caramanis, Dhillon, and Dimakis]{qi2019}
Qi~Lei, Jiacheng Zhuo, Constantine Caramanis, Inderjit~S Dhillon, and
  Alexandros~G Dimakis.
\newblock Primal-dual block generalized frank-wolfe.
\newblock In \emph{Advances in Neural Information Processing Systems},
  volume~32, 2019.

\bibitem[Li et~al.(2019)Li, Liu, and Chen]{li2019}
Gen Li, Xiaokang Liu, and Kun Chen.
\newblock Integrative multi-view regression: Bridging group-sparse and low-rank
  models.
\newblock \emph{Biometrics}, 75\penalty0 (2):\penalty0 593--602, 2019.

\bibitem[Loukas et~al.(2021)Loukas, Fergadiotis, Androutsopoulos, and
  Malakasiotis]{edgar2021}
Lefteris Loukas, Manos Fergadiotis, Ion Androutsopoulos, and Prodromos
  Malakasiotis.
\newblock {EDGAR}-{CORPUS}: Billions of tokens make the world go round.
\newblock In \emph{Proceedings of the Third Workshop on Economics and Natural
  Language Processing}, pages 13--18, Punta Cana, Dominican Republic, 2021.

\bibitem[Lounici et~al.(2011)Lounici, Pontil, Van De~Geer, and
  Tsybakov]{lounici2011}
Karim Lounici, Massimiliano Pontil, Sara Van De~Geer, and Alexandre~B Tsybakov.
\newblock Oracle inequalities and optimal inference under group sparsity.
\newblock \emph{The Annals of Statistics}, 39\penalty0 (4):\penalty0
  2164--2204, 2011.

\bibitem[Pilanci and Wainwright(2016)]{pilanci2016iterative}
Mert Pilanci and Martin~J Wainwright.
\newblock Iterative hessian sketch: Fast and accurate solution approximation
  for constrained least-squares.
\newblock \emph{Journal of Machine Learning Research}, 17\penalty0
  (1):\penalty0 1842--1879, 2016.

\bibitem[Reinsel et~al.(2022)Reinsel, Velu, and Chen]{RRR}
Gregory~C. Reinsel, Raja~P. Velu, and Kun Chen.
\newblock \emph{Multivariate Reduced-Rank Regression}.
\newblock Springer New York, NY, 2022.

\bibitem[Richt{{\'a}}rik and Tak{{\'a}}{\v{c}}(2016)]{Hydra}
Peter Richt{{\'a}}rik and Martin Tak{{\'a}}{\v{c}}.
\newblock Distributed coordinate descent method for learning with big data.
\newblock \emph{Journal of Machine Learning Research}, 17\penalty0
  (75):\penalty0 1--25, 2016.

\bibitem[Ruhe(1970)]{ruhe1970}
Axel Ruhe.
\newblock Perturbation bounds for means of eigenvalues and invariant subspaces.
\newblock \emph{BIT Numerical Mathematics}, 10:\penalty0 343--354, 1970.

\bibitem[Salton and Buckley(1988)]{Salton1988}
Gerard Salton and Christopher Buckley.
\newblock Term-weighting approaches in automatic text retrieval.
\newblock \emph{Information Processing and Management}, 24\penalty0
  (5):\penalty0 513--523, 1988.

\bibitem[Temlyakov(2000)]{temlyakov2000}
V.~N. Temlyakov.
\newblock Weak greedy algorithms.
\newblock \emph{Advances in Computational Mathematics}, 12\penalty0
  (2):\penalty0 213--227, 2000.

\bibitem[Temlyakov(2015)]{temlyakov2015greedy}
V.~N. Temlyakov.
\newblock Greedy approximation in convex optimization.
\newblock \emph{Constructive Approximation}, 41\penalty0 (2):\penalty0
  269--296, 2015.

\bibitem[Vershynin(2018)]{vershynin_2018}
Roman Vershynin.
\newblock \emph{High-Dimensional Probability: An Introduction with Applications
  in Data Science}.
\newblock Cambridge Series in Statistical and Probabilistic Mathematics.
  Cambridge University Press, 2018.

\bibitem[Wang et~al.(2017)Wang, Lee, Mahdavi, Kolar, and Srebro]{Wang2017}
Jialei Wang, Jason~D. Lee, Mehrdad Mahdavi, Mladen Kolar, and Nathan Srebro.
\newblock Sketching meets random projection in the dual: A provable recovery
  algorithm for big and high-dimensional data.
\newblock \emph{Electronic Journal of Statistics}, 11\penalty0 (2):\penalty0
  4896--4944, 2017.

\bibitem[Wang et~al.(2016)Wang, Dunson, and Leng]{deco}
Xiangyu Wang, David Dunson, and Chenlei Leng.
\newblock Decorrelated feature space partitioning for distributed sparse
  regression.
\newblock In \emph{Proceedings of the 30th International Conference on Neural
  Information Processing Systems}, NIPS'16, pages 802--810, 2016.

\bibitem[Wedin(1972)]{wedin1972}
Per-{\AA}ke Wedin.
\newblock Perturbation bounds in connection with singular value decomposition.
\newblock \emph{BIT Numerical Mathematics}, 12\penalty0 (1):\penalty0 99--111,
  1972.

\bibitem[Yang et~al.(2016)Yang, Mahoney, Saunders, and Sun]{screenandclean}
Jiyan Yang, Michael~W. Mahoney, Michael~A. Saunders, and Yuekai Sun.
\newblock Feature-distributed sparse regression: A screen-and-clean approach.
\newblock In \emph{Proceedings of the 30th International Conference on Neural
  Information Processing Systems}, NIPS'16, pages 2711--2719, 2016.

\bibitem[Yeh et~al.(2020)Yeh, Yeh, and Shen]{Yeh2020}
Hsiang-Yuan Yeh, Yu-Ching Yeh, and Da-Bai Shen.
\newblock Word vector models approach to text regression of financial risk
  prediction.
\newblock \emph{Symmetry}, 12\penalty0 (1), 2020.

\bibitem[Zheng et~al.(2018)Zheng, Bellet, and Gallinari]{zheng2018}
Wenjie Zheng, Aur{\'e}lien Bellet, and Patrick Gallinari.
\newblock A distributed frank-wolfe framework for learning low-rank matrices
  with the trace norm.
\newblock \emph{Machine Learning}, 107\penalty0 (8):\penalty0 1457--1475, 2018.

\bibitem[Zhuo et~al.(2020)Zhuo, Lei, Dimakis, and Caramanis]{zhuo2020}
Jiacheng Zhuo, Qi~Lei, Alex Dimakis, and Constantine Caramanis.
\newblock Communication-efficient asynchronous stochastic frank-wolfe over
  nuclear-norm balls.
\newblock In \emph{International Conference on Artificial Intelligence and
  Statistics}, pages 1464--1474, 2020.

\end{thebibliography}

\end{document}